\documentclass[12pt]{iopart}

\usepackage{iopams}
\usepackage{bm,color}  
\usepackage{graphicx}
\usepackage{amssymb}
\usepackage{amsthm}  
\usepackage{color}
\usepackage{epsfig}

\makeatletter
\@addtoreset{equation}{section}
\def\theequation{\thesection.\arabic{equation}}
\makeatother

\usepackage{algorithm}
\usepackage{algpseudocode}

\algnewcommand{\Initialize}[1]{%
  \State \textbf{Initialize:}
  \State \hspace*{\algorithmicindent}\parbox[t]{0.8\linewidth}{\raggedright #1}
}

\newtheorem{dfn}{Definition}
\newtheorem{thm}{Theorem}

\newcommand{\vm}[1]{\boldsymbol{#1}}

\begin{document}

\title[Prediction Errors for Penalized Regressions
based on GAMP]
{Prediction Errors for Penalized Regressions
based on Generalized Approximate Message Passing}

\author{Ayaka Sakata$^{1,2,3}$}

\address{$^1$Department of Statistical Modeling, The Institute of Statistical Mathematics, Midori-cho, Tachikawa, Tokyo 190-8562, Japan}
\address{$^2$Department of Statistical Science, The Graduate University for Advanced Science (SOKENDAI), Hayama-cho, Kanagawa 240-0193, Japan}
\address{$^3$JST PRESTO, 4-1-8 Honcho, Kawaguchi, Saitama 332-0012, Japan}
\ead{ayaka@ism.ac.jp}
\vspace{10pt}	

\begin{abstract}
We discuss the prediction accuracy of assumed statistical models
in terms of prediction errors for the generalized linear model and 
penalized maximum likelihood methods.
We derive the forms of estimators for the prediction errors, such as
$C_p$ criterion, information criteria, and leave-one-out cross validation (LOOCV) error, 
using the generalized approximate message passing (GAMP) 
algorithm and replica method.
These estimators coincide with each other when the 
number of model parameters is sufficiently small; however,
there is a discrepancy between them in particular 
in the parameter region
where the number of model parameters is larger than the data dimension.
In this paper, we review the prediction errors and corresponding estimators, and 
discuss their differences.
In the framework of GAMP, we show that the
information criteria can be 
expressed by using the variance of the estimates.
Further, we demonstrate how to approach LOOCV error from the information criteria by 
utilizing the expression provided by GAMP.
\end{abstract}

%
%
%
%
%

\section{Introduction}

Over the last decade, there has been a considerable increase in the use of statistical methods and related machine learning techniques in various research fields.
This trend, which can be called the `data science 
revolution' \cite{Taddy}, is driven by qualitative and quantitative 
changes in data, increase in computational power,
and development of algorithms.
Further, physics has also been part of this trend; for example, the reconstruction of black hole 
images in the event horizon telescope collaboration using sparse estimation is a recent 
cooperative achievement between astrophysics and statistics \cite{EHT}. 

Natural science is the quest for physical laws as the ground truth, 
and hence, statistical inference in natural science aims at the inference of the truth. 
However, all data do not necessarily possess the ground truth. 
After industrialization, data from
non-natural science fields
have become diversified \cite{Shewhart}, and 
the statistical inference for such non-nature data can be considered practical science for the human society. 
For example, industrial data are used for the reasonable 
determination, control, and creation of values. 
In such cases, it is irrelevant if there are true physical laws, 
or if our statistical models are different from these laws. 
The statistical model is considered acceptable if practical operations based on this model do not result in any problems.

Once the purpose of statistical inference becomes people-oriented, the goodness of a statistical 
model becomes subjective, as the justification by nature is impossible. 
Mathematical 
schemes are required for the construction of statistical models to make statistical inference a scientific procedure. The Akaike information 
criterion (AIC) \cite{AIC} and subsequent cross validation (CV) method \cite{Stone1974}
are schemes where the statistical model is 
objectively selected using data among the set of models prepared subjectively. 
In these methods, the goodness of the model is quantified by its predictive performance, 
i.e. based on whether the model built by the given data can describe the unknown data, 
which are expected to be observed in future. 
The key idea for the mathematical formalism of the prediction is that the predictive 
performance of the model is defined by the distributions rather than the value of a certain 
statistic. 
This idea led to the geometrical understanding of 
the maximum likelihood method and associated model selection
in terms of information theory as a projection of the real to the model space
\cite{Amari}.

Conventional statistics focuses on the construction of statistical models that take into account empirical knowledge on the system, sometimes it requires experts' wisdom. 
In contrast to such approach, 
deep learning,
which is a central topic in the recent research on statistics and 
associated machine learning,
provides a framework to construct the generative model with less manual tuning. 
The development of deep learning does not indicate the decline of the traditional statistical inference.
There have been attempts to understand the properties of deep learning using statistical methodologies such as understanding the double descent 
phenomena in an overparametrized region using ridge regression \cite{surprise}, 
and characterizing the flat minima \cite{flat_minima} in deep neural networks using 
information criteria \cite{Thomas}. 
The new paradigm offered by deep learning gives conventional statistics 
a role as a simplified system for understanding the picture behind complicated 
inference problems. This relationship will resonate with physicists,
who are trying to extract essentials of physical
phenomena using idealized theoretical frameworks or simplified experimental systems.
In addition, 
there exists problems wherein it is difficult to acquire large amounts 
of datasets because of the operating cost and technical difficulties. 
In such cases, conventional
statistical modeling remains an important approach.

In this paper, we review model quantifications in terms of prediction performance,
and explain the physical consideration for statistical methodologies.
We introduce a statistical-physics-based approach for the mathematical 
analysis of the prediction accuracy of the model under consideration.
The mathematical formalism for statistical inference has similarities
with statistical physics, in particular with the theory of spin-glass systems, 
as the mathematical treatment of the randomness between spins' interactions can be 
regarded as that of data
\cite{finite_temperature,Iba1999,Nishimori_book}. 
Utilizing the similarity, statistical physics-based methods have been applied to 
the analyses of the inference problems as a mathematical tool \cite{CS_replica,Krzakala2012}.
The statistical physics methods are effective even in parameter regions where 
asymptotic normality does not hold; this is key in classical statistics. 
This paper indicates the complementarity between traditional statistical 
methods and statistical physics methods.

We consider regression problems on the 
generalized linear model under penalty functions.
We focus on the estimator of the generalization gap, which is the
difference between the prediction and training errors \cite{ESL,Efron2004}.
As estimators for the prediction errors and corresponding generalization gaps, 
we analyze the $C_p$ criterion \cite{Mallows},
Widely Applicable Information Criterion (WAIC) \cite{WAIC}
or Takeuchi Information Criterion (TIC) \cite{Konishi_Kitagawa}
and CV method.
We compare these estimators in the same setting of the penalized maximum likelihood method, i.e. the zero-temperature limit of the Bayesian inference in the terminology of the statistical physics. 
The generalization gaps are quantified by the generalized message passing (GAMP) algorithm.
We show that the generalization gap obtained using the $C_p$ criterion can be expressed by the variance of the estimated parameters. 
This relationship resembles the fluctuation-response relationship in statistical physics, 
in the sense that the response of the error 
is proportional to 
the variance of the estimate at the equilibrium \cite{Watanabe_FDT}.
We also show that both $C_p$ and WAIC or TIC are
insufficient for describing the 
parameter region where the number of parameter is 
sufficiently large.
We correct TIC towards CV by removing the assumption 
introduced for the derivation of TIC.
The resultant form reproduces the CV error with high accuracy 
even in the non-asymptotic limit in statistics,
and one time estimation is sufficient under full data to obtain the form.



The remainder of the paper is organized as follows.
Sec. \ref{sec:overview} provides a brief overview of the prediction errors and their 
estimators.
In sec. \ref{sec:problem_setting}, we explain the problem setting 
discussed in this paper.
We explain the GAMP algorithm in sec. \ref{sec:GAMP},
and we apply the algorithm for the calculation of the
prediction error and corresponding generalization gap.
In sec. \ref{sec:FV_correct},
we present the difference between the information criteria and CV error,
and we explain how to approach the CV error from TIC by utilizing the 
expression obtained using GAMP.
We derive the analytical form of the prediction error for random predictors where each component is independently and identically distributed according to Gaussian distribution 
in sec. \ref{sec:replica}.
Finally, sec. \ref{sec:summary} presents the summary and the conclusion.

\section{Prediction errors and their estimators}
\label{sec:overview}

We focus on a parametric model $f(z|\bm{\theta})$
with parameter $\bm{\theta}\in\bm{\Theta}\subset\mathbb{R}^M$
where $\bm{\Theta}$ represents a parameter space. 
The model parameter $\bm{\theta}$ is estimated under the model $f(\cdot|\bm{\theta})$
and training 
data $\bm{y}\in\mathbb{R}^{M}$
to describe the data generative process.
Here, we consider regression problems wherein the parameter $\bm{\theta}$ is
described by 
$\bm{\theta}(\bm{F},\bm{x})=\bm{F}\bm{x}\slash\sqrt{N}$,
where $\bm{F}\in\mathbb{R}^{M\times N}$ and $\bm{x}\in\mathbb{R}^N$
are the predictor matrix and the regression coefficient to be estimated,
respectively.
Therefore,
parameter space $\Theta$ is specified by the 
distributions of the predictors and regression coefficients.
The data set consists of $\bm{y}$ and $\bm{F}$, hence
we denote ${\cal D}=\{\bm{y},\bm{F}\}$.
In the following discussion,
we assume that each component of the training data 
$y_\mu~(\mu=1,\cdots,M)$ is independently and identically generated according to 
the predetermined rule.
For mathematical convenience, we denote the true rule as the distribution 
$q(y_\mu|\bm{F})$.
Furthermore, we consider a case in which
the columns of the predictors are standardized
such that $\mathbb{E}_{\bm{F}}[\sum_{\mu=1}^MF_{\mu i}^2]=M$
for any $i$.
Let us prepare a set of candidate models 
${\cal M}=\{f_1(z|\bm{\theta}(\bm{F},\widehat{\bm{x}}_1)),\cdots,f_m(z|\bm{\theta}(\bm{F},\widehat{\bm{x}}_m))\}$
for the true distribution, 
where 
$\widehat{\bm{x}}_k({\cal D})$
represents the estimated regression coefficient under the $k$-th model $f_k$
using training data ${\cal D}$. 
The model selection is the problem of adopting a model based on a certain criterion
from the set of candidate models.

We consider the penalized maximum likelihood method
to construct the estimate $\widehat{\bm{x}}({\cal D})$ as
\begin{eqnarray}
\widehat{\bm{x}}({\cal D})=\mathop{\mathrm{argmax}}_{\bm{x}}\varphi({\cal D},\bm{x}),
\end{eqnarray}
where 
\begin{eqnarray}
\varphi({\cal D},\bm{x})=\ln f(\bm{y}|\bm{\theta}(\bm{F},\bm{x}))-h(\bm{x})
\end{eqnarray}
and $f(\bm{y}|\bm{\theta}(\bm{F},\bm{x}))$ and
$h(\bm{x})$ are 
the likelihood and penalty terms, respectively.
In this paper, we focus on separable likelihoods
\begin{eqnarray}
\ln f(\bm{y}|\bm{\theta}(\bm{F},\bm{x}))=\sum_{\mu=1}^M\ln f(y_\mu|\theta_\mu(\mathbf{f}_\mu,\bm{x})),
\end{eqnarray}
where $\mathbf{f}_\mu\in\mathbb{R}^{1\times N}$ represents the $\mu$-th row vector of the predictor matrix,
and $\theta_\mu(\mathbf{f}_\mu,\bm{x})=\frac{1}{\sqrt{N}}\mathbf{f}_\mu\bm{x}$.
Here, $h(\bm{x})$ contains the smoothing or 
regularization parameter introduced to avoid overfitting \cite{ESL, goodd1971}.
Models with varying regularization parameter
constitute the model space;
the determination of the regularization parameter 
corresponds to a model-selection problem.


In general, a model is constructed to represent the given training data sample ${\cal D}$
accurately. The accuracy of representing the training data can 
be quantified by the training error defined by 
\begin{eqnarray}
\mathrm{err}_{\mathrm{train}}({\cal D})=-\frac{1}{M}\sum_{\mu=1}^M\ln 
 f(y_\mu|\theta_\mu(\mathbf{f}_\mu,\widehat{\bm{x}}({\cal D}))).
\end{eqnarray}
The adequacy of the models with 
the parameter $\bm{\theta}(\bm{F},\widehat{\bm{x}}({\cal D}))$
is quantified by the prediction error,
which indicates
the goodness of the model 
for the description of the test data not used for training. 
The high prediction error implies the overfitting,
and therefore, a model with a small prediction error is required in practice.

\subsection{Extra-sample and in-sample prediction errors}

\begin{figure}
\centering
\includegraphics[width=5in]{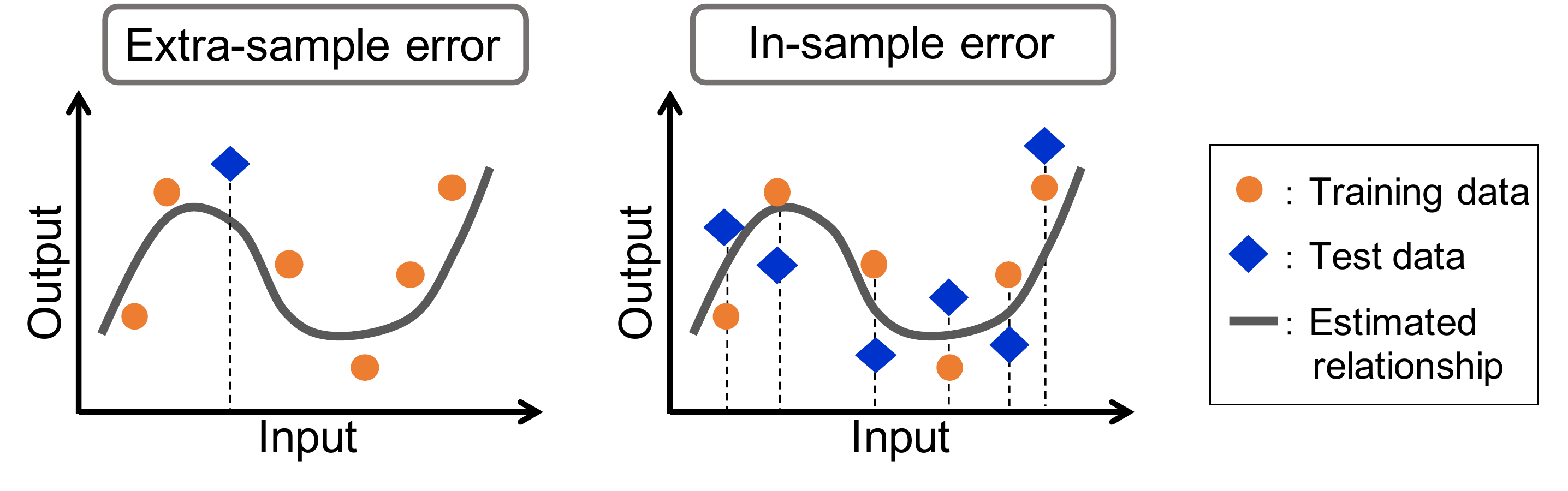}
\caption{Definitions of extra-sample and in-sample prediction errors.
A test sample whose position is not in the training data is used for the evaluation of the extra-sample prediction error.
Test samples whose positions are the same as that of the 
training sample are used for the in-sample prediction error.}
\label{fig:extra_and_in}
\end{figure}

There are two definitions of the prediction error:
extra-sample prediction error and in-sample prediction error, 
which are considering different generative processes of the test data 
used for evaluating the prediction error \cite{ESL}.
Their definitions are given as follows.
\begin{dfn}[Extra-sample prediction error]
Let us denote 
${\bf f}_{\mathrm{new}}\in\mathbb{R}^{1\times N}$
as a predictor value that is not included in the training sample ${\cal D}$.
The extra-sample prediction error for the penalized maximum likelihood method
is defined by 
\begin{eqnarray}
\mathrm{err}_{\mathrm{pre}}^{(\mathrm{ex})}({\cal D})=-\mathbb{E}_{z_{\mathrm{new}},{\bf f}_{\mathrm{new}}}
\left[\ln f\left(z_{\mathrm{new}}\Big|\frac{1}{\sqrt{N}}{\bf f}_{\mathrm{new}}\widehat{\bm{x}}({\cal D})\right)\right],
\label{eq:def_ex_beta0}
\end{eqnarray}
where $z_{\mathrm{new}}\in\mathbb{R}$ represents the test data associated with the predictor 
${\bf f}_{\mathrm{new}}$, and 
$\mathbb{E}_{z_{\mathrm{new}},{\bf f}_{\mathrm{new}}}[\cdot]$ 
denotes the expectation with respect to 
$z_{\mathrm{new}}$ and ${\bf f}_{\mathrm{new}}$.

\end{dfn}

\begin{dfn}[In-sample prediction error]
The in-sample prediction error for the penalized maximum likelihood method is given by
\begin{eqnarray}
\mathrm{err}_{\mathrm{pre}}^{(\mathrm{in})}({\cal D})= -\frac{1}{M}
\sum_{\mu=1}^M\mathbb{E}_{z_\mu\sim q(z|{\bf f}_\mu)}[\ln{f(z_\mu|\theta_\mu({\bf f}_\mu,\widehat{\bm{x}}({\cal D})))}],
\label{eq:def_in}
\end{eqnarray}
where $z_\mu$ represents the new response variable (test data) observed at 
$\mu$-th training point.

\end{dfn}

Fig.\ref{fig:extra_and_in} schematically shows the difference between 
the test data for 
the extra-sample and in-sample prediction errors.
The terms `extra' and `in' indicate 
whether the position of the test data overlaps the training set.
The test data is on the location not covered by training samples in the case of the extra-sample error, hence the extra-sample error 
describes the prediction error for varying the predictors.
The in-sample prediction error considers the test data generated by
repeated observation on the same position.
The in-sample prediction error sometimes does not match the 
actual problem settings; however, it yields a convenient comparison method between models and 
leads to an understanding of the model complexity as explained later.
In the parameter region where the conventional statistics focused on,
the tendencies of the two prediction errors 
are considered to be similar with each other \cite{ESL}.

The exact values of both prediction errors cannot be obtained because we
do not know the true generative process of the data.
Instead, estimators for the prediction errors are 
used for evaluating the prediction accuracy.

\subsection{Cross-validation error}

The cross-validation error is an estimator for the extra-sample prediction error
\cite{Stone1974}.
\begin{dfn}[Leave-one-out cross validation error]
Let us denote a leave-one-out (LOO) sample ${\cal D}_{\setminus\mu}$
where the $\mu$-th data sample ${\cal D}_\mu\equiv\{y_\mu,\mathbf{f}_\mu\}$
is discarded from the data set ${\cal D}$.
The estimate under the LOO sample, which is denoted by $\widehat{\bm{x}}({\cal D}_{\setminus\mu})$,
is defined as
\begin{eqnarray}
\widehat{\bm{x}}({\cal D}_{\setminus\mu})=\mathop{\mathrm{argmin}}_{\bm{x}}\sum_{\nu\in\bm{M}\setminus\mu}\ln f(y_\nu|\theta_\nu(\mathbf{f}_\nu,\widehat{\bm{x}}({\cal D}_\nu)))-h(\bm{x}),
\end{eqnarray}
where 
$\bm{M}=\{1,2,\ldots,M\}$ and $\bm{M}\setminus\mu$ 
denotes $\bm{M}$ except $\mu$.
The leave-one-out cross validation (LOOCV) error is defined as
\begin{eqnarray}
\mathrm{err}_{\mathrm{LOOCV}}({\cal D})=-\frac{1}{M}\sum_{\mu=1}^M
\ln
f\left(y_\mu\Big|\frac{1}{\sqrt{N}}{\bf f}_\mu\widehat{\bm{x}}({\cal D}_{\setminus\mu})\right).
\label{eq:def_CV}
\end{eqnarray}
\end{dfn}
Intuitively, for a sufficiently large $M$,
the empirical mean with respect to the data sample is assumed to converge to the 
expectation value
with respect to the true distribution. Therefore, the LOOCV error is considered an adequate estimator of the extra-sample prediction error.

\subsubsection{Special form of LOOCV error for the linear estimation rule}

For calculating the LOOCV error,
we need to compute the estimate $M$ times
under $M$-LOO-samples ${\cal D}_{\setminus\mu}$
for $\mu\in\bm{M}$.
There is a long history of analytical and numerical 
attempts to reduce the computational effort involved in evaluating the LOOCV error.
Here, we focus our discussion on the linear estimation rule.
\begin{dfn}[Linear estimation rule]
In the linear estimation rule, 
the output $\bm{y}$ is estimated as $\widehat{\bm{y}}=\bm{H}\bm{y}$
on the assumed model with an arbitrary matrix $\bm{H}\in\mathbb{R}^{M\times M}$.
\end{dfn}
The ordinary least square (OLS) method is one of the
One of the linear estimation rules,
wherein the regression coefficient is given as
$\widehat{\bm{x}}({\cal D})=\frac{1}{\sqrt{N}}(\frac{1}{N}\bm{F}^{\top}\bm{F})^{-1}\bm{F}^{\top}\bm{y}$
at $M>N$,
then $\widehat{\bm{y}}=\bm{H}\bm{y}$ with $\bm{H}=\frac{1}{N}\bm{F}(\bm{F}^{\top}\bm{F})^{-1}\bm{F}^{\top}$.
Here, the estimated output corresponds to the 
projection of the actual output onto the column space of the predictor matrix,
and the matrix $\bm{H}$ for the OLS case is known as the hat matrix \cite{hat_matrix}.

In the linear estimation methods,
the computation of the 
LOOCV error is reduced to a one-time calculation of the estimate 
under full training data ${\cal D}$ and
is expressed using the matrix $\bm{H}$.
The expression for the Gaussian likelihood
shown in the following theorem
is known as the predicted residual sum of squares (PRESS) 
statistics \cite{Allen1974},
and was derived by \cite{Cook1977,Cook1979}.
\begin{thm}
For the linear estimation rule where $\widehat{\bm{y}}=\bm{H}\bm{y}$ holds,
the LOOCV error 
under the Gaussian likelihood
is given as
\begin{eqnarray}
\mathrm{err}_{\mathrm{LOOCV}}({\cal D})=\frac{1}{2M}\sum_{\mu=1}^M
\left(\frac{y_\mu-\widehat{y}_\mu({\cal D})}{1-H_{\mu\mu}}\right)^2.
\label{eq:LOOCV_linear}
\end{eqnarray}
\end{thm}
\begin{proof}
We prepare a data set ${\cal D}_{\setminus\mu^+}=\{\bm{y}_{\setminus\mu^+},\bm{F}\}$,
where 
$\bm{y}_{\setminus\mu^+}=\{y_1,\cdots,y_{\mu-1},\widehat{y}_\mu({\cal D}_{\setminus\mu}),y_{\mu+1},\cdots,y_M\}$ and 
$\widehat{y}_\mu({\cal D}_{\setminus\mu})=\mathbf{f}_\mu\widehat{\bm{x}}({\cal D}_{\setminus\mu})\slash\sqrt{N}$, as shown in Fig.\ref{fig:cavity_plus}.
The estimate under the data ${\cal D}_{\setminus\mu^+}$ is defined by
\begin{eqnarray}
\nonumber
\widehat{\bm{x}}({\cal D}_{\setminus\mu^+})&=\mathop{\mathrm{argmax}}_{\bm{x}}
\Big\{\sum_{\nu\in\bm{M}\setminus\mu}\ln f(y_\nu|\theta_\nu(\mathbf{f}_\nu,\bm{x}))\\
&\hspace{3.0cm}+
\ln f(\widehat{y}_\mu({\cal D}_{\setminus\mu})|\theta_\mu(\mathbf{f}_\mu,\bm{x}))-h(\bm{x})\Big\}.
\label{eq:estimator_add}
\end{eqnarray}
We can show that $\widehat{\bm{x}}({\cal D}_{\setminus\mu})=\widehat{\bm{x}}({\cal D}_{\setminus\mu^+})$
holds for any $\mu\in\bm{M}$. 
By definition, the following inequality holds for any $\bm{x}$; i.e.
\begin{eqnarray}
\nonumber
\sum_{\nu\in\bm{M}\setminus\mu}\ln f(y_\nu|\theta_\nu(\widehat{\bm{x}}(\mathbf{f}_\nu,{\cal D}_{\setminus\mu}))
&-h(\widehat{\bm{x}}({\cal D}_{\setminus\mu}))\\
&\geq 
\sum_{\nu\in\bm{M}\setminus\mu}\!\!\ln f(y_\nu|\theta_\nu(\mathbf{f}_\nu,\bm{x}))-h(\bm{x}).
\label{eq:cavity_by_def}
\end{eqnarray}
Further, it is natural to consider that the following relationship holds for any $\theta$
\begin{eqnarray}
\ln f(y_\mu=\theta_\mu|\theta_\mu)\geq \ln f(y_\mu=\theta_\mu|\theta).
\label{eq:f_trivial}
\end{eqnarray}
Combining (\ref{eq:cavity_by_def}) with (\ref{eq:f_trivial}),
we obtain
\begin{eqnarray}
\nonumber
\varphi({\cal D}_{\setminus\mu^+},\bm{x})&=
\sum_{\nu\in\bm{M}\setminus\mu}\ln f(y_\nu|\theta_\nu(\mathbf{f}_\nu,\bm{x}))
+\ln f(\theta_\mu(\mathbf{f}_\mu,\widehat{\bm{x}}({\cal D}_{\setminus\mu}))|\theta_\mu(\mathbf{f}_\mu,\bm{x}))-h(\bm{x})\\
\nonumber
&\leq\sum_{\nu\in\bm{M}\setminus\mu}\ln f(y_\nu|\theta_\nu(\mathbf{f}_\nu,\widehat{\bm{x}}({\cal D}_{\setminus\mu})))-h(\widehat{\bm{x}}({\cal D}_{\setminus\mu}))\\
\nonumber
&\hspace{3.0cm}+\ln f(\theta_\mu(\mathbf{f}_\mu,\widehat{\bm{x}}({\cal D}_{\setminus\mu}))|\theta_\mu(\mathbf{f}_\mu,\widehat{\bm{x}}({\cal D}_{\setminus\mu})))\\
&=\varphi({\cal D}_{\setminus\mu^+},\widehat{\bm{x}}({\cal D}_{\setminus\mu}))\end{eqnarray}
hence $\widehat{\bm{x}}({\cal D}_{\setminus\mu^+})=\widehat{\bm{x}}({\cal D}_{\setminus\mu})$ is one of the solutions of (\ref{eq:estimator_add}), and 
for the linear estimator, the following relationship holds:
\begin{eqnarray}
\frac{1}{\sqrt{N}}\sum_{i=1}^NF_{\nu i}\widehat{x}_i({\cal D}_{\setminus\mu^+})=\sum_{\eta\in\bm{M}\setminus\mu}H_{\nu\eta}y_{\eta}+H_{\nu\mu}\widehat{y}_{\mu}({\cal D}_{\setminus\mu}).
\end{eqnarray}
Meanwhile, by utilizing the relationship $\widehat{\bm{x}}({\cal D}_{\setminus\mu^+})=\widehat{\bm{x}}({\cal D}_{\setminus\mu})$,
we obtain
\begin{eqnarray}
\nonumber
\widehat{y}_{\mu}({\cal D}_{\setminus\mu})&=\frac{1}{\sqrt{N}}\sum_{i=1}^NF_{\mu i}\widehat{x}_i({\cal D}_{\setminus\mu})=\frac{1}{\sqrt{N}}\sum_{i=1}^NF_{\mu i}\widehat{x}_i({\cal D}_{\setminus\mu^+})\\
\nonumber
&=H_{\mu\mu}\widehat{y}_{\mu}({\cal D}_{\setminus\mu})+\sum_{\eta\in\bm{M}\setminus\mu}H_{\mu\eta}y_{\eta}\\
&=H_{\mu\mu}\widehat{y}_{\mu}({\cal D}_{\setminus\mu})+\widehat{y}_\mu({\cal D})-H_{\mu\mu}y_\mu.
\label{eq:linear_CV_proof}
\end{eqnarray}
For the Gaussian likelihood,
the LOOCV error corresponds to $\frac{1}{2M}\sum_{\mu=1}^M(y_\mu-\widehat{y}_\mu({\cal D}_{\setminus\mu}))$, and using (\ref{eq:linear_CV_proof}), 
we obtain (\ref{eq:LOOCV_linear}).
\end{proof}
Eq.(\ref{eq:linear_CV_proof}) is valid for 
the linear estimation rule,
wherein one can compute the LOOCV error
by substituting 
$\hat{y}_\mu({\cal D}_{\setminus\mu})$
with the likelihood under consideration.

\begin{figure}
\centering
\includegraphics[width=5in]{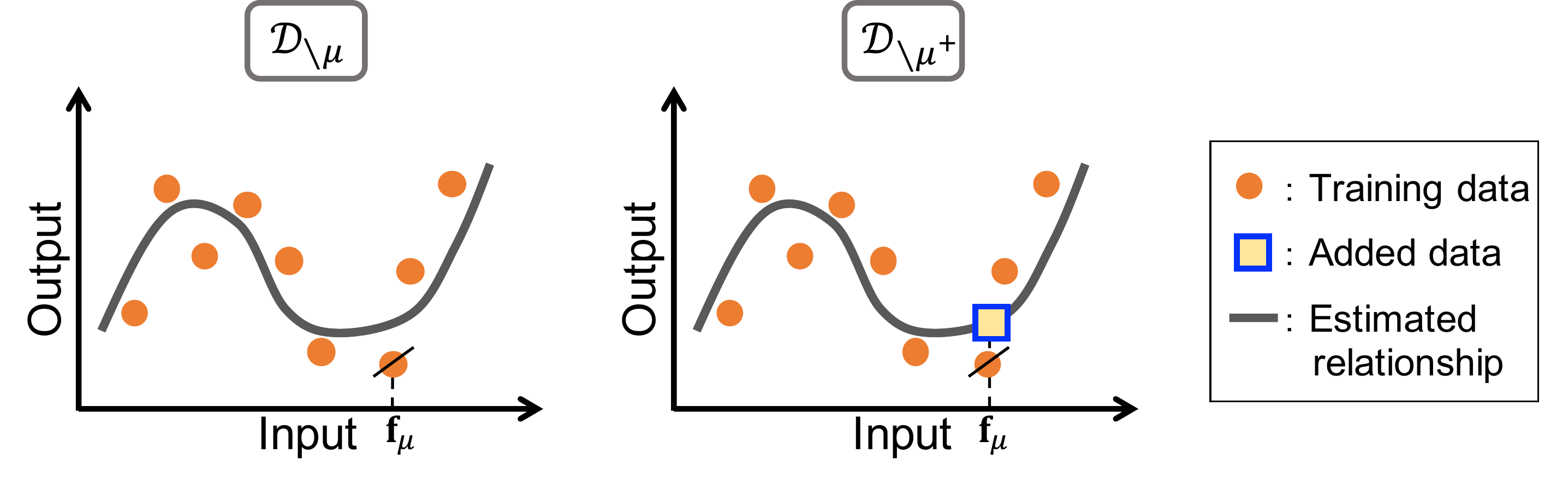}
\caption{Schematic of data samples ${\cal D}_{\setminus\mu}$ and 
${\cal D}_{\setminus\mu^+}$. 
The data set ${\cal D}_{\setminus\mu}$ comprises the input--output relationship 
except the $\mu$-th sample denoted by $\varnothing$.
The point $\square$ of ${\cal D}_{\setminus\mu^+}$ is 
at $\{\mathbf{f}_\mu,\theta_\mu(\widehat{\bm{x}}({\cal D}_{\setminus\mu}))\}$,
where $\widehat{\bm{x}}({\cal D}_{\setminus\mu})$ denotes the 
estimate under ${\cal D}_{\setminus\mu}$, and 
${\cal D}_{\setminus\mu^+}=\{{\cal D}_{\setminus\mu}\cup\square\}$.}
\label{fig:cavity_plus}
\end{figure}

\subsubsection{Sherman--Morrison Formula for the LOOCV error}

The same expression as (\ref{eq:LOOCV_linear}) can be 
obtained using the 
Sherman--Morrison formula 
\cite{sherman1950adjustment} for linear estimators,
as shown in \cite{Seber_and_Lee}.
We define the predictor matrix without the $\nu$-th row vector as $\bm{F}_{\setminus\nu}$.
The Sherman--Morrison formula is given as
\begin{eqnarray}
\left(\frac{1}{N}{\bm{F}_{\setminus\nu}}^{\top}\bm{F}_{\setminus\nu}\right)^{-1}
&\!\!\!=\!\left(\frac{1}{N}\bm{F}^{\top}\bm{F}\right)^{-1}\!\!\!+\!\frac{\displaystyle\left(\frac{1}{N}\bm{F}^{\top}\bm{F}\right)^{-1}\!\!\!\frac{1}{N}{\bf f}_\nu^\top{\bf f}_\nu\left(\frac{1}{N}\bm{F}^{\top}\bm{F}\right)^{-1}}
{1-\displaystyle\frac{1}{N}{\bf f}_\nu\left(\frac{1}{N}\bm{F}^{\top}\bm{F}\right)^{-1}{\bf f}_\nu^\top}.
\label{eq:Sherman-Morrison}
\end{eqnarray}
The estimates under the LOO sample ${\cal D}_{\setminus\mu}$
is given by $\widehat{\bm{x}}({\cal D}_{\setminus\mu})=\frac{1}{\sqrt{N}}(\frac{1}{N}\bm{F}_{\setminus\mu}^\top\bm{F}_{\setminus\mu})^{-1}\bm{F}_{\setminus\mu}^\top$,
and hence, $\widehat{\bm{y}}({\cal D}_{\setminus\mu})=\bm{H}^{\setminus\mu}\bm{y}$, where $\bm{H}^{\setminus\mu}=\frac{1}{N}\bm{F}(\frac{1}{N}\bm{F}_{\setminus\mu}^\top\bm{F}_{\setminus\mu})^{-1}\bm{F}_{\setminus\mu}^{\top}$.
Multiplying $\bm{F}\slash\sqrt{N}$ and $\bm{F}_{\setminus\mu}^{\top}\slash\sqrt{N}$ from the left and right
of both sides of (\ref{eq:Sherman-Morrison}), respectively,
we get 
\begin{eqnarray}
\widehat{y}_\mu({\cal D}_{\setminus\mu})=\widehat{y}_\mu({\cal D})-H_{\mu\mu}y_\mu+\frac{H_{\mu\mu}(\widehat{y}_\mu({\cal D})-H_{\mu\mu}y_\mu)}{1-H_{\mu\mu}},
\label{eq:SM_CV}
\end{eqnarray}
Then, transforming (\ref{eq:SM_CV}), we obtain the same expression as 
(\ref{eq:LOOCV_linear}).

\subsection{
Information criteria as estimators of the in-sample prediction error}

Information criteria are known as the
estimators of the in-sample prediction error.
These criteria were originally
introduced to evaluate the quality of the statistical model
constructed with the maximum likelihood method
using the Kullback--Leibler (KL) divergence
defined by
\begin{eqnarray}
{\rm KL}(q:f)=\mathbb{E}_{\bm{z}\sim q(\bm{z})}
\left[\ln{q(\bm{z})}\right]-
\mathbb{E}_{\bm{z}\sim q(\bm{z})}[\ln{f(\bm{z}|\bm{\theta}(\bm{F},\widehat{\bm{x}}_{\rm ML}({\cal D}))}],
\label{eq:def_KL}
\end{eqnarray}
where $\widehat{\bm{x}}_{\rm ML}({\cal D})$ represents the maximum likelihood estimator 
under the training sample ${\cal D}$.
A statistical model with smaller KL divergence to the true distribution is
expected to be appropriate for the expression of given data.
The second term of KL divergence 
(\ref{eq:def_KL}) corresponds to the in-sample prediction error
under the ML estimator,
and the first term is a constant.
Therefore, the minimization of KL divergence corresponds to that of 
the in-sample prediction error.

Takeuchi information criterion (TIC) is an estimator of the in-sample prediction error
and in particular under the maximum likelihood method,
TIC corresponds to an unbiased estimator of the in-sample prediction error 
\cite{Konishi_Kitagawa}.
\begin{dfn}[Takeuchi information criterion]
The TIC is defined as
\begin{eqnarray}
\mathrm{TIC}({\cal D})=2\left\{\mathrm{err}_{\mathrm{train}}({\cal D})+\frac{1}{M}\mathrm{Tr}\left(\widehat{\cal I}(\widehat{\bm{x}}({\cal D}))\widehat{\cal J}^{-1}(\widehat{\bm{x}}({\cal D}))\right)\right\},
\label{eq:def_TIC}
\end{eqnarray}
where $\widehat{\cal I}$ and $\widehat{\cal J}$ represent the estimator of the 
Fisher information matrix
$\overline{\cal I}$
and the expected Hessian $\overline{\cal J}$
given as 
\begin{eqnarray}
\overline{\cal I}(\widehat{\bm{x}})&=\frac{1}{M}\mathbb{E}_{\bm{z}\sim q(\bm{z})}\left[\frac{\partial\ln f(\bm{z}|\bm{\theta}(\bm{F},\bm{x}))}{\partial\bm{x}}\frac{\partial\ln f(\bm{z}|\bm{\theta}(\bm{F},\bm{x}))}{\partial\bm{x}^{\top}}\Big|_{\bm{x}=\widehat{\bm{x}}}\right]\\
\overline{\cal J}(\widehat{\bm{x}})&=-\frac{1}{M}\mathbb{E}_{\bm{z}\sim q(\bm{z})}\left[\frac{\partial^2\ln f(\bm{z}|\bm{\theta}(\bm{F},\bm{x}))}{\partial\bm{x}\partial\bm{x}^{\top}}\Big|_{\bm{x}=\widehat{\bm{x}}}\right],
\end{eqnarray}
respectively.
\end{dfn}
The factor 2 in the definition of TIC comes from
the convention that TIC for the Gaussian likelihood corresponds to the 
estimator for the squared error.
As the estimators of Fisher information matrix and 
the expected Hessian,
sample-mean of the (non-centered) covariance of gradient
${\cal I}(\bm{x})$ 
and that of the negative Hessian ${\cal J}(\bm{x})$ 
are used, where
\begin{eqnarray}
{\cal I}_{ij}(\widehat{\bm{x}})&=\frac{1}{M}\sum_{\mu=1}^M\!\frac{\partial}{\partial x_i}\!\ln \!f(y_\mu|\theta_\mu(\mathbf{f}_\mu,\bm{x}))\Big|_{\bm{x}
=\widehat{\bm{x}}}
\!\frac{\partial}{\partial x_j}\!\ln \!f(y_\mu|\theta_\mu(\mathbf{f}_\mu,\bm{x}))\Big|_{\bm{x}
=\widehat{\bm{x}}},\label{eq:def_I}\\
{\cal J}_{ij}(\widehat{\bm{x}})&=-\frac{1}{M}\frac{\partial^2}{\partial x_i\partial x_j}
\sum_{\mu=1}^M
f(y_\mu|\theta_\mu(\bm{F},\bm{x})\Big|_{\bm{x}=\widehat{\bm{x}}}.
\label{eq:def_J}
\end{eqnarray}
with the expectation that they converge to $\overline{\cal I}$ and $\overline{\cal J}$
from the law of large numbers.

For the special case where the model space contains the true model,
where a regression coefficient $\bm{x}^*$
exists such that $q(\bm{z})=f(\bm{z}|\frac{1}{\sqrt{N}}\bm{F}\bm{x}^*)$,
it is shown that
TIC is reduced to Akaike's information criterion (AIC) \cite{AIC},
which is defined as
\begin{eqnarray}
\mathrm{AIC}({\cal D})=2\left(\mathrm{err}_{\mathrm{train}}({\cal D})+\frac{N}{M}\right).
\end{eqnarray}
This is the consequence of 
$\overline{\cal I}\overline{\cal J}^{-1}=\bm{I}_N$
that holds in the special case \cite{Cox_Hinkley},
where $\bm{I}_N$ represents the $N$-dimensional identity matrix.

The difference between the prediction and training error is termed the 
generalization gap.
Based on the correspondence between 
the generalization gap obtained by AIC and the dimension of the parameters,
the generalization gap is considered for expressing the model complexity.
The model can express various distributions with an increase in the complexity of the model, and therefore, there is a decrease in the training error. 
However, the models with large parameters 
are prone to overfitting, which hampers the prediction of unknown data,
because these models can fit the randomness in given training data.
The form of AIC is given by the training error
penalized by the number of model parameters.
Therefore, model selection based on the minimization of AIC 
represents a tradeoff between fitting accuracy to training data and model complexity.


We note that the meaning of AIC as an unbiased estimator of the 
in-sample prediction error
holds in the case where the 
the maximum likelihood estimator is considered,
and the model candidates include true distribution.
In the case where the latter condition does not hold,
TIC is theoretically adequate for evaluating the in-sample prediction error.
This does not mean, however, that the AIC should not be used, although 
theoretical justification is difficult.
In fact, AIC is, at times, superior to TIC 
in terms of low computational cost and the small variance.
To evaluate TIC, we need to compute the inverse of Hessian,
whose computational cost is $O(N^3)$,
and the fluctuation of the generalization gap
of TIC tends to be large, in particular when 
eigenvalues close to zero exist \cite{Ripley2007}.

\subsection{Asymptotic equivalence between the LOOCV error and information criterion}

The LOOCV error has an asymptotic equivalence to the information criterion,
as proved in \cite{Stone1977}
with the following assumptions:
\begin{description}
\item[A1] 
The differences in log-likelihoods  under full sample and LOO sample can be expressed 
by the first order of the difference between the corresponding estimates,
under the consideration that the difference $\widehat{x}_i({\cal D}_{\setminus\mu})-
\widehat{x}_i({\cal D})$ is sufficiently small for any $i\in\bm{N}$ and $\mu\in\bm{M}$,
\end{description}
Here, we set $\bm{N}=\{1,\ldots,N\}$.
Under assumption {\bf A1}, 
we expand the difference in log-likelihoods for 
${\cal D}$ and ${\cal D}_{\setminus\mu}$
up to the first order
of $\widehat{\bm{x}}({\cal D})-\widehat{\bm{x}}({\cal D}_{\setminus\mu})$ as
\begin{eqnarray}
\nonumber
\ln f(y_\mu&|\theta_\mu(\mathbf{f}_\mu,\widehat{\bm{x}}({\cal D})))-\ln f(y_\mu|\theta_\mu(\mathbf{f}_\mu,\widehat{\bm{x}}({\cal D}_{\setminus\mu})))\\
&=\sum_{i=1}^N\frac{\partial}{\partial x_i}\ln f(y_\mu|\theta_\mu(\mathbf{f}_\mu,\bm{x}))\Big|
_{\bm{x}=\widehat{\bm{x}}({\cal D}_{\setminus\mu})}
(\widehat{x}_i({\cal D})-\widehat{x}_i({\cal D}_{\setminus\mu})).
\label{eq:lnf_expand}
\end{eqnarray}
Following the procedure explained in \ref{sec:app_CV}
under assumption {\bf A1},
we obtain the relationship
\begin{eqnarray}
\sum_{j=1}^N{\cal J}_{ij}^{\varphi}(\widehat{\bm{x}}({\cal D}_{\setminus\nu}))(\widehat{x}_i({\cal D})-\widehat{x}_i({\cal D}_{\setminus\nu}))=
\frac{1}{M}\frac{\partial}{\partial x_i}\ln f(y_\nu|\theta_\nu(\mathbf{f}_\nu,\bm{x}))\Big|_{\bm{x}=\widehat{\bm{x}}({\cal D}_{\setminus\nu})},
\label{eq:diff_expand}
\end{eqnarray}
where
\begin{eqnarray}
{\cal J}_{ij}^{\varphi}(\widehat{\bm{x}})&=-\frac{1}{M}\frac{\partial^2}{\partial x_i\partial x_j}
\varphi({\cal D},\bm{x})\Big|_{\bm{x}=\widehat{\bm{x}}}.
\label{eq:def_J_general}
\end{eqnarray}
Here, we introduce the assumption:
\begin{description}
\item[A2] ${\cal J}^{\varphi}_{ij}(\widehat{\bm{x}}({\cal D}_{\setminus\nu}))={\cal J}^{\varphi}_{ij}(\widehat{\bm{x}}({\cal D}))$ holds
for any $\nu\in\bm{M}$ and $i,j\in\bm{N}$.
\end{description}
Under assumption {\bf A2},
we obtain
\begin{eqnarray}
\widehat{\bm{x}}({\cal D})-\widehat{\bm{x}}({\cal D}_{\setminus\nu})=
\frac{1}{M}({\cal J}^\varphi(\widehat{\bm{x}}({\cal D})))^{-1}\frac{\partial}{\partial\bm{x}}\ln f(y_\nu|\theta_\nu(\mathbf{f}_\nu,\bm{x}))\Big|_{\bm{x}=\widehat{\bm{x}}({\cal D}_{\setminus\nu})},
\label{eq:diff_cavity_final}
\end{eqnarray}
where we implicitly assume that the Hessian (\ref{eq:def_J_general}) is invertible.

We introduce another assumption as follows: 
\begin{description}
\item[A3] ${\cal I}^{\setminus\cdot}(\{\widehat{\bm{x}}({\cal D}_{\setminus\nu})\})={\cal I}(\widehat{\bm{x}}({\cal D}))$ holds, where
\begin{eqnarray}
\nonumber
&{\cal I}_{ij}^{\setminus\cdot}(\{\widehat{\bm{x}}({\cal D}_{\setminus\nu})\})=\frac{1}{M}\sum_{\mu=1}^M\frac{\partial}{\partial x_i}\ln f(y_\mu|\theta_\mu(\mathbf{f}_\mu,\bm{x}))\Big|_{\bm{x}
=\widehat{\bm{x}}({\cal D}_{\setminus\mu})}\\
&\hspace{4.0cm}\times\frac{\partial}{\partial x_j}\ln f(y_\mu|\theta_\mu(\mathbf{f}_\mu,\bm{x}))\Big|_{\bm{x}=\widehat{\bm{x}}({\cal D}_{\setminus\mu})}.
\label{eq:def_I_cavity}
\end{eqnarray}
\end{description}
Under assumption {\bf A3}, substituting (\ref{eq:diff_cavity_final}) into (\ref{eq:lnf_expand}) yields
\begin{eqnarray}
\mathrm{err}_{\mathrm{LOOCV}}({\cal D})=\mathrm{err}_{\mathrm{train}}({\cal D})+\frac{1}{M}\mathrm{Tr}\left({\cal I}(\widehat{\bm{x}}({\cal D}))
({\cal J}^{\varphi}(\widehat{\bm{x}}({\cal D})))^{-1}\right),
\label{eq:gap_LOOCV}
\end{eqnarray}
where $\mathrm{Tr}(\cdot)$ denotes the summation of the diagonal components.
The expression (\ref{eq:gap_LOOCV}) is equivalent to 
TIC (\ref{eq:def_TIC}) for the ML estimator 
where ${\cal J}^{\varphi}={\cal J}$.


The equivalence between LOOCV error and TIC
under the assumptions {\bf A1}-{\bf A3} is 
called asymptotic equivalence in the terminology of statistics.
We mention the case where these assumptions are valid.
The assumption {\bf A1} holds when $N$ is sufficiently small. 
At $N\to\infty$, the higher order terms with respect to the difference
$\widehat{x}_i({\cal D})-\widehat{x}_i({\cal D}_{\setminus\nu})$
have a contribution 
onto the difference in the log-likelihoods,
as shown in \ref{sec:app_expansion_JtoK},
and the correction of (\ref{eq:diff_expand}) is required.
The assumptions {\bf A2} and {\bf A3} holds when $M$ is sufficiently large 
compared with $N$.
In addition, for the derivation of (\ref{eq:diff_cavity_final})
under the assumption {\bf A2}, 
the matrix ${\cal J}^\varphi$ needs to be invertible.
This condition is violated in the singular models
such as mixture models, matrix factorization, and deep neural networks.
For such problems, TIC cannot be applied in the original form,
but it is suggested that TIC discarding zero-eigenmodes well characterize the
generalization property of the deep neural network \cite{Thomas}.

\subsection{Mallows' $C_p$ and model complexity in linear models}

The model selection criteria explained thus far are for models constructed 
using the maximum likelihood method. 
A general expression of the estimator for the in-sample generalization error
is derived from
another model selection criterion, that is, Mallows' $C_p$ \cite{Mallows}.
We consider a linear estimation rule where 
$\widehat{\bm{y}}=\bm{H}\bm{y}$ holds with a matrix $\bm{H}$,
and Gaussian likelihood with variance $\sigma^2$.
\begin{dfn}[$C_p$ criterion and degrees of freedom]
$C_p$ criterion is an estimator for the in-sample prediction error defined by
\begin{eqnarray}
{c_p}(\bm{y})=2\sigma^2{\rm err}_{\rm
 train}(\bm{y})+2{\sigma^2}\mathrm{df}.
 \label{eq:cp_def}
\end{eqnarray}
The term df is called {\it degrees of freedom} (DF) \cite{Hastie_Tibshirani1990,ESL}
defined by
\begin{eqnarray}
\mathrm{df}=\frac{1}{M}\mathrm{Tr}(\bm{H}).
\label{eq:df_def}
\end{eqnarray}
\end{dfn}

Under the known variance assumption,
$C_p$ criterion divided by $\sigma^2$
is reduced to AIC.
To demonstrate this, let us apply a singular value decomposition to $\bm{F}$
as $\bm{F}=\bm{U}\bm{D}\bm{V}^{\top}$, where
$\bm{U}\in\mathbb{R}^{M\times N}$ and $\bm{V}\in\mathbb{R}^{N\times N}$ 
represent the left and right singular vectors, respectively,
and $\bm{D}$ represents the $N\times N$ dimensional diagonal matrix whose $(i,i)$-component
corresponds to the $i$-th singular value $d_i$.
We obtain $\mathrm{Tr}(\bm{H})=\sum_{i=1}^Nd_i^2\slash d_i^2=N$,
which indicates the equivalence of the $C_p$ criterion and AIC.

$C_p$ can be extended to general estimators $\widehat{\bm{x}}({\cal D})$
in regression problems.
The following theorem shown in \cite{Efron2004} is the basis for the extension:
\begin{thm}
Let us consider the case where the following conditions hold:
\begin{itemize}
\item The data generative process is given by Gaussian distribution
$\bm{y}\sim{\cal N}(\bm{\mu},\sigma^2\bm{I}_M)$,
where $\bm{I}_M$ represents the $M$-dimensional identity matrix.
\item The likelihood is given by Gaussian distribution as
\begin{eqnarray}
f(\bm{y}|\bm{\theta})\propto\exp\left(-\frac{||\bm{y}-\bm{\theta}||_2^2}{2\sigma^2}\right)
\end{eqnarray}
\item The variance $\sigma^2$ is known.
\end{itemize}
The following estimator
\begin{eqnarray}
\widehat{\mathrm{Err}}({\cal D})=2\sigma^2\mathrm{err}_{\mathrm{train}}({\cal D})+\frac{2}{M}\widehat{\mathrm{Cov}}(\bm{y},\widehat{\bm{y}}({\cal D}))
\label{eq:Efron_estimator}
\end{eqnarray}
is an unbiased estimator of the in-sample prediction error
with the multiplying factor $2\sigma^2$
when $\widehat{\mathrm{Cov}}(\bm{y},\widehat{\bm{y}}({\cal D}))$
is an unbiased estimator of the covariance 
$\mathrm{Cov}_{\bm{y}}(\bm{y},\widehat{\bm{y}}({\cal D}))
=\mathbb{E}_{\bm{y}\sim q(\bm{y})}[\bm{y}^{\top}\widehat{\bm{y}}({\cal D})]-
\mathbb{E}_{\bm{y}\sim q(\bm{y})}[\bm{y}]^{\top}\mathbb{E}_{\bm{y}\sim q(\bm{y})}
[\widehat{\bm{y}}({\cal D})]$.
\label{thm:Cp}
\end{thm}
\begin{proof}
In the conditions considered in Theorem \ref{thm:Cp},
the in-sample prediction and training errors are respectively given by
\begin{eqnarray}
\mathrm{err}_{\mathrm{pre}}^{(\mathrm{in})}({\cal D})&=\frac{1}{2\sigma^2M}\mathbb{E}_{\bm{z}\sim q(\bm{z})}[||\bm{z}-\bm{\theta}(\bm{F},\widehat{\bm{x}}({\cal D}))||_2^2],\\
\mathrm{err}_{\mathrm{train}}({\cal D})&=\frac{1}{2\sigma^2M}||\bm{y}-\bm{\theta}(\bm{F},\widehat{\bm{x}}({\cal D}))||_2^2.
\end{eqnarray}
Using the relationship $\mathbb{E}_{\bm{z}\sim q(\bm{z})}[\bm{z}]=\mathbb{E}_{\bm{y}\sim q(\bm{y})}[\bm{y}]$,
the generalization gap is given by the form known as
covariance penalty \cite{Efron2004} 
\begin{eqnarray}
\Delta^{(\mathrm{in})}=\frac{1}{\sigma^2M}\mathrm{Cov}_{\bm{y}}\left(\bm{y},\widehat{\bm{y}}({\cal D})\right).
\label{eq:def_cov}
\end{eqnarray}
Therefore, we obtain
\begin{eqnarray}
2\sigma^2\mathrm{err}_{\mathrm{pre}}^{(\mathrm{in})}({\cal D})=
2\sigma^2\mathrm{err}_{\mathrm{train}}({\cal D})+\frac{2}{M}\mathrm{Cov}_{\bm{y}}\left(\bm{y},\widehat{\bm{y}}({\cal D})\right),
\end{eqnarray}
and this expression indicates that \eref{eq:Efron_estimator} is 
an unbiased estimator of the in-sample prediction error.
\end{proof}

For constructing the estimator for the covariance penalty,
the following identity that holds when 
$\bm{y}\sim{\cal N}(\bm{\mu},\sigma^2\bm{I}_M)$
is convenient, known as Stein's lemma \cite{SURE};
\begin{eqnarray}
\frac{1}{\sigma^2M}\mathrm{Cov}_{\bm{y}}\left(\bm{y},\widehat{\bm{y}}({\cal D})\right)=\frac{1}{M}\sum_{\mu=1}^M\mathbb{E}_{\bm{y}}\left[\frac{\partial\widehat{y}_\mu({\cal D})}{\partial y_\mu}\right].
\label{eq:cov_relationship}
\end{eqnarray}
For the linear estimation rule, we obtain
$\partial\widehat{y}_\mu({\cal D})\slash\partial y_\mu=H_{\mu\mu}$;
therefore, DF is an unbiased estimator of the covariance penalty.
We can obtain the estimator of the generalization gap
even when the analytical form of $\widehat{\bm{y}}$
cannot be derived using the relationship (\ref{eq:cov_relationship}).
Stein's unbiased risk estimate (SURE)
of the mean squared error \cite{SURE}
is an estimator of $2\sigma^2\mathrm{err}_{\mathrm{pre}}^{(\mathrm{in})}$ for the 
Gaussian likelihood applicable to
general estimators of $\widehat{\bm{y}}$.
\begin{dfn}[SURE and generalized degrees of freedom]
SURE is an estimator of the in-sample prediction error defined as
\begin{eqnarray}
\mathrm{SURE}({\cal D})=2\sigma^2\mathrm{err}_{\mathrm{train}}({\cal D})+2\sigma^2\mathrm{gdf}({\cal D}),
\end{eqnarray}
where $\mathrm{gdf}(\bm{y})$ is termed the 
generalized degree of freedom (GDF) \cite{Ye1998},
which is defined as
\begin{eqnarray}
\mathrm{gdf}({\cal D})=\frac{1}{M}\sum_{\mu=1}^M\frac{\partial\widehat{y}_\mu({\cal D})}{\partial y_\mu}.
\label{eq:GDF_def}
\end{eqnarray}
\end{dfn}
GDF coincides with DF for the linear estimator, as we have already observed.
The expression (\ref{eq:GDF_def})
is applicable to various models; however, 
the exact correspondence between GDF and the number of 
estimated variables in the model does not hold in general.
The functional form of $\widehat{y}_\mu({\cal D})$
depends on the true generative process of the data,
assumed model, and estimation method; therefore, 
GDF depends on these factors.

\subsubsection{Shrinkage estimators reduce model complexity}

A well-known tendency of GDF is that 
it decreases under shrinking estimators.
As an example, let us consider ridge regression
\begin{eqnarray}
\widehat{\bm{x}}({\cal D})=\mathop{\mathrm{argmax}}_{\bm{x}}\left\{-\frac{1}{2}\Big|\Big|\bm{y}-\frac{1}{\sqrt{N}}\bm{F}\bm{x}\Big|\Big|_2^2-\frac{\lambda}{2}||\bm{x}||_2^2\right\}.
\label{eq:ridge_def}
\end{eqnarray}
The estimate $\widehat{\bm{x}}$ and the
expression of data using the estimate $\widehat{\bm{y}}$ are provided by
\begin{eqnarray}
\widehat{\bm{x}}&=\frac{1}{\sqrt{N}}\left(\frac{1}{N}\bm{F}^{\top}\bm{F}+\lambda\bm{I}\right)^{-1}\bm{F}^{\top}\bm{y}\\
\widehat{\bm{y}}&=\bm{H}\bm{y}
\end{eqnarray}
where $\bm{H}=\frac{1}{N}\bm{F}\left(\frac{1}{N}\bm{F}^{\top}\bm{F}+\lambda\bm{I}\right)^{-1}\bm{F}^{\top}$.
Applying the singular value decomposition $\frac{1}{\sqrt{N}}\bm{F}=\bm{U}\bm{D}\bm{V}^{\top}$,
DF is given by
\begin{eqnarray}
\mathrm{df}=\frac{1}{M}\sum_{i=1}^N\frac{d_i^2}{d_i^2+\lambda},
\end{eqnarray} 
where $d_i$ represents the $i$-th singular value.
Therefore, 
DF decrease with an increase in the regularization parameter $\lambda$.

However, in case of the $\ell_1$-penalized maximum likelihood method,
GDF corresponds to the number of parameters in the model
although the estimator is shrinking from ordinary least square estimator 
\cite{L1_GDF,sparse_book,Sakata2016}.
This reason is considered that 
the selection of the non-zero components increase model complexity, which is an increase balanced with the decrease of the model complexity attributed to shrinkage.

\subsection{Importance sampling cross validation and functional variance}
\label{sec:ISCV_and_FV}

We introduce another expression of the LOOCV error
in the framework of the Bayesian inference; however, we consider the penalized maximum likelihood method in this paper.
In our problem setting, posterior distribution is defined by
\begin{eqnarray}
p_\beta(\bm{x}|{\cal D})=\frac{1}{Z_\beta}\exp(\beta\varphi({\cal D},\bm{x})),
\end{eqnarray}
where $Z_\beta = \int d\bm{x}\exp(\beta\varphi({\cal D},\bm{x}))$ and $\beta$ is the inverse temperature.
In the usual Bayesian inference, $\beta=1$ is considered.
This parameter connects the Bayesian inference and the 
penalized maximum likelihood method as 
\begin{eqnarray}
\widehat{x_i}({\cal D})=\lim_{\beta\to\infty}\int d\bm{x} x_ip_\beta(\bm{x}|{\cal D}).
\end{eqnarray}
We introduce the importance sampling cross validation (ISCV) error as follows.
\begin{dfn}[Importance sampling cross-validation error]
Under the posterior distribution $p_{\beta}(\bm{x}|{\cal D})$,
ISCV error is defined by
\begin{eqnarray}
\mathrm{err}_{\mathrm{ISCV}}({\cal D};\beta)=\frac{1}{M}\sum_{\mu=1}^M\ln\int d\bm{x}\frac{1}{f^\beta(y_\mu|\theta_\mu(\bm{x},\mathbf{f}_\mu))}p_{\beta}(\bm{x}|{\cal D}).
\end{eqnarray}
\end{dfn}
\begin{thm}
ISCV error is equivalent to the LOOCV error
for the Bayesian inference at any $\beta$ \cite{Gelfand_etal,Aki}.
\end{thm}
\begin{proof}
We start with the LOOCV error for the Bayesian inference
\begin{eqnarray}
\mathrm{err}_{\mathrm{LOOCV}}({\cal D};\beta)=-\frac{1}{M}
\sum_{\mu=1}^M\ln\int d\bm{x}f(y_\mu|\theta_\mu(\bm{x},\mathbf{f}_\mu))p_\beta(\bm{x}|{\cal D}_{\setminus\mu}),
\label{eq:def_CV_Bayes}
\end{eqnarray}
which is reduced to (\ref{eq:def_CV}) at $\beta\to\infty$.
Setting $\phi(\bm{x})=\exp(-h(\bm{x}))$,
(\ref{eq:def_CV_Bayes}) is transformed as 
\begin{eqnarray}
\nonumber
\mathrm{err}_{\mathrm{LOOCV}}({\cal D};\beta)&=\frac{1}{M}\sum_{\mu=1}^M\ln\frac{\int d\bm{x}\prod_{\nu\in\bm{M}\setminus\mu}f^\beta(y_\nu|\theta_\nu(\mathbf{f}_\nu,\bm{x}))\phi^\beta(\bm{x})}{\int d\bm{x}\prod_{\nu\in\bm{M}}f^\beta(y_\nu|\theta_\nu(\mathbf{f}_\nu,\bm{x}))\phi^\beta(\bm{x})}\\
&=\frac{1}{M}\sum_{\mu=1}^M\ln\int d\bm{x}\frac{1}{f^\beta(y_\mu|\theta_\mu(\mathbf{f}_\mu,\bm{x}))}
p_{\beta}(\bm{x}|{\cal D}),
\end{eqnarray}
and the R.H.S. corresponds to ISCV error.
\end{proof}
For the sake of generality, we define the 
posterior expectation of the $\eta$-powered likelihood as 
\begin{eqnarray}
{\cal T}(\eta)=\lim_{\beta\to\infty}\frac{1}{M\beta}\sum_{\mu=1}^M\ln\int d\bm{x} p_{\beta}(\bm{x}|{\cal D}) f^{\eta\beta}(y_\mu|\theta_\mu(\mathbf{f}_\mu,\bm{x})).
\end{eqnarray}
The training and ISCV errors correspond to $-{\cal T}(1)$ and ${\cal T}(-1)$, respectively.
Applying Maclaurin expansion to ${\cal T}(\eta)$ as
\begin{eqnarray}
{\cal T}(\eta)=\sum_{k=0}^\infty\frac{1}{k!}\frac{\partial^k{\cal T}(\eta)}{\partial \eta^k}\Big|_{\eta=0}\eta^k,
\label{eq:T_expansion}
\end{eqnarray}
and setting $\eta=\pm1$, we obtain
\begin{eqnarray}
\mathrm{err}_{\mathrm{train}}({\cal D})&=-\frac{\partial{\cal T}(\eta)}{\partial\eta}\Big|_{\eta=0}-\frac{1}{2}\frac{\partial^2{\cal T}(\eta)}{\partial\eta^2}\Big|_{\alpha=0}
-\frac{1}{6}\frac{\partial^3{\cal T}(\eta)}{\partial\eta^3}\Big|_{\alpha=0}+O(\eta^4)\\
\mathrm{err}_{\mathrm{ISCV}}({\cal D})&=-
\frac{\partial{\cal T}(\eta)}{\partial\eta}\Big|_{\eta=0}+
\frac{1}{2}\frac{\partial^2{\cal T}(\eta)}{\partial\eta^2}\Big|_{\alpha=0}
-\frac{1}{6}\frac{\partial^3{\cal T}(\eta)}{\partial\eta^3}\Big|_{\alpha=0}+O(\eta^4),
\end{eqnarray}
where ${\cal T}(0)=0$.
Thus, the following relationship holds:
\begin{eqnarray}
\mathrm{err}_{\mathrm{ISCV}}({\cal D})=\mathrm{err}_{\mathrm{train}}({\cal D})+
\frac{\partial^2{\cal T}(\eta)}{\partial\eta^2}\Big|_{\eta=0}+O(\eta^4).
\end{eqnarray}
The second derivative of ${\cal T}$ is given by
\begin{eqnarray}
\nonumber
\frac{\partial^2}{\partial\eta^2}{\cal T}(\eta)&=\beta\sum_{\mu=1}^M
\Big[
\frac{\mathbb{E}_{\mathrm{post}}\left[\left(\ln f(y_\mu|\theta_\mu)\right)^2\exp(\eta\beta\ln f(y_\mu|\theta_\mu))\right]}{\mathbb{E}_{\mathrm{post}}\left[\exp(\eta\beta\ln f(y_\mu|\theta_\mu))\right]}\\
&-\left\{\frac{\mathbb{E}_{\mathrm{post}}\left[\ln f(y_\mu|\theta_\mu)\exp(\eta\beta\ln f(y_\mu|\theta_\mu))\right]}{\mathbb{E}_{\mathrm{post}}\left[\exp(\eta\beta\ln f(y_\mu|\theta_\mu))\right]}\right\}^2
\Big],
\end{eqnarray}
which is called {\it functional variance} at $\eta = 0$ \cite{WAIC}.
\begin{dfn}[Functional variance and widely applicable information criterion]
The functional variance (FV) is defined by the variance of the component-wise
log-likelihood as
\begin{eqnarray}
\nonumber
\mathrm{FV}({\cal D})=\frac{\beta}{M}\sum_{\mu=1}^M
\Big\{\int d\bm{x} & p_\beta(\bm{x}|{\cal D})\left(\ln f(y_\mu|\theta_\mu(\mathbf{f}_\mu,\bm{x}))\right)^2\\
&-
\left(\int d\bm{x} p_\beta(\bm{x}|{\cal D})\ln f(y_\mu|{\theta}_\mu(\mathbf{f}_\mu,\bm{x}))\right)^2
\Big\}.
\end{eqnarray}
The estimator of the ISCV using the functional variance is
known as the widely applicable information criterion (WAIC)
\begin{eqnarray}
\mathrm{WAIC}({\cal D})=\mathrm{err}_{\mathrm{train}}({\cal D})+\mathrm{FV}({\cal D}).
\end{eqnarray}
\end{dfn}
The terms of $O(\eta^3)$ have a 
negligible contribution in the asymptotic limit,
and hence, WAIC is asymptotically equivalent to the LOOCV 
or ISCV error \cite{Watanabe_CV}.
However, for cases with a large number of parameters,
the convergence of the series expansion (\ref{eq:T_expansion})
should be considered,
which immediately concerns the adequacy of WAIC.

For the understanding of the FV,
we introduce the following form of FV;
\begin{eqnarray}
\mathrm{FV}({\cal D})=\frac{1}{M}\sum_{\mu=1}^M\frac{\partial^2}{\partial \gamma_\mu^2}\ln Z_{\beta,\bm{\gamma}}({\cal D})\Big|_{\bm{\gamma}=1},
\label{eq:FV_derivative}
\end{eqnarray}
where 
\begin{eqnarray}
Z_{\beta,\bm{\gamma}}({\cal D})=\int d\bm{x}f^{\beta \gamma_\mu}(y_\mu|\theta_\mu(\mathbf{f}_\mu,\bm{x}))
\phi^\beta(\bm{x})
\end{eqnarray}
is the extended partition function with an auxiliary variable $\bm{\gamma}\in\mathbb{R}^M$ \cite{Iba-Yano2022}.
This form is useful for the derivation of FV in the framework of GAMP later.

\subsubsection{Equivalence of TIC and WAIC}

The derivations of TIC and WAIC from the LOOCV error differ from each other.
TIC is obtained by removing one observation and assuming that the
difference in estimates due to the remove is sufficiently small.
Meanwhile, WAIC considers
component-wise variance of the log-likelihood.
However, it can be shown that TIC and WAIC are equivalent to each other,
as explained in \cite{watanabe2010equations},
when the Laplace approximation is valid.
By expanding the log-likelihood around the MAP estimate 
$\widehat{\bm{x}}({\cal D})$, 
we obtain
\begin{eqnarray}
\nonumber
\ln f(y_\mu|\theta_\mu&(\mathbf{f}_\mu,\bm{x}))=\ln f(y_\mu|\theta_\mu(\mathbf{f}_\mu,\widehat{\bm{x}}({\cal D})))\\
&+\sum_{i=1}^N\frac{\partial}{\partial x_i}\ln f(y_\mu|\theta_\mu(\mathbf{f}_\mu,\bm{x}))\Big|_{\bm{x}=\widehat{\bm{x}}({\cal D})}(x_i-\widehat{x}_{i}({\cal D})).
\label{eq:TIC_WAIC_1}
\end{eqnarray}
The function to be minimized is expanded as
\begin{eqnarray}
\nonumber
\varphi({\cal D},\bm{x})&=\varphi({\cal D},\widehat{\bm{x}}({\cal D}))\\
&+\frac{1}{2}\sum_{i,j}\frac{\partial^2}{\partial x_i\partial x_j}
\varphi({\cal D},\bm{x})\Big|_{\bm{x}=\widehat{\bm{x}}({\cal D})}
(x_i-\widehat{x}_i({\cal D}))(x_j-\widehat{x}_j({\cal D})),
\end{eqnarray}
where we consider that the gradient at $\widehat{\bm{x}}$ is zero.
Hence, we obtain the Laplace approximation of the posterior distribution as
\begin{eqnarray}
p_{\beta}(\bm{x}|{\cal D})\propto\exp\left(-\frac{\beta M}{2}(\bm{x}-\widehat{\bm{x}}({\cal D})){\cal J}^{\varphi}(\widehat{\bm{x}}({\cal D}))(\bm{x}-\widehat{\bm{x}}({\cal D}))^{\mathrm{T}}\right),
\label{eq:TIC_WAIC_2}
\end{eqnarray}
where ${\cal J}^{\varphi}(\widehat{\bm{x}})$ is given by (\ref{eq:def_J}).
Using (\ref{eq:TIC_WAIC_1}) and (\ref{eq:TIC_WAIC_2}), we obtain
\begin{eqnarray}
\nonumber
\mathrm{FV}({\cal D})&=\frac{\beta}{M}\sum_{\mu=1}^M\sum_{ij}\int d\bm{x} p_{\beta}(\bm{x}|{\cal D})\frac{\partial}{\partial x_i}\ln f(y_\mu|\theta_\mu(\mathbf{f}_\mu,\bm{x}))\Big|_{\bm{x}=\widehat{\bm{x}}({\cal D})}\\
\nonumber
&\hspace{1.0cm}\times\frac{\partial}{\partial x_j}\ln f(y_\mu|\theta_\mu(\mathbf{f}_\mu,\bm{x}))\Big|_{\bm{x}=\widehat{\bm{x}}({\cal D})}(x_i-\widehat{x}_i({\cal D}))(x_j-\widehat{x}_j({\cal D}))\\
&=\frac{1}{M}\mathrm{Tr}\left({\cal I}(\widehat{\bm{x}}({\cal D})){\cal J}^{-1}(\widehat{\bm{x}}({\cal D}))\right),
\end{eqnarray}
which corresponds to the bias term of TIC.



\section{Our purpose and problem setting}
\label{sec:problem_setting}

In the conventional problem settings in statistics,
$M\to\infty$ with sufficiently small $N$ are considered,
hence the difference between the in-sample prediction error and 
the extra-sample prediction error has not been focused on.
Recently, characteristic phenomena known as double descent phenomena,
coined by \cite{Belkin2019},
wherein the extra-sample prediction error exhibits two minima:
one each in the underparameterized and overparametrized regions,
has been studied \cite{surprise}.
These phenomena were reported in 
\cite{Opper1990,Krogh1992,Opper1996} for linear models,
and have since been reported \cite{Advani2020,Geiger2019}
for neural networks, attracting
considerable attention.
For a broader understanding of these phenomena,
some studies are attempted on the random features model
\cite{Mei-Montanari,Gerace2020,DoubleTrouble,Loureiro2022}.
In general, the modern neural network considers the overparametrized region,
and some studies attempted to extend the conventional statistical 
methods to the overparametrized region \cite{Thomas}.
Following such direction,
we focus on the differences between the estimator of the prediction errors,
namely,
$C_p$, TIC or WAIC, and  LOOCV error,
in the parameter region $M\sim O(N)$, including the 
regime where the number of parameters is larger than the data dimension.
In the subsequent sections, we derive 
their forms using the generalized approximate message passing (GAMP)
algorithm \cite{GAMP,Sakata2018},
and discuss their behavior in the $M\sim O(N)$ region.

Hereafter, we denote $\alpha=M\slash N$.
For the dense regression coefficient,
$\alpha$ corresponds to the ratio of data dimension to 
the number of parameters, and 
$\alpha\gg 1$ is the parameter region where the 
conventional studies of statistics focused on.
We focus on the finite $\alpha$ region where $\alpha\sim O(1)$.
In this paper, we consider the generalized linear model (GLM),
and demonstrate the results for the penalized linear regression
and penalized logistic regression.
The forms shown in this study can be applied to general
exponential family distributions.

We here summarize our contribution shown in the following sections.
\begin{itemize}
\item We derive the forms of the estimators for the generalization gaps using GAMP,
which is applicable to any $\alpha$ region 
as long as the local stability condition of GAMP is not violated.
\item We show that GDF and FV are related to the variance of the 
estimated parameter in the framework of GAMP.
\item Using the expression by GAMP,
it is shown that the convergence condition of the
series expansion (\ref{eq:T_expansion}), which is the basis of WAIC,
is related to GDF in case of the Gaussian likelihood.
\item The resulting forms of the estimators for the in-sample prediction error
have differences with the extra-sample prediction error 
in the finite $\alpha$ region.
We derive the differences in the finite $\alpha$ region
using the expression by GAMP.
\end{itemize}

\subsection{Problem setting}

We consider GLM as the likelihood where 
the assumed model is provided by the natural exponential family 
distribution;
\begin{eqnarray}
\ln f(\bm{y}|\bm{\theta})=\bm{\theta}^{\top}\bm{y}-a(\bm{\theta})+b(\bm{y}),
\end{eqnarray}
where 
$a(\bm{\theta})=\sum_{\mu=1}^Ma(\theta_\mu)$
and $b(\bm{y})=\sum_{\mu=1}^Mb(y_\mu)$.
In the natural exponential family,
$\bm{\theta}$ relates to the statistics of the distribution as 
\begin{eqnarray}
\mathbb{E}_{{\cal L}(\bm{\theta})}[y_\mu]&=\frac{\partial}{\partial\theta_\mu}a(\theta_\mu),~~~
\mathbb{V}_{{\cal L}(\bm{\theta})}[y_\mu]&=\frac{\partial^2}{\partial^2\theta_\mu}a(\theta_\mu),
\end{eqnarray}
where $\mathbb{E}_{{\cal L}(\bm{\theta})}[\cdot]$ and $\mathbb{V}_{{\cal L}(\bm{\theta})}[\cdot]$
denote the mean and variance according to the likelihood $f$
with the parameter $\bm{\theta}$. 
In GLM, the model parameter and the mean of the distribution are related to each other
through link function $\ell$ as
\begin{eqnarray}
\ell(\mathbb{E}_{{\cal L}(\bm{\theta})}[y_\mu])=\theta_\mu,~~~
\ell^{-1}(\theta_\mu)=\mathbb{E}_{{\cal L}(\bm{\theta})}[y_\mu].
\end{eqnarray}
The distribution is specified by the functions $a(\cdot)$ and $b(\cdot)$.
For example, the linear regression with Gaussian likelihood 
corresponds to 
\begin{eqnarray}
a(\bm{\theta})&=\frac{1}{2}||\bm{\theta}||_2^2,~~~
b(\bm{y})&=-\frac{1}{2}||\bm{y}||_2^2,
\end{eqnarray}
and the link function is given by the identity function $\ell(\theta)=\theta$.
Other representative GLMs are 
presented in \ref{sec:app_GLM}.
The generalization gap of the in-sample prediction error 
for GLM
is written in the same form as that of (\ref{eq:def_cov}):
\begin{eqnarray}
\nonumber
\mathrm{err}_{\mathrm{train}}({\cal D})&=-\frac{1}{M}\ln f(\bm{y}|\bm{\theta}(\bm{F},\widehat{\bm{x}}({\cal D})))\\
&=\frac{1}{M}\left\{-\bm{\theta}(\bm{F},\widehat{\bm{x}}({\cal D}))^{\top}\bm{y}+a(\bm{\theta}(\bm{F},\widehat{\bm{x}}({\cal D})))-b(\bm{y})\right\}\\
\nonumber
{\rm err}_{\rm pre}^{(\mathrm{in})}({\cal D})&=-\frac{1}{M}\mathbb{E}_{\bm{z}\sim q(\bm{z})}[\ln f(\bm{z}|\bm{\theta}(\bm{F},\widehat{\bm{x}}({\cal D})))]\\
\nonumber
&=\frac{1}{M}\Big\{-\bm{\theta}(\bm{F},\widehat{\bm{x}}({\cal D}))^{\top}\mathbb{E}_{\bm{z}\sim q(\bm{z})}[\bm{z}]+a(\bm{\theta}(\bm{F},\widehat{\bm{x}}({\cal D})))\\
&\hspace{5.0cm}-\mathbb{E}_{\bm{z}\sim q(\bm{z})}[b(\bm{z})]\Big\},
\end{eqnarray}
hence
\begin{eqnarray}
\Delta^{(\mathrm{in})}({\cal D})=\frac{1}{M}\mathrm{Cov}_{\bm{y}}(\bm{y},|\bm{\theta}(\bm{F},\widehat{\bm{x}}({\cal D}))).
\end{eqnarray}

For the penalty term, we restrict our discussion to the separable 
penalties, where $h(\bm{x})=\sum_{i=1}^Nh(x_i)$.
We focus on the 
elastic net penalty
\begin{eqnarray}
h(x)=\lambda_1|x|+\frac{\lambda_2}{2}x^2,
\label{eq:def_EN_penalty}
\end{eqnarray}
where $\lambda_1$ and $\lambda_2$ are regularization
parameters for controlling the intensity of the penalty.
The elastic net penalty at $\lambda_1=0$ and $\lambda_2=0$ correspond to
$\ell_2$ and $\ell_1$ norm penalty, respectively.

\section{GAMP for GLM}
\label{sec:GAMP}

Message passing is an algorithm used to approximately calculate the 
marginal posterior distribution.
The algorithm is based on the 
graphical representation of the 
conditional independency of the variables according to the 
probability distribution 
\cite{Pearl,PGM}.
Graphical models provide visual support to draw inferences and 
construct efficient algorithms.
In practice, message passing shows efficient performances
as a decoding algorithm for
low density parity check (LDPC) code \cite{Richardson_Urbanke},
code division multiple access (CDMA) \cite{Kabashima},
and others.

Message passing with a Gaussian approximation is known as
approximate message passing (AMP) \cite{DMM} that reduces the
computational cost required for the inference.
The AMP was proposed in the study on compressed sensing
and it was extended to a generalized approximate message passing (GAMP) \cite{GAMP} to deal with general separable likelihood and prior.
The AMP has a similarity between 
the cavity approach \cite{beyond} and Thouless--Anderson--Palmer (TAP)
equation \cite{TAP} in statistical physics for spin glasses.
This relationship induces a statistical-physics based study on the
information science and machine learning \cite{Mezard2009} such as
compressed sensing \cite{Krzakala2012},
regression \cite{Gerbelot},  and matrix factorization \cite{MF}.
The details of the derivation of the GAMP for the 
regression problem of the separable likelihoods and priors are 
shown in \ref{sec:app_AMP}.

The message passing procedure provides the exact marginal distribution when the 
graph consists of tree structures.
Our problem setting corresponds to the case where the graph has loops, hence
the marginal distribution calculated by the message passing is 
an approximation.
More precisely, the resulting marginal posterior distributions are 
locally consistent marginals, and 
do not necessarily constitute an overall posterior distribution \cite{Mezard2009}.
Furthermore, it is known that 
GAMP with random initialization
only provides an accurate approximation of the posterior distribution 
only in the case in which the global minima of the free energy is the closest minima to zero overlap \cite{zdeborova2016statistical,Loureiro2021}.
Although a lack of the exactness and limitations exist in
accurate approximation,
GAMP is useful for calculating the generalization gap.
In the estimation of the in-sample prediction error,
analytical expression of the degrees of freedom 
for general regression problems are limited, and 
specific expressions are studied for LASSO \cite{L1_GDF},
group lasso, fused lasso \cite{Kato2009},
and minimax concave penalty \cite{Mazumder}.
Numerically, the bootstrap method is useful for evaluating the generalization gap 
\cite{Efron2004};
however, its computational cost is expensive, and it 
often leads unstable estimates.
Extension of the path seeking algorithm \cite{Friedman2012}
is another numerical approach that
is applied for estimating the generalization gap \cite{Hirose2013},
which requires additional matrix operation.
Compared with these numerical approaches,
GAMP has several advantages as an estimation method of
the generalization gap for the in-sample prediction error.
First, GAMP provides a general expression of the 
generalization gap for the in-sample prediction error.
This expression leads to a unified interpretation of the 
prediction as a
relationship between the fluctuation and response,
as discussed in sec.\ref{sec:GDF_fluctuation}.
Second, GAMP does not require additional computation
for evaluating the generalization gap for the in-sample prediction error,
as the generalization gap can be expressed 
using variables that constitute GAMP.

In addition to the advantages of evaluating the in-sample prediction error,
we can obtain LOOCV error using GAMP by
assuming that the cavity distribution in GAMP corresponds to the 
posterior distribution under the LOO sample.
While the verification of the assumption is required, 
it greatly reduces the computational cost required to estimate the
LOOCV error, as discussed in sec.\ref{sec:correspondence_cavity}.

\subsection{General form of message passing algorithm}

\begin{figure}
\centering
\includegraphics[width=4in]{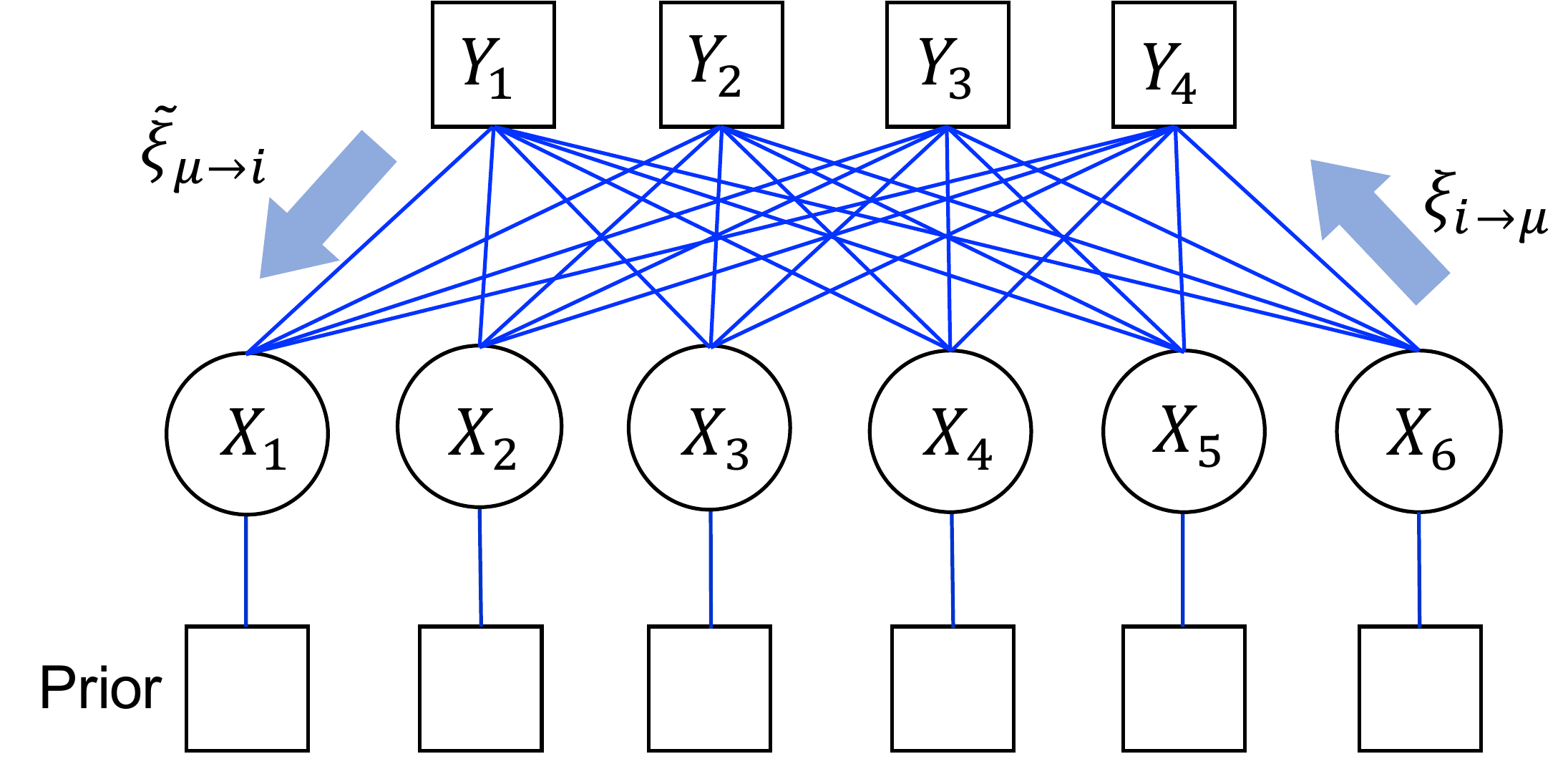}
\caption{Graphical representation of regression problem
and messages defined on the edges}
\label{fig:graphical_representation}
\end{figure}

The graphical representation of the regression problems considered here is shown in
\Fref{fig:graphical_representation}.
Factor nodes $\square$, variable nodes $\bigcirc$, and 
the edges represent the data $\bm{y}$, regression coefficient $\bm{x}$, and 
predictor matrix $\bm{F}$, respectively.
The message passing procedure is 
described by following two messages defined on the edges
between the factor nodes and the variable nodes:
\begin{itemize}
\item ${\xi}_{i\to \mu}(x_{i})$: Message from the variable $x_{i}$ to
data $y_{\mu}$ (output message)
\item $\widetilde{\xi}_{\mu \to i}(x_{i})$: Message from 
data $y_{\mu}$ to the variable $x_{i}$ (input message)
\end{itemize}
Intuitively, output and input messages are the
probability distributions before and after the observation of $\mu$-th data.
Hence, in particular, the output message ${\xi}_{i\to \mu}(x_{i})$ 
is considered to describe the probability distribution
under the LOO sample ${\cal D}_{\setminus\mu}$
\cite{Opper_Winther_GP}.
We term the graphical model without $\mu$-th data as $\mu$-cavity system.

For general graphs,
these messages need to be recursively updated
to obtain the posterior distribution.
Here, we place the superscript $(t)$ on the messages and related quantities 
to indicate that their values are the results of the $t$-times update.
The messages at the $t$-th and $(t+1)$-th step are related as
\begin{eqnarray}
\widetilde{\xi}^{(t)}_{\mu\to i}(x_{i})&\propto
\int d\bm{x}_{\backslash i}
f^\beta(y_{\mu}|\bm{\theta})\prod_{j\in\bm{N}\setminus i}
{\xi}^{(t)}_{j\to\mu}(x_{j})\label{eq:def_input}\\
{\xi}_{i\to \mu}^{(t+1)}(x_{i})&\propto
\phi^\beta(x_{i})\prod_{\nu\in\bm{M}\setminus \mu}
\widetilde{\xi}_{\nu\to i}^{(t)}(x_{i}),
\label{eq:def_output}
\end{eqnarray}
where $\bm{N}=\{1,\cdots,N\}$,
$\bm{M}=\{1,\cdots,M\}$, and 
$\bm{N}\setminus i$ denote $\bm{N}$ except $i$.
Using the input messages,
the marginal distribution is given by
\begin{eqnarray}
p^{(t)}_{\beta,i}(x_{i})\propto
\phi^\beta(x_{i})\prod_{\nu\in\bm{M}}
\widetilde{\xi}^{(t)}_{\nu\to i}(x_{i}),
\end{eqnarray}
and the $i$-th component of the penalized maximum likelihood 
estimate is given by
\begin{eqnarray}
\widehat{x}^{(t)}_{i}&=\lim_{\beta\to\infty}\int d{x_i}~x_ip_{\beta,i}^{(t)}(x_{i}).
\end{eqnarray}


\subsection{Approximate message passing}

The updating rule (\ref{eq:def_input}) and (\ref{eq:def_output})
consists of $2MN$ messages.
For simplification, we introduce the following assumptions for sufficiently large $M$ and $N$.
\begin{description}
\item[\bf{a1:}] 
All components of the predictor matrix are $O(1)$.

\item[\bf{a2:}]
The correlation between predictors is negligible:
the off-diagonal components of $\frac{1}{N}\bm{F}^\top\bm{F}$ are $O(N^{-1\slash 2})$.
\end{description}
Under assumptions {\bf a1} and {\bf a2},
the input messages $\{\widetilde{\xi}_{\mu\to i}\}$
can be considered Gaussian distributions, as shown in \ref{sec:app_AMP}.
Therefore, it is sufficient to consider the means and variances of the messages; we denote them as
\begin{eqnarray}
\widehat{x}_{i\to\mu}^{(t)}&=\int dx_i x_i\xi_{i\to\mu}^{(t)}(x_i)\\
s_{i\to\mu}^{(t)}&=\beta\int dx_i (x_i-\widehat{x}_{i\to\mu}^{(t)})^2\xi_{i\to\mu}^{(t)}(x_i)
\label{eq:s_def}\\
\widetilde{x}_{\mu\to i}^{(t)}&=\int dx_i x_i\widehat{\xi}_{\mu\to i}^{(t)}(x_i)\\
\widetilde{s}_{\mu\to i}^{(t)}&=\beta\int dx_i (x_i-\widetilde{x}_{\mu\to i}^{(t)})^2\widetilde{\xi}_{\mu\to i}^{(t)}(x_i).\label{eq:s_tilde_def}
\end{eqnarray}
An algorithm for updating these means and variances
under the arbitrary likelihood and prior distribution
is called generalized approximate message passing (GAMP).
In GAMP, the marginalized distribution $p_{\beta,i}(x_i)$
is represented by
\begin{eqnarray}
p_{\beta,i}^{(t)}(x_i)={\cal P}_\beta(x_i;\Sigma_i^{(t)},\mathrm{m}_i^{(t)}),
\label{eq:eff_marginal}
\end{eqnarray}
where ${\cal P}_\beta$ is the distribution defined by
\begin{eqnarray}
{\cal P}_\beta(x;\Sigma,\mathrm{m})=\frac{1}{{\cal Z}}\phi^\beta(x)
\exp\left\{-\frac{\beta(x-\mathrm{m})^2}{2\Sigma}\right\},
\label{eq:posterior_AMP}
\end{eqnarray}
with the normalization constant ${\cal Z}$.
We define the mean and variance of (\ref{eq:posterior_AMP})
as $\mathbb{M}_\beta(\Sigma,\mathrm{m})$ and $\beta^{-1}\mathbb{V}_\beta(\Sigma,\mathrm{m})$,
respectively, and denote
\begin{eqnarray}
\widehat{x}_i^{(t)}&=\int dx_i x_i{\cal P}_\beta(x;\Sigma_i^{(t)},\mathrm{m}_i^{(t)})=\mathbb{M}_\beta(\Sigma_i^{(t)},\mathrm{m}_i^{(t)})\\
s_i^{(t)}&=\beta\int dx_i (x_i-\widehat{x}_i^{(t)})^2{\cal P}_\beta(x;\Sigma_i^{(t)},\mathrm{m}_i^{(t)})=\mathbb{V}_\beta(\Sigma_i^{(t)},\mathrm{m}_i^{(t)}),
\end{eqnarray}
where $\widehat{x}_i^{(t)}$ corresponds to the penalized maximum likelihood estimate
of step $t$ at $\beta\to\infty$,
which limit will be taken later.
Here, the following relationship holds.
\begin{eqnarray}
\mathbb{V}_{\beta}(\Sigma,\mathrm{m})=\beta^{-1}\Sigma\frac{\partial}{\partial \mathrm{m}}\mathbb{M}_{\beta}(\Sigma,\mathrm{m}).
\label{eq:VtoM}
\end{eqnarray}

The form of (\ref{eq:posterior_AMP}) indicates that 
the posterior distribution is given by the prior distribution for $x_i$ 
and the Gaussian distribution with mean $\mathrm{m}_i$ and 
variance $\Sigma_i$, which represents a contribution from other variables.
Here, we set 
\begin{eqnarray}
\widehat{\theta}_\mu^{\setminus\mu(t)}&=\frac{1}{\sqrt{N}}\sum_{i=1}^NF_{\mu i}\widehat{x}_{i\to\mu}^{(t)},\label{eq:GAMP_theta_cavity}\\
{s_{\theta}}_\mu^{\setminus\mu(t)}&=\frac{1}{N}\sum_{i=1}^NF_{\mu i}^2s_{i\to\mu}^{(t)},
\end{eqnarray}
where (\ref{eq:GAMP_theta_cavity}) corresponds to the 
estimated model parameter in the $\mu$-cavity system.
After some calculations shown in \ref{sec:app_AMP}, 
we derive $\mathrm{m}_i$ and $\Sigma_i$ as 
\begin{eqnarray}
\nonumber
\mathrm{m}_{i}^{(t)}&={\Sigma}_{i}^{(t)}\sum_{\mu=1}^M\Big\{
\frac{1}{\sqrt{N}}F_{\mu i}\beta^{-1}g_{\mathrm{out}}({s_{\theta}}_\mu^{\setminus\mu(t)},\widehat{\theta}_\mu^{\setminus\mu(t)},y_\mu)\\
&\hspace{3.0cm}-\frac{1}{N}F_{\mu i}^{2}\widehat{x}_{i}^{(t)}
\beta^{-1}\partial_{\widehat{\theta}} g_{\mathrm{out}}({s_{\theta}}_\mu^{\setminus\mu(t)},\widehat{\theta}_\mu^{\setminus\mu(t)},y_\mu)\Big\}\\
\frac{1}{{\Sigma}_{i}^{(t)}}&=-\frac{1}{N}\sum_{\mu=1}^MF^{2}_{\mu i}\beta^{-1}\partial_{\widehat{\theta}}g_{\mathrm{out}}({s_{\theta}}_\mu^{\setminus\mu(t)},\widehat{\theta}_\mu^{\setminus\mu(t)},y_\mu),
\end{eqnarray}
where 
\begin{eqnarray}
g_{\mathrm{out}}({s_{\theta}}_\mu^{\setminus\mu(t)},\widehat{\theta}_\mu^{\setminus\mu(t)},y_\mu)&=\Big\langle\frac{\beta(\theta_\mu-\widehat{\theta}_\mu^{\setminus\mu(t)})}{{s_{\theta}}_\mu^{\setminus\mu(t)}}\Big\rangle_{\theta_\mu}
\label{eq:g_out_def}\\
\partial_{\widehat{\theta}}g_{\mathrm{out}}({s_{\theta}}_\mu^{\setminus\mu(t)},\widehat{\theta}_\mu^{\setminus\mu(t)},y_\mu)
&=\frac{\partial}{\partial\widehat{\theta}_\mu^{\setminus\mu(t)}}g_{\mathrm{out}}({s_{\theta}}_\mu^{\setminus\mu(t)},\widehat{\theta}_\mu^{\setminus\mu(t)},y_\mu).
\end{eqnarray}
The expectation $\langle\cdot\rangle_{\theta_\mu}$ in (\ref{eq:g_out_def})
denotes the expectation with respect to $\theta_\mu$
based on the distribution
\begin{eqnarray}
p_{\mathrm{c},\mu}^{(t)}(\theta_\mu|y_\mu)=\displaystyle\frac{f^\beta(y_\mu|\theta_\mu)\psi^\beta_{\mathrm{c}}(\theta_\mu|{s_\theta}_{\mu}^{\setminus\mu(t)},\widehat{\theta}_{\mu}^{\setminus\mu(t)})}
{\int d\theta_\mu f^\beta(y_\mu|\theta_\mu)\psi^\beta_{\mathrm{c}}
(\theta_\mu|{s_\theta}_{\mu}^{\setminus\mu(t)},\widehat{\theta}_{\mu}^{\setminus\mu^{(t)}})},
\label{eq:eff_dist}
\end{eqnarray}
where
\begin{eqnarray}
\psi_{\mathrm{c}}(\theta_{\mu}|{s_\theta}_{\mu}^{\setminus\mu},\widehat{\theta}^{\setminus\mu}_{\mu})=\exp\left(-\frac{(\theta_{\mu}-\widehat{\theta}_{\mu}^{\setminus\mu})^2}{2{s_\theta}_{\mu}^{\setminus\mu}}\right).
\end{eqnarray}
Defining the following quantities
\begin{eqnarray}
\widehat{\theta}^{(t)}_\mu&=\frac{1}{\sqrt{N}}\sum_{i=1}^NF_{\mu i}\widehat{x}_i^{(t)}\label{eq:theta_hat_def}\\
{s_{\theta}}_\mu^{(t)}&=\frac{1}{N}\sum_{j=1}^NF_{\mu j}^2s_j^{(t)},
\end{eqnarray}
where (\ref{eq:theta_hat_def}) corresponds to the estimated model
parameter, their values are related to those for cavity distribution as 
\begin{eqnarray}
\widehat{\theta}_\mu^{(t)}&=\widehat{\theta}_{\mu}^{\setminus\mu(t)}
+
\beta^{-1}(g_{\mathrm{out}}^{(t-1)})_{\mu}{s_\theta}_{\mu}^{\setminus\mu(t)}
\label{eq:theta_hat}\\
{s_\theta}_{\mu}^{(t)}&={s_\theta}_{\mu}^{\setminus\mu(t)},
\label{eq:s_theta}
\end{eqnarray}
where the second term of (\ref{eq:theta_hat}) is known as Onsager reaction term \cite{TAP}.

\subsection{$\beta\to\infty$ limit}

At $\beta\to\infty$,
the saddle point method can be used to calculate the integrals with respect to $x$ and $\theta$ based on the 
distributions (\ref{eq:eff_marginal})
and (\ref{eq:eff_dist}).
First, the mean of the distribution (\ref{eq:eff_marginal})
denoted by $\widehat{x}$
is given as the solution of the maximization problem
\begin{eqnarray}
\widehat{x}_i=\mathop{\mathrm{argmax}}_xf_\xi(x;\mathrm{m}_i,\Sigma_i)
\label{eq:GAMP_x_max}
\end{eqnarray}
where
\begin{eqnarray}
f_\xi(x;\mathrm{m},\Sigma)=\ln\phi(x)-\frac{(x-\mathrm{m})^2}{2\Sigma}.
\end{eqnarray}
The solution depends on the prior distribution.
Here, we show the solution under the elastic net penalty as
\begin{eqnarray}
\widehat{x}_i&=\frac{\mathrm{m}_i-\mathrm{sgn}(\mathrm{m}_i)\lambda_1\Sigma_i}{\lambda_2\Sigma_i+1}
\mathbb{I}(|\mathrm{m}_i|>\lambda_1\Sigma_i),\\
s_i&=\frac{\Sigma_i}{\lambda_2\Sigma_i+1}
\mathbb{I}(|\mathrm{m}_i|>\lambda_1\Sigma_i),
\end{eqnarray}
where $\mathbb{I}(a)$ represents the indicate function returns 1 if $a$ is true, otherwise it is 0.
Setting $\lambda_1=0$ and $\lambda_2=0$, we obtain the 
solution under the $\ell_2$ penalty and $\ell_1$ penalty, respectively.

Next, the mean (\ref{eq:eff_dist})
at $\beta\to\infty$, denoted by 
$\theta_\mu^*$, is given by
\begin{eqnarray}
\theta_\mu^*&=\mathop{\mathrm{argmax}}_{\theta}
\widehat{f}_\xi (\theta;y_\mu,\widehat{\theta}_\mu^{\setminus\mu},{s_{\theta}}_\mu^{\setminus\mu}),
\label{eq:theta_max}
\end{eqnarray}
with the function $\widehat{f}_\xi$ defined as
\begin{eqnarray}
\widehat{f}_\xi(\theta;y,\widehat{\theta},s)=\ln f(y|\theta)-\frac{(\theta-\widehat{\theta})^2}{2s}.
\label{eq:opt_theta_func}
\end{eqnarray}
The maximizer $\theta_\mu^*$ in \eref{eq:theta_max} satisfies
\begin{eqnarray}
\frac{\theta_\mu^*-\widehat{\theta}_\mu^{\setminus\mu}}{{s_\theta}_\mu^{\setminus\mu}}=y_\mu-a^\prime(\theta_\mu^*),
\label{eq:theta_mu}
\end{eqnarray}
where $a^\prime(\theta_\mu^*)=\frac{\partial}{\partial\theta}a(\theta)|_{\theta=\theta_\mu^*}$ corresponds to the
mean of $y_\mu$ expressed by the assumed model,
i.e. $a^\prime(\theta_\mu^*)=\mathbb{E}_{{\cal L}(\bm{\theta}^*)}[y_\mu]=\ell^{-1}(\theta_\mu^*)$. From \eref{eq:theta_mu}, $\partial_{\hat{\theta}}\theta_\mu^*\equiv\partial\theta_\mu^*\slash\partial\widehat{\theta}_\mu^{\setminus\mu}$ satisfies 
\begin{eqnarray}
\frac{\partial_{\hat{\theta}}\theta_\mu^*-1}{{s_\theta}_\mu}=-\partial_{\widehat{\theta}}\theta_\mu^*a^{\prime\prime}(\theta_\mu^*),
\label{eq:partial_theta}
\end{eqnarray}
where $a^{\prime\prime}(\theta_\mu^*)$ corresponds to 
the variance of $y_\mu$ under the assumed model, i.e. 
$a^{\prime\prime}(\theta_\mu^*)=\mathbb{V}_{{\cal L}(\bm{\theta}^*)}[y_\mu]$.
Solving \eref{eq:partial_theta},
we obtain
\begin{eqnarray}
\frac{\partial_{\widehat{\theta}}\theta_\mu^*-1}{{s_\theta}_\mu}=-\frac{\mathbb{V}_{{\cal L}(\theta^*_\mu)}[y_\mu]}{1+{s_\theta}_\mu\mathbb{V}_{{\cal L}(\theta^*_\mu)}[y_\mu]}.
\end{eqnarray}
For convenience, we introduce rescaled variables
\begin{eqnarray}
\beta^{-1}(g_{\mathrm{out}})_\mu&=(\check{g}_{\mathrm{out}})_\mu,~~~
\beta^{-1}(\partial_{\hat{\theta}} g_{\mathrm{out}})_\mu&=(\check{\partial}_\theta g_{\mathrm{out}})_\mu,\label{eq:g_rescale}
\end{eqnarray}
and they are given by 
\begin{eqnarray}
(\check{g}_{\rm out})_\mu=\frac{\theta_\mu^*-\widehat{\theta}_\mu^{\setminus\mu}}{{s_\theta}_\mu},~~~(\check{\partial}_{\hat{\theta}}g_{\rm out})_\mu=\frac{\partial_{\hat{\theta}}\theta_\mu^*-1}{{s_\theta}_\mu},\label{eq:g_out_theta}
\end{eqnarray}
at $\beta\to\infty$. 
Based on scaling with respect to $\beta$,
(\ref{eq:s_def}), (\ref{eq:s_tilde_def}) and (\ref{eq:g_rescale}),
the variables in GAMP stay finite 
even at the limit $\beta\to\infty$.
This scaling can be understood through 
Laplace approximation of the posterior distribution
around the MAP estimates. 
From eqs.(\ref{eq:theta_mu}) and (\ref{eq:partial_theta}), we obtain
\begin{eqnarray}
(\check{g}_{\mathrm{out}})_\mu&=y_\mu-\mathbb{E}_{{\cal L}}(\theta_\mu^*)[y_\mu],\label{eq:check_g}\\
(\check{\partial}_{\widehat{\theta}}g_{\mathrm{out}})_\mu&=-\frac{\mathbb{V}_{{\cal L}(\theta^*_\mu)}[y_\mu]}{1+{s_\theta}_\mu\mathbb{V}_{{\cal L}
(\theta^*_\mu)}[y_\mu]}.
\label{eq:_check_g_out}
\end{eqnarray}
The solution $\theta^*_\mu$ of the maximization problem (\ref{eq:theta_max})
is equivalent to the estimate for the parameter $\widehat{\theta}_\mu$
defined by (\ref{eq:theta_hat_def}),
which can be confirmed by substituting (\ref{eq:check_g}) to (\ref{eq:theta_hat}).

The analytical form of the estimate $\bm{\theta}^*$ and related
quantities depend on the likelihood.
For Gaussian likelihood, the estimate is given by
\begin{eqnarray}
\theta_\mu^*=\frac{{s_\theta}_\mu y_\mu}{1+{s_\theta}_\mu}+
\frac{\widehat{\theta}_\mu^{\setminus\mu}}{1+{s_\theta}_\mu}.
\end{eqnarray}
Further, $\mathbb{E}_{{\cal L}}(\theta_\mu^*)[y_\mu]=\theta_\mu^*$ and 
$\mathbb{V}_{{\cal L}}(\theta_\mu^*)[y_\mu]=1$, and hence,
\begin{eqnarray}
(\check{g}_{\mathrm{out}})_\mu&=\frac{y_\mu-\widehat{\theta}_\mu^{\setminus\mu}}{1+{s_\theta}_\mu},~~~
(\check{\partial}_{\widehat{\theta}}g_{\mathrm{out}})_\mu&=-\frac{1}{1+{s_\theta}_\mu}.
\label{eq:_check_g_out_Gauss}
\end{eqnarray}
We summarise the GAMP solutions 
for representative distributions in the exponential family in
\ref{sec:app_GLM}.
The procedure of GAMP for general likelihoods and 
priors is summarised in Algorithm \ref{alg:AMP},
where $\circ$ and $\oslash$ denote
the component-wise product and division, respectively,
and $\bm{0}_{N\times 1}$ denotes the $N$-dimensional vector
whose components are zero.

\begin{algorithm}[h]
\caption{GAMP for penalized regression problems}
\label{alg:AMP}
\begin{algorithmic}[1]
\Require {${\cal D}=\{\bm{y},\bm{F}\}$}
\Ensure {$\widehat{\bm{x}}$}
\State{$\widehat{x}^{(-1)}\gets \bm{0}_{N\times 1}$, $\bm{g}^{(-1)}_{\mathrm{out}}\gets\bm{0}_{M\times 1}$}
\State{$\widehat{\bm{x}}^{(0)},~\bm{s}^{(0)}\gets$ Initial values from $\mathbb{R}^{N}$} 
\For{$t = 0 \, \ldots \, T_{\max}$}
\State{$\widehat{\bm{\theta}}^{(t)}\gets\frac{1}{\sqrt{N}}\bm{F}\widehat{\bm{x}}^{(t)}-\frac{1}{N}\bm{g}_{\mathrm{out}}^{(t-1)}\circ\left((\bm{F}\circ\bm{F})\bm{s}^{(t)}\right)$}
\State{$\bm{s}_\theta^{(t)}\gets\frac{1}{N}
(\bm{F}\circ\bm{F})\bm{s}^{(t)}$}
\State{$\bm{g}_{\mathrm{out}}^{(t)}\gets \bm{g}_{\mathrm{out}}(\bm{s}_\theta^{(t)},\widehat{\bm{\theta}}^{(t)};\bm{y})$}
\State{$\partial_{\hat{\theta}}\bm{g}_{\mathrm{out}}^{(t)}\gets \partial_{\widehat{\theta}}\bm{g}_{\mathrm{out}}(\bm{s}_\theta^{(t)},\widehat{\bm{\theta}}^{(t)};\bm{y})$}
\State{$\bm{\Sigma}^{(t)}\gets -\bm{1}_{N\times 1}\oslash
\left(\frac{1}{N}(\bm{F}\circ\bm{F})^{\top}\partial_{\hat{\theta}} \bm{g}_{\mathrm{out}}^{(t)}\right)$}
\State{$\bm{\mathrm{m}}^{(t)}\gets\bm{\Sigma}^{(t)}\circ
\left(\frac{1}{\sqrt{N}}\bm{F}^{\top}\bm{g}_{\mathrm{out}}^{(t)}-\frac{1}{N}\widehat{\bm{x}}^{(t)}\circ
((\bm{F}\circ\bm{F})^{\top}\partial_{\hat{\theta}} \bm{g}_{\mathrm{out}}^{(t)})\right)$}
\State{$\widehat{\bm{x}}^{(t+1)}\gets \mathbb{M}_\beta(\bm{\Sigma}^{(t)},\bm{\mathrm{m}}^{(t)})$}
\State{$\bm{s}^{(t+1)}\gets \mathbb{V}_\beta(\bm{\Sigma}^{(t)},\bm{\mathrm{m}}^{(t)})$}
\EndFor
\end{algorithmic}
\end{algorithm}

\subsubsection{Neglecting a part of site-dependency for random predictors}

For random predictors with i.i.d. random entries of zero mean and variance 1,
${s_\theta}_\mu$
converges to the averaged value of $s_i$
\begin{eqnarray}
{s_\theta}_\mu=\frac{1}{N}\sum_{i=1}^NF_{\mu i}^2s_i \to s_{\theta}=\frac{1}{N}\sum_{i=1}^Ns_i
\end{eqnarray}
at sufficiently large $N$ because of the law of large numbers.
Further, the site dependency of $\Sigma_i$ can be ignored as
\begin{eqnarray}
\frac{1}{{\Sigma}^{(t)}}&=-\frac{\alpha}{M}\sum_{\mu=1}^M\check{\partial}_{\widehat{\theta}}g_{\mathrm{out}}({s_{\theta}}_\mu^{\setminus\mu(t)},\widehat{\theta}_\mu^{\setminus\mu(t)},y_\mu).
\end{eqnarray}

\subsubsection{Comparison of GAMP with the linear algebraic expression}

Now, we compare the expression of the LOOCV error obtained using GAMP and that obtained from the linear algebraic form given by (\ref{eq:LOOCV_linear}) for the Gaussian likelihood.
In GAMP, we denote the model expression of data 
under the $\mu$-cavity system as
$\widehat{y}_\mu^{\setminus\mu}=\frac{1}{\sqrt{N}}F_{\mu i}\widehat{x}_{i\to\mu}$; this 
corresponds to $\widehat{\theta}_\mu^{\setminus\mu}$.
Hence, we obtain 
\begin{eqnarray}
y_\mu-\widehat{y}_\mu^{\setminus\mu}=(y_\mu-\widehat{y}_\mu)(1+{s_\theta}_\mu).
\end{eqnarray}
Meanwhile, from (\ref{eq:linear_CV_proof}), we obtain
\begin{eqnarray}
y_\mu-\widehat{y}_\mu({\cal D}_{\setminus\mu})=\frac{y_\mu-\widehat{y}_\mu({\cal D})}{1-H_{\mu\mu}}.
\end{eqnarray}
Assuming the equivalence between $\widehat{y}_\mu^{\setminus\mu}$ in GAMP
and $\widehat{y}_\mu({\cal D}_{\setminus\mu})$ in the exact calculation,
the correspondence between ${s_\theta}_\mu$ in GAMP
and the matrix $\bm{H}$ is given as
\begin{eqnarray}
H_{\nu\nu}\leftrightarrow\frac{{s_\theta}_\nu}{1+{s_\theta}_\nu}
\label{eq:hat_vs_GAMP}
\end{eqnarray}
for any $\nu\in\bm{M}$.
For computing the hat matrix in the linear estimator, 
we need to calculate the matrix inverse.
An advantage of GAMP is that there is no need to compute the inverse of the matrix
to obtain the estimate, although a recursive update is required.
Further, the R.H.S of (\ref{eq:hat_vs_GAMP})
can be calculated for any estimator,
even when the linear estimation rule is not considered and 
the matrix $\bm{H}$ is not defined;
hence, (\ref{eq:hat_vs_GAMP})
can be regarded as the extension of the matrix $\bm{H}$ 
for other estimators.

\subsubsection{Correspondence
between $\mu$-cavity system and LOO sample}
\label{sec:correspondence_cavity}

Here, we 
discuss the correspondence between 
$\widehat{x}_{i\to\mu}$ in GAMP
and exact $\widehat{x}_i({\cal D}_\mu)$.
When these estimates are equivalent to each other at any $i$ and $\mu$,
the parameters 
$\widehat{\theta}_\mu^{\setminus\mu}$
and $\widehat{\theta}({\cal D}_{\setminus\mu})$ coincide.
In the case where the graph structure consists of tree structures,
the exactness of the output message $\xi_{i\to\mu}(x_i)$ as 
the posterior distribution under the LOO sample ${\cal D}_{\setminus\mu}$
can be mathematically proved \cite{Mezard2009},
hence, $\widehat{x}_{i\to\mu}=\widehat{x}({\cal D}_{\setminus\mu})$
holds.
In the current problem setting,
the graph contains loops, and the correspondence does not hold in general.
Hereafter, we use $\widehat{\theta}_{\mu}^{\setminus\mu}$
as an estimated value of $\widehat{\theta}({\cal D}_{\setminus\mu})$,
and the related generalization gap of LOOCV error given by the estimated value 
is denoted as 
$\widehat{\Delta}_{\mathrm{LOOCV}}$.


\subsubsection{Local stability condition of GAMP}

The GAMP algorithm sometimes fails to converge 
depending on the model parameters $N$, $M$, and regularization parameters.
The instability condition can be derived by 
considering the linear stability analysis around the GAMP's fixed point.
We consider the case that $\widehat{\bm{x}}$ is at a fixed point with
the corresponding values of $\widehat{\theta}_\mu^{\setminus\mu}$ for $\mu\in\bm{M}$.
When $\widehat{\bm{x}}$ is slightly deviated from the value at the fixed point, 
the deviation propagates to $\widehat{\theta}_\mu^{\setminus\mu}$ by GAMP, 
and back to $\widehat{\bm{x}}$.
If estimates under the small deviance are moved away from the fixed point
by the GAMP update,
the fixed point can be regarded as locally unstable.

We consider the difference in $\widehat{\theta}_\mu^{\setminus\mu}$
induced by the small deviance in $\widehat{\bm{x}}$ as 
\begin{eqnarray}
\nonumber
\delta\widehat{\theta}_\mu^{\setminus\mu}\equiv\widehat{\theta}_\mu^{\setminus\mu}(\{\widehat{x}_i+\delta\widehat{x}_i\})
-\widehat{\theta}_\mu^{\setminus\mu}(\{\widehat{x}_i\})&=
\sum_{i=1}^N\frac{\partial\widehat{\theta}_\mu^{\setminus\mu}}{\partial\widehat{x}_i}\delta\widehat{x}_i\\
&=\frac{1}{\sqrt{N}}\sum_{i=1}^NF_{\mu i}\delta\widehat{x}_i,
\label{eq:theta_diff_AT}
\end{eqnarray}
where $\widehat{\theta}_\mu^{\setminus\mu}(\widehat{\bm{x}})$ denotes 
the value of $\widehat{\theta}_\mu^{\setminus\mu}$
under $\widehat{\bm{x}}$.
We restrict our discussion to the predictors whose
entries are i.i.d. with mean 0 and variance 1.
For such predictors,
the summation of L.H.S. in (\ref{eq:theta_diff_AT})
is reduced to zero at sufficiently large $N$, hence
we consider the squared sum as
\begin{eqnarray}
{\delta\widehat{\theta}_\mu^{\setminus\mu}}^2\simeq \frac{1}{N}\sum_{i=1}^N\delta\widehat{x}_i^2.
\label{eq:theta_diff_AT_sq}
\end{eqnarray}
Meanwhile, we consider the deviation of $\widehat{\bm{x}}$ attributed to the 
slight deviance of $\theta_\mu^{\setminus\mu}$ defined by
\begin{eqnarray}
\delta\widehat{x}_i\equiv\widehat{x}_{i}(\{\widehat{\theta}_\mu^{\setminus\mu}+\delta\widehat{\theta}_\mu^{\setminus\mu}\})-\widehat{x}_{i}(\{\widehat{\theta}_\mu^{\setminus\mu}\})=\sum_{\mu=1}^M\frac{\partial\widehat{x}_i}{\partial\widehat{\theta}_\mu^{\setminus\mu}}\delta\widehat{\theta}_\mu^{\setminus\mu},
\end{eqnarray}
where $\widehat{x}_{i}(\{\widehat{\theta}_\mu^{\setminus\mu}\})$
denotes the estimates under $\{\widehat{\theta}_\mu^{\setminus\mu}\}$.
The deviation of $\widehat{\theta}_\mu^{\setminus\mu}$ propagates to $\widehat{x}_i$
through $\mathrm{m}_i$, and hence, we obtain
\begin{eqnarray}
\nonumber
\frac{\partial\widehat{x}_i}{\partial\widehat{\theta}_\mu^{\setminus\mu}}&=\frac{\partial \mathrm{m}_i}{\partial \widehat{\theta}_\mu^{\setminus\mu}}
\frac{\partial\widehat{x}_i}{\partial \mathrm{m}_i}\\
&=\Sigma_i\left\{\frac{1}{\sqrt{N}}F_{\mu i}(\check{\partial}_{\hat{\theta}}{g}_{\mathrm{out}})_\mu+O(N^{-1})\right\}s_i\Sigma_i^{-1},
\label{eq:AT_bare}
\end{eqnarray}
where we use the relationship (\ref{eq:VtoM}).
For the random preditors,
the summation of (\ref{eq:AT_bare}) with 
respect to $\mu$ is expected to converge to zero for a sufficiently large $M$, and
hence, we consider the squared summation under the assumptions as
\begin{eqnarray}
\delta\widehat{x}_i^2\simeq\frac{s_i^2}{N}\sum_{\mu=1}^MF_{\mu i}^2(\check{\partial}_{\hat{\theta}}{g}_{\mathrm{out}})_\mu^2
{\delta\widehat{\theta}_\mu^{\setminus\mu}}^2\simeq\frac{\alpha s_i^2}{M}\sum_{\nu=1}^M(\check{\partial}_{\hat{\theta}}{g}_{\mathrm{out}})_\nu^2
{\delta\widehat{\theta}_\mu^{\setminus\mu}}^2.
\end{eqnarray}
Therefore, setting $\overline{\delta\widehat{x}^2}\equiv\frac{1}{N}\sum_{i=1}^N\delta\widehat{x}_i^2$,
we obtain the recursive relationship of $\overline{\delta\widehat{x}^2}$ as
\begin{eqnarray}
\overline{\delta\widehat{x}^2}=\alpha\overline{s^2}
\overline{(\check{\partial}{g}_{\mathrm{out}})^2}
\overline{\delta\widehat{x}^2},
\end{eqnarray}
where
we set $\overline{s^2}\equiv\frac{1}{N}\sum_{i=1}^Ns_i^2$
and $\overline{(\check{\partial}{g}_{\mathrm{out}})^2}\equiv\frac{1}{M}\sum_{\nu=1}^M(\check{\partial}_{\hat{\theta}}{g}_{\mathrm{out}})_\nu^2$.
Thus, the solution $\overline{\delta\widehat{x}^2}=0$ is unstable when
\begin{eqnarray}
\alpha\overline{s^2}
\overline{(\check{\partial}{g}_{\mathrm{out}})^2}>1,
\label{eq:AT_GAMP_final}
\end{eqnarray}
which indicates the local instability of the fixed point.
This condition corresponds the instability of the replica symmetric ansatz in the replica analysis
shown in sec.\ref{sec:Replica_AT}.

\subsection{Expression of GDF by GAMP}
\label{sec:GDF_fluctuation}

We start with the 
form (\ref{eq:theta_hat}) for the estimated parameter $\widehat{\theta}_\mu$.
By definition, 
$\widehat{\theta}_\mu^{\setminus\mu}$ and $s_{\theta_\mu}^{\setminus\mu}$
are defined in the $\mu$-cavity system,
and hence, their values are independent from $y_\mu$
when the graphical model does not have any loops including $y_\mu$.
In fact, when the graph contains loops, 
$\widehat{\theta}_\mu^{\setminus\mu}$ and $s_{\theta_\mu}^{\setminus\mu}$ are dependent on $y_\mu$, 
but we assume that their dependency on $y_\mu$ is negligible.
Under this consideration,
the dependency on $y_\mu$ only appears in the Onsager reaction term,
i.e. the term $g_{\mathrm{out}}$, and hence,
\begin{eqnarray}
\frac{\partial\widehat{\theta}_\mu}{\partial y_\mu}={s_{\theta}}_\mu
\frac{\partial(\check{g}_{\mathrm{out}})_{\mu}}{\partial y_\mu}.
\end{eqnarray}
From (\ref{eq:g_out_def}), in the case of GLM,
\begin{eqnarray}
\nonumber
\frac{\partial(\check{g}_{\mathrm{out}})_\mu}{\partial y_\mu}&=\frac{\beta}{{s_{\theta}}_\mu}\left\{\left\langle(\theta_\mu-\widehat{\theta}_\mu^{\setminus\mu})\!\left(\theta_\mu+b^\prime(y_\mu)\right)\right\rangle_{\theta_\mu}\!\!\!\!-\!
\left\langle\theta_\mu-\widehat{\theta}_\mu^{\setminus\mu}\right\rangle_{\theta_\mu}
\!\!\!\!\left\langle\theta_\mu+b^\prime(y_\mu)\right\rangle_{\theta_\mu}\!\right\}\\
&=\frac{{\chi_{\theta}}_\mu}{{s_{\theta}}_\mu},
\end{eqnarray}
where we define the variance of the model parameter as
\begin{eqnarray}
{\chi_{\theta}}_\mu=\beta\left(\langle\theta_\mu^2\rangle_{\theta_\mu}-\langle\theta_\mu\rangle_{\theta_\mu}^2\right).
\label{eq:chi_theta_def}
\end{eqnarray}
Hence, we obtain 
\begin{eqnarray}
\frac{\partial\widehat{\theta}_\mu}{\partial y_\mu}={\chi_{\theta}}_\mu,
\label{eq:diff_fin}
\end{eqnarray}
which holds for the general value of $\beta$, and implies that 
the GAMP evaluates GDF and the generalization gap for the in-sample error
as \cite{Sakata2016}
\begin{eqnarray}
\mathrm{gdf}^{(\mathrm{AMP})}({\cal D})=\frac{1}{M}\sum_{\mu=1}^M{\chi_{\theta}}_\mu
\label{eq:GDF_GAMP}
\end{eqnarray}
Eq. (\ref{eq:GDF_GAMP}) indicates the proportionality between the
response of the learning error 
and the variance of the estimated parameter at equilibrium,
which can be regarded as a fluctuation-response relationship described by GAMP 
\cite{Watanabe_FDT}.

The form of ${\chi_\theta}_\mu$ can be derived at the $\beta\to\infty$ limit,
when the following second order approximation is valid.
Expanding the function
$f_\xi(\theta;y_\mu,\widehat{\theta}_\mu^{\setminus\mu},{s_{\theta}}_\mu)$
around ${\theta}_\mu^*$ up to the second order,
we obtain
\begin{eqnarray}
\nonumber
f_\xi(\theta_\mu;y_\mu,\widehat{\theta}_\mu^{\setminus\mu},{s_{\theta}}_\mu)
&=f_\xi(\theta^*_\mu;y_\mu,\widehat{\theta}_\mu^{\setminus\mu},{s_{\theta}}_\mu)\\
&+\frac{(\theta_\mu-\theta_\mu^*)^2}{2}\frac{\partial^2}{\partial\theta^2}f_\xi(\theta;y_\mu,\widehat{\theta}_\mu^{\setminus\mu},{s_{\theta}}_\mu)\Big|_{\theta=\theta_\mu^*}.
\end{eqnarray}
Hence, in case of GLM, we obtain \cite{Sakata2018}
\begin{eqnarray}
{\chi_{\theta}}_\mu=\frac{{s_\theta}_\mu}{1+{s_\theta}_\mu a^{\prime\prime}(\theta_\mu^*)}.
\label{eq:GLM_theta_variance}
\end{eqnarray}
For the Gaussian likelihood, where $a^{\prime\prime}(\theta_\mu^*)=1$,
GDF is given by 
\begin{eqnarray}
\mathrm{gdf}^{(\mathrm{AMP})}({\cal D})=\frac{1}{M}\sum_{\mu=1}^M
\frac{{s_\theta}_\mu}{1+{s_\theta}_\mu}.
\end{eqnarray}
Considering the correspondence between the linear estimation rule and GAMP 
(\ref{eq:hat_vs_GAMP}),
the expression of GDF by GAMP can be regarded as DF
using the extended hat matrix (\ref{eq:hat_vs_GAMP})
for the Gaussian likelihood.

\subsection{Expression of the functional variance by GAMP}

From the relationship (\ref{eq:FV_derivative}),
FV is expressed using the effective distribution by GAMP (\ref{eq:eff_dist}) as
\begin{eqnarray}
\nonumber
\mathrm{FV}^{(\mathrm{AMP})}({\cal D})&=\frac{\beta}{M}\sum_{\mu=1}^M\left\{\left\langle\left(\ln f(y_\mu|\theta_\mu)\right)^2\right\rangle_{\theta_\mu}-\left\langle\ln f(y_\mu|\theta_\mu)\right\rangle_{\theta_\mu}^2\right\}\\
&=\frac{1}{\beta M}\sum_{\mu=1}^M\frac{\partial^2}{\partial\gamma^2}\ln \int d\theta_\mu f^{\gamma\beta}(y_\mu|\theta_\mu)\psi_c^{\beta}(\theta_\mu|{s_\theta}_\mu,\widehat{\theta}_\mu^{\setminus\mu})\Big|_{\gamma=1}
\end{eqnarray}
where $\gamma$ is an auxiliary variable.
At $\beta\to\infty$, FV is given by
\begin{eqnarray}
\mathrm{FV}^{(\mathrm{AMP})}({\cal D})=\frac{1}{M}\sum_{\mu=1}^M\frac{\partial^2}{\partial\gamma^2}{\cal T}_\gamma(\gamma)\Big|_{\gamma=1}
\label{eq:FV_GAMP_beta0}
\end{eqnarray}
where 
\begin{eqnarray}
{\cal T}_\gamma(\gamma)&=\gamma\ln f(y_\mu|\theta_\mu^*(\gamma))-\frac{(\theta_\mu^*(\gamma)-\widehat{\theta}_{\mu})^2}{2{{s_\theta}_\mu}}\\
\theta_\mu^*(\gamma)&=\mathop{\mathrm{argmax}}_{\theta}\left\{\gamma\ln f(y_\mu|\theta)-\frac{(\theta-\widehat{\theta}_\mu^{\setminus\mu})^2}{2{s_\theta}_\mu}\right\}.
\end{eqnarray}
In the case of GLM, the following relationship holds for any $\gamma$ and $\mu\in\bm{M}$:
\begin{eqnarray}
\gamma(y_\mu-a^\prime(\theta_\mu^*(\gamma)))-\frac{\theta_\mu^*(\gamma)-\widehat{\theta}_\mu^{\setminus\mu}}{{s_\theta}_\mu}=0,
\label{eq:FV_GAMP_solution}
\end{eqnarray}
which is reduced to (\ref{eq:theta_mu}) at $\gamma=1$, and 
we denote $\theta_\mu^*(\gamma=1)=\theta_\mu^*$.
Further, by differentiating both sides of (\ref{eq:FV_GAMP_solution})
with respect to $\gamma$,
we obtain
\begin{eqnarray}
y_\mu-a^{\prime}(\theta_\mu^*(\gamma))-\gamma a^{\prime\prime}(\theta_\mu^*(\gamma))\frac{\partial\theta_\mu^*(\gamma)}{\partial\gamma}=\frac{1}{{s_\theta}_\mu}\frac{\partial\theta_\mu^*(\gamma)}{\partial\gamma}.
\label{eq:theta_gamma_diff}
\end{eqnarray}
Hence, at $\gamma = 1$, 
\begin{eqnarray}
\frac{\partial\theta_\mu^*(\gamma)}{\partial\gamma}\Big|_{\gamma=1}=(\check{g}_{\mathrm{out}})_\mu
\frac{{s_\theta}_{\mu}}{1+{s_\theta}_\mu a^{\prime\prime}(\theta_\mu^*)},
\label{eq:diff_theta}
\end{eqnarray}
where we used the relationship (\ref{eq:_check_g_out}).

Utilizing (\ref{eq:FV_GAMP_solution}), 
(\ref{eq:FV_GAMP_beta0}) is transformed to 
\begin{eqnarray}
\nonumber
&\mathrm{FV}^{(\mathrm{AMP})}({\cal D})\\
&\hspace{1.0cm}=2(\check{g}_{\mathrm{out}})_\mu\frac{\partial\theta_\mu^*(\gamma)}{\partial\gamma}\Big|_{\gamma=1}
-\left(a^{\prime\prime}(\theta_\mu^*)+\frac{1}{{s_\theta}_\mu}\right)\left(\frac{\partial\theta_\mu^*(\gamma)}{\partial\gamma}\Big|_{\gamma=1}\right)^2,
\label{eq:FV_GAMP_last}
\end{eqnarray}
and substituting (\ref{eq:diff_theta}) to
(\ref{eq:FV_GAMP_last}),
we get 
\begin{eqnarray}
\mathrm{FV}^{(\mathrm{AMP})}({\cal D})=\frac{1}{M}\sum_{\mu=1}^M(\check{g}_{\mathrm{out}})_\mu^2
\frac{{s_\theta}_\mu}{1+{s_\theta}_\mu a^{\prime\prime}(\theta_\mu^*)}.
\end{eqnarray}
Comparing FV with $\sigma^2\times \mathrm{gdf}$, where
both are estimators for the generalization gap with respect to the in-sample prediction error,
the strength of the noise $\sigma^2$, which is assumed to be known in $C_p$ criterion,
is replaced with 
$(\check{g}_{\mathrm{out}})_\mu^2$ component-wise, in case of FV.
In the actual usage of the $C_p$ criterion, 
an estimator is used instead of the variance,
when its value is not known in advance \cite{Efron2004}.
FV can be regarded as the $C_p$ criterion defined with 
the estimator of the variance.
When the mean of the model $a^\prime(\theta_\mu^*)$
is exactly equal to the data $y_\mu$, FV is reduced to zero
as $(\check{g}_{\mathrm{out}})_\mu^2$ is the squared error between
them given in (\ref{eq:check_g}).

\subsubsection{Convergence of the series expansion}

We consider the convergence condition of the 
series expansion of (\ref{eq:T_expansion})
for the derivation of WAIC using the expression obtained using GAMP.
\begin{thm}
For the Gaussian likelihood,
the radius of convergence for the 
the series expansion with respect to $\eta$ is given by
\begin{eqnarray}
\eta\lim_{n\to\infty}\left|\frac{\sum_\mu(\check{g}_{\mathrm{out}})^2_\mu\left(\displaystyle\frac{{s_\theta}_\mu}{1+{s_\theta}_\mu}\right)^{n}}{\sum_\mu(\check{g}_{\mathrm{out}})^2_\mu\left(\displaystyle\frac{{s_\theta}_\mu}{1+{s_\theta}_\mu}\right)^{n-1}}\right|<1.
\label{eq:Gauss_conv}
\end{eqnarray}
When the components of the predictors are i.i.d. random variables
of mean zero and variance 1,
the radius of the convergence for the series expansion is given by
\begin{eqnarray}
\eta\left|\frac{{s_\theta}}{1+{s_\theta}}\right|<1.
\label{eq:Gauss_conv_rand}
\end{eqnarray}
\end{thm}
\begin{proof}
We prove the theorem in an inductive manner.
The first-order term of (\ref{eq:T_expansion})
is given by
\begin{eqnarray}
\frac{\partial{\cal T}_\gamma(\gamma)}{\partial\gamma}=\frac{1}{M}\sum_{\mu=1}^M\ln f(y_\mu|\theta_\mu^*(\gamma)).
\end{eqnarray}
We start our proof from the second-order term obtained as
\begin{eqnarray}
\frac{\partial^2{\cal T}_\gamma(\gamma)}{\partial\gamma^2}=\frac{1}{M}\sum_{\mu=1}^M(\check{g}_{\mathrm{out}}(\gamma))^2_\mu\frac{{s_\theta}_\mu}{1+\gamma{s_\theta}_\mu},
\end{eqnarray}
which corresponds to $\mathrm{FV}^{(\mathrm{AMP})}$ at $\gamma=1$, and 
we define
\begin{eqnarray}
(\check{g}_{\mathrm{out}}(\gamma))_\mu=y_\mu-\theta_\mu^*(\gamma)=\frac{\theta_\mu^*(\gamma)-\widehat{\theta}_\mu^{\setminus\mu}}{\gamma{s_\theta}_\mu},
\end{eqnarray}
where $(\check{g}_{\mathrm{out}}(\gamma=1))_\mu=(\check{g}_{\mathrm{out}})_\mu$
For $k = 3$, we obtain
\begin{eqnarray}
\frac{\partial^3{\cal T}_\gamma(\gamma)}{\partial\gamma^3}=-\frac{3}{M}\sum_{\mu=1}^M(\check{g}_{\mathrm{out}}(\gamma))^2_\mu\left(\frac{{s_\theta}_\mu}{1+\gamma{s_\theta}_\mu}\right)^2.
\end{eqnarray}
For general $k$, we can show that 
if the following relationship holds
\begin{eqnarray}
\frac{\partial^k{\cal T}(\gamma)}{\partial\gamma^k}=\frac{k!}{2M}(-1)^k
\sum_{\mu=1}^M(\check{g}_{\mathrm{out}}(\gamma))^2_\mu\left(\frac{{s_\theta}_\mu}{1+\gamma{s_\theta}_\mu}\right)^{k-1},
\label{eq:diff_general}
\end{eqnarray}
then 
\begin{eqnarray}
\frac{\partial^{k+1}{\cal T}(\gamma)}{\partial\gamma^{k+1}}=\frac{(k+1)!}{2M}(-1)^{k+1}\sum_{\mu=1}^M(\check{g}_{\mathrm{out}}(\gamma))^2_\mu\left(\frac{{s_\theta}_\mu}{1+\gamma{s_\theta}_\mu}\right)^{k}.
\end{eqnarray}
Therefore, (\ref{eq:diff_general}) holds for any $k\geq 2$, and we obtain
\begin{eqnarray}
{\cal T}(\eta)=\frac{\eta}{M}\sum_{\mu=1}^M\ln f(y_\mu|\theta_\mu^*)+\frac{1}{2M}\sum_{k=2}^\infty(-\eta)^k\sum_{\mu=1}^M
(\check{g}_{\mathrm{out}})^2_\mu
\left(\frac{{s_\theta}_\mu}{1+{s_\theta}_\mu}\right)^{k-1}.
\end{eqnarray}
This expression indicates that the radius of the convergence is given by (\ref{eq:Gauss_conv}).
For random i.i.d. predictors, ${s_\theta}_\mu$ is reduced ${s_\theta}$
for any $\mu$, and the radius of the convergence is given by (\ref{eq:Gauss_conv_rand}).
\end{proof}
For the Gaussian likelihood with random i.i.d. predictor, the 
radius of the convergence of the series expansion is 
given by GDF, as indicated in (\ref{eq:Gauss_conv_rand}).

\section{Approaching LOOCV error from FV at finite $\alpha$}
\label{sec:FV_correct}

\subsection{Discrepancy between the FV and LOOCV errors in the finite $\alpha$ region}

We show the prediction errors and their estimators for the random predictor $\bm{F}$, whose components are
independently and identically generated by Gaussian distribution
with mean 0 and variance 1.
In the numerical simulation, 
we set the true regression coefficient
$\bm{x}^{(0)}\in\mathbb{R}^N$
and the model parameter as $\bm{\theta}^{(0)}=\frac{1}{\sqrt{N}}\bm{F}\bm{x}^{(0)}$.


\subsubsection{Ridge regression under Gaussian noise}

\begin{figure}
\begin{minipage}{0.495\hsize}
\centering
\includegraphics[width=2.5in]{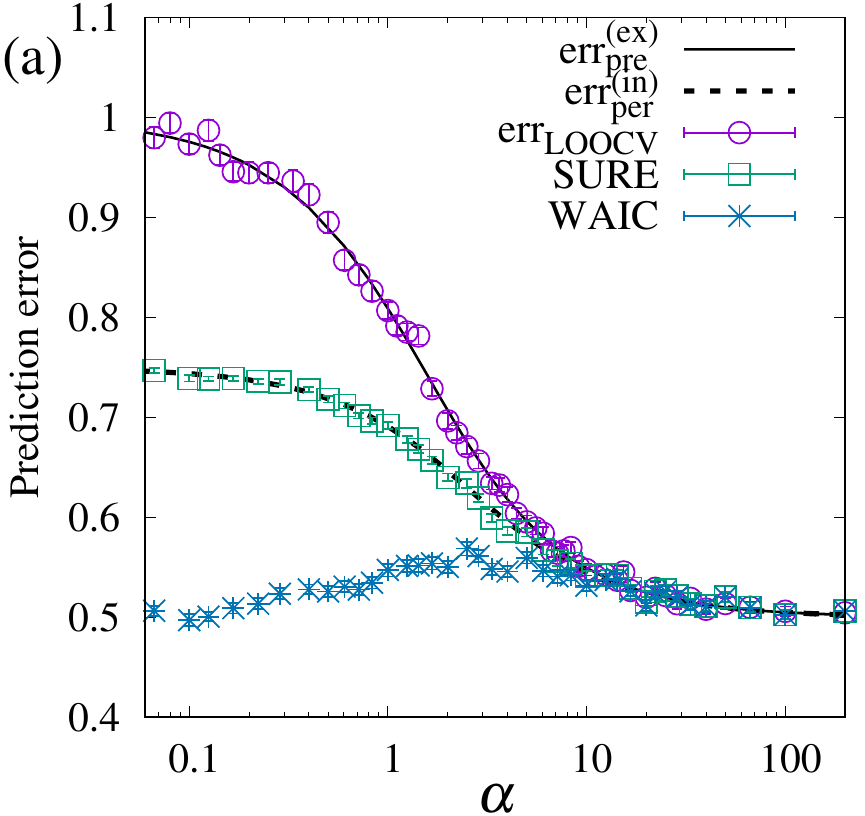}
\end{minipage}
\begin{minipage}{0.495\hsize}
\includegraphics[width=2.5in]{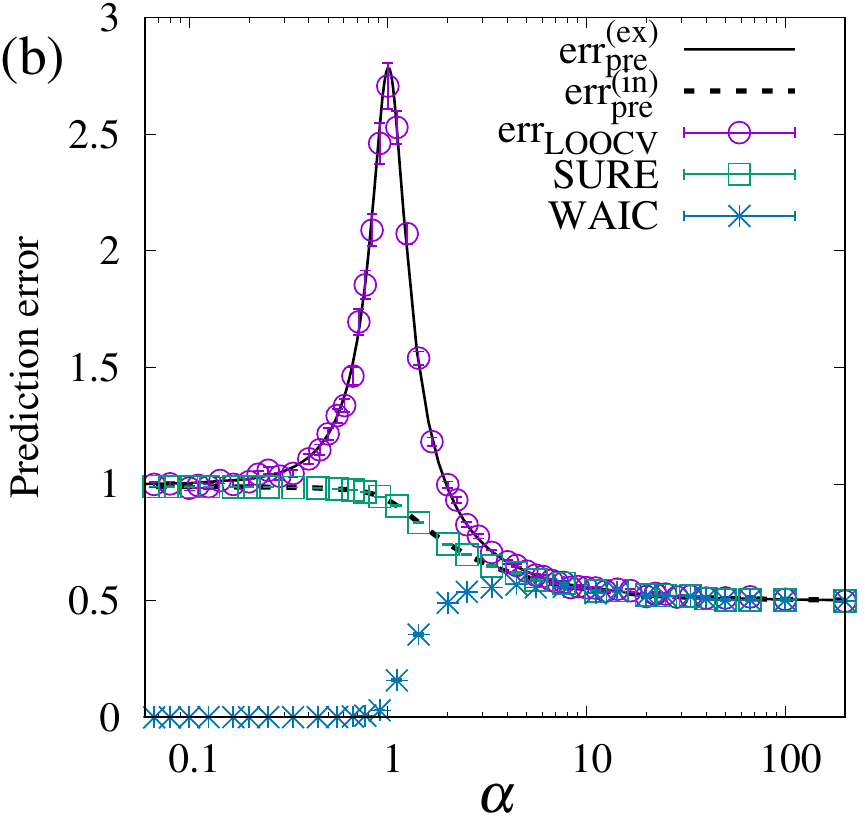}
\end{minipage}
\caption{Prediction errors for the ridge regression at $M = 200$ 
for (a) $\lambda = 1,\sigma = 1$
and (b) $\lambda = 0.01,\sigma = 1$.
$\mathrm{err}_{\mathrm{LOOCV}}$, SURE, and WAIC calculated by GAMP
are denoted by
$\circ$, $\square$, and $*$, respectively.
The value of each point is averaged over 100-times numerical simulations, and the bar on the point indicates standard error.
The solid and dashed lines represent the analytical results obtained using the replica method
for the extra-sample prediction error $\mathrm{err}_{\mathrm{pre}}^{(\mathrm{ex})}$
and in-sample prediction error $\mathrm{err}_{\mathrm{pre}}^{(\mathrm{in})}$, respectively.}
\label{fig:Ridge_pre}
\begin{minipage}{0.495\hsize}
\centering
\includegraphics[width=2.5in]{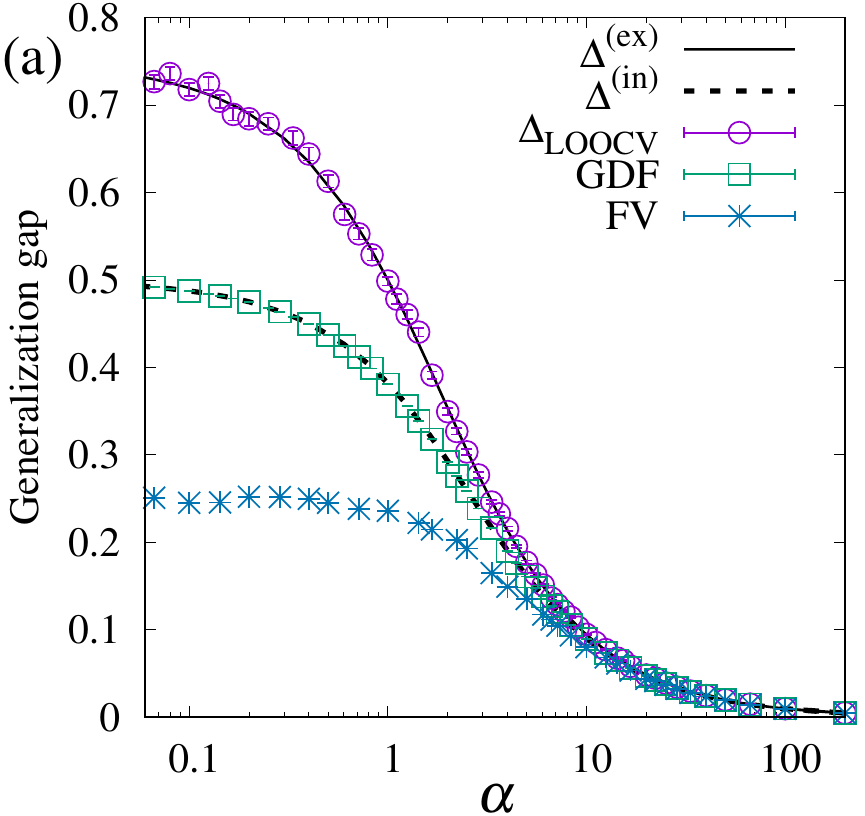}
\end{minipage}
\begin{minipage}{0.495\hsize}
\includegraphics[width=2.5in]{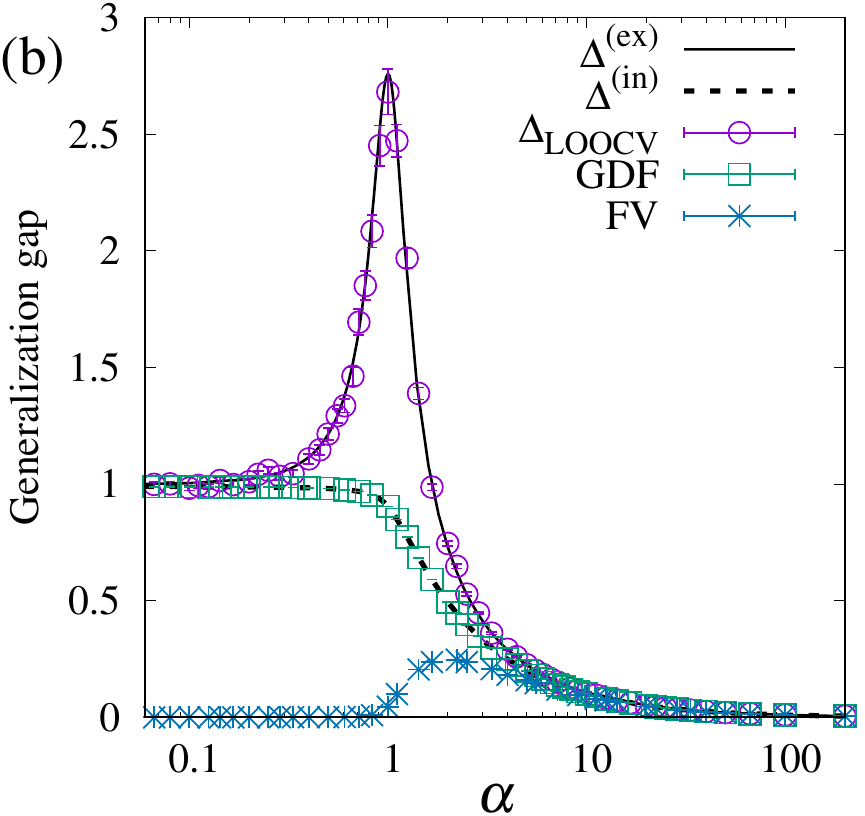}
\end{minipage}
\caption{Generalization gaps for ridge regression at $M = 200$
for (a) $\lambda = 1,\sigma = 1$
and (b) $\lambda = 0.01,\sigma = 1$.
$\Delta_{\mathrm{LOOCV}}$, GDF, and FV calculated by GAMP are denoted by
$\circ$, $\square$, and $*$, respectively.
The value of each point is averaged over 100-times numerical simulations, and the bar on the point indicates standard error.
The solid and dashed lines denote the analytical results obtained using the replica method
for the generalization gap of the extra-sample prediction error $\Delta^{(\mathrm{ex})}$ and the
in-sample prediction error $\Delta^{(\mathrm{in})}$, respectively.}
\label{fig:Ridge_gap}
\end{figure}

There is no need to introduce message passing to obtain the estimate
$\widehat{\bm{x}}({\cal D})$ for the ridge regression (\ref{eq:ridge_def})
because the ridge regression can be solved analytically.
We apply our method for evaluating the FV and LOOCV error 
to ridge regression to clarify its adequacy.
In the numerical simulation for the ridge regression,
we independently generate the components of $\bm{x}^{(0)}$ according 
to the Gaussian distribution with mean 0 and variance 1,
and we generate the data as $y_\mu\sim{\cal N}(\theta_\mu^{(0)},\sigma^2)$.
Figs.\ref{fig:Ridge_pre} and \ref{fig:Ridge_gap}
show prediction errors and generalization gaps 
for the ridge regression at $M = 200$,
where $N$ runs $[1,3000]$; therefore, the region $\alpha\in[0.667,200]$ is shown.
Parameters for Figs.\ref{fig:Ridge_pre} and \ref{fig:Ridge_gap}
are set as (a) $\lambda=1$ and $\sigma=1$,
and (b) $\lambda=0.01$ and $\sigma=1$, respectively.
The LOOCV error, $\mathrm{err}_{\mathrm{LOOCV}}$,
and corresponding generalization gap $\Delta_{\mathrm{LOOCV}}$ are
computed by (\ref{eq:LOOCV_linear}),
and SURE, WAIC, and their corresponding generalization gaps,
namely GDF and FV, are 
calculated by GAMP.
The value of each point is averaged over 100 realization of $\{\bm{y},\bm{F},\bm{x}^{(0)}\}$.
Theoretical results obtained using the replica method (see sec. \ref{sec:replica})
for the extra-sample and 
in-sample prediction errors are shown by solid and dashed lines, respectively.
As shown in the figure, the GDF by GAMP
works well for all region of $\alpha$ as an estimator of the in-sample prediction error.
When the value of $\alpha$ is sufficiently large, 
all prediction errors and all generalization gaps coincide with each other; this means all 
estimators can work as
appropriate estimates of the generalization gap of the extra-sample prediction error.
However, for small $\alpha$, the differences between the 
extra-sample prediction error
and their estimates increase, and the cusp around 
$\alpha=1$ for a sufficiently small $\lambda$,
which is known as the double descent phenomena \cite{surprise}, 
cannot be described by 
WAIC and SURE, as shown in Fig.\ref{fig:Ridge_pre} (b) and Fig.\ref{fig:Ridge_gap} (b).

\subsubsection{Linear regression with elastic net penalty 
under Gaussian noise}

\begin{figure}
\begin{minipage}{0.495\hsize}
\centering
\includegraphics[width=2.5in]{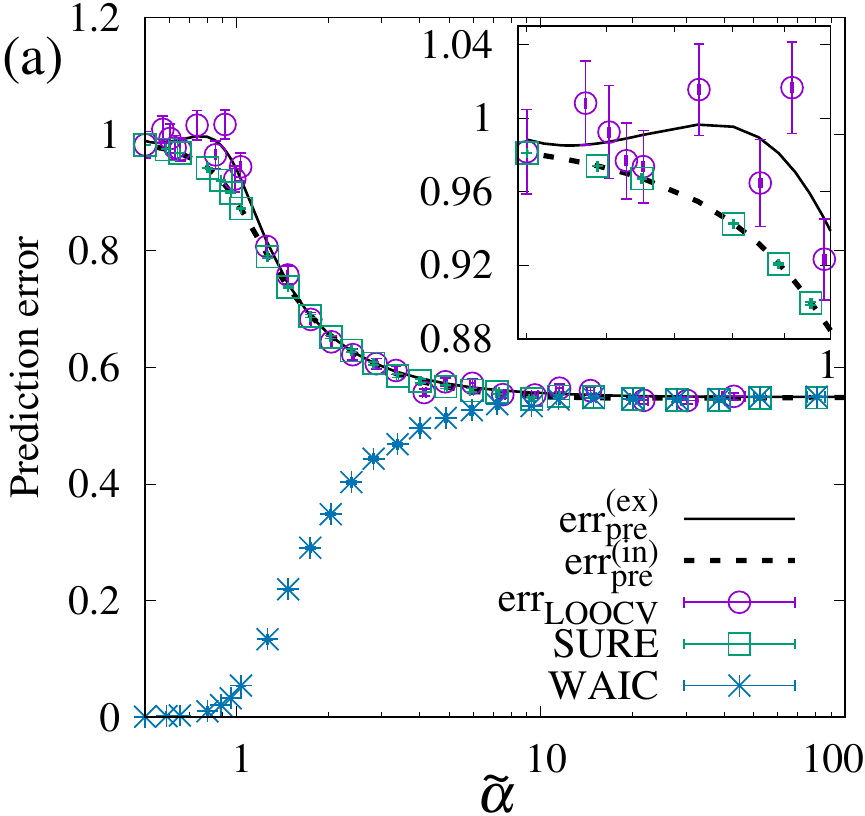}
\end{minipage}
\begin{minipage}{0.495\hsize}
\includegraphics[width=2.5in]{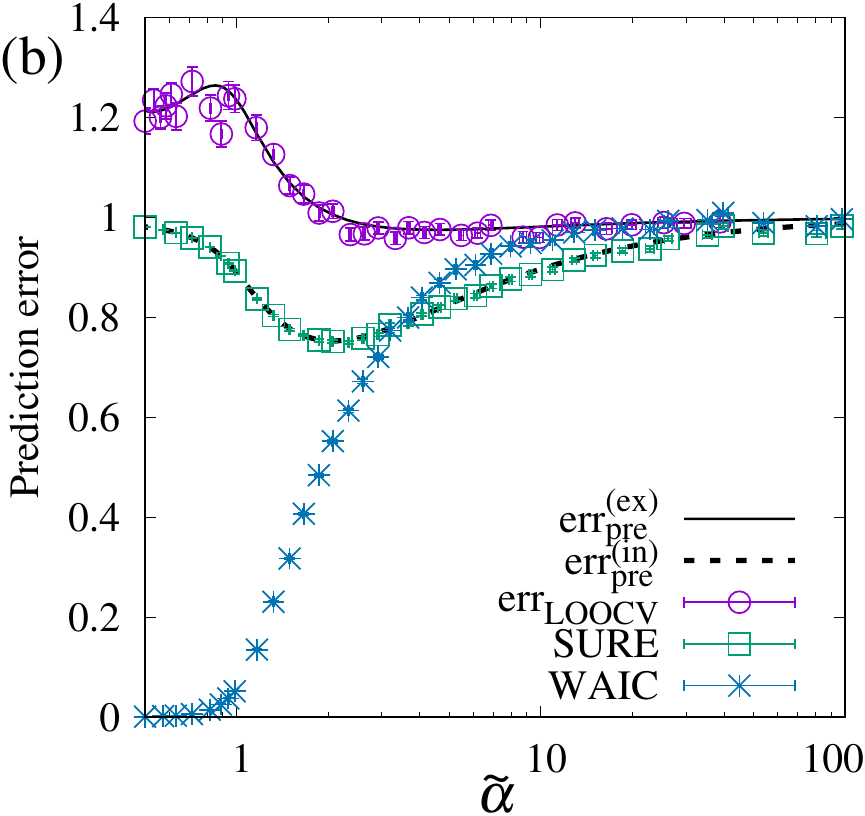}
\end{minipage}
\caption{Prediction errors for linear regression with elastic net penalty
at $M = 200$, $N = 400$, $\lambda_2 = 0.01$, and $\sigma = 1$
for (a) $\rho = 0.1$ and (b) $\rho = 1$;
$\lambda_1$ is controlled to set 
$\widetilde{\alpha}=M\slash ||\widehat{\bm{x}}||_0$ at each point.
The LOOCV error by the CD algorithm,
SURE, and WAIC calculated by GAMP 
are denoted by $\circ$, $\square$, and $*$, respectively.
The value of each point is averaged over 100-times numerical simulations, and 
in particular the bar on the point indicates the standard error.
The solid and dashed lines represent the analytical results using the replica method
for the extra-sample and in-sample prediction errors, respectively.}, 
\label{fig:EN_pre}
\begin{minipage}{0.495\hsize}
\centering
\includegraphics[width=2.5in]{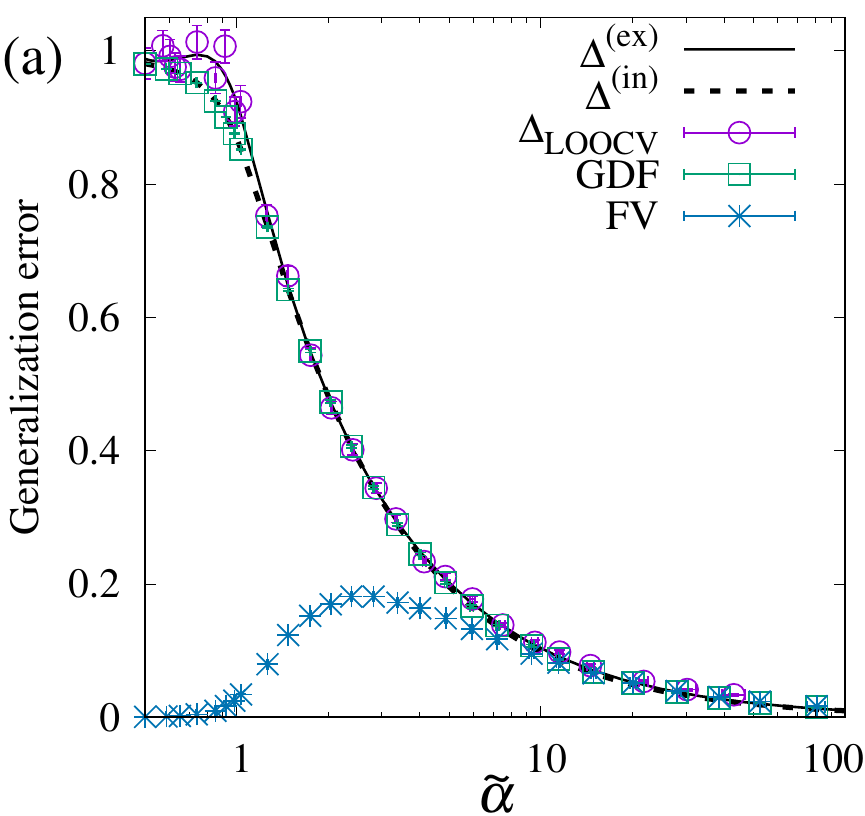}
\end{minipage}
\begin{minipage}{0.495\hsize}
\includegraphics[width=2.5in]{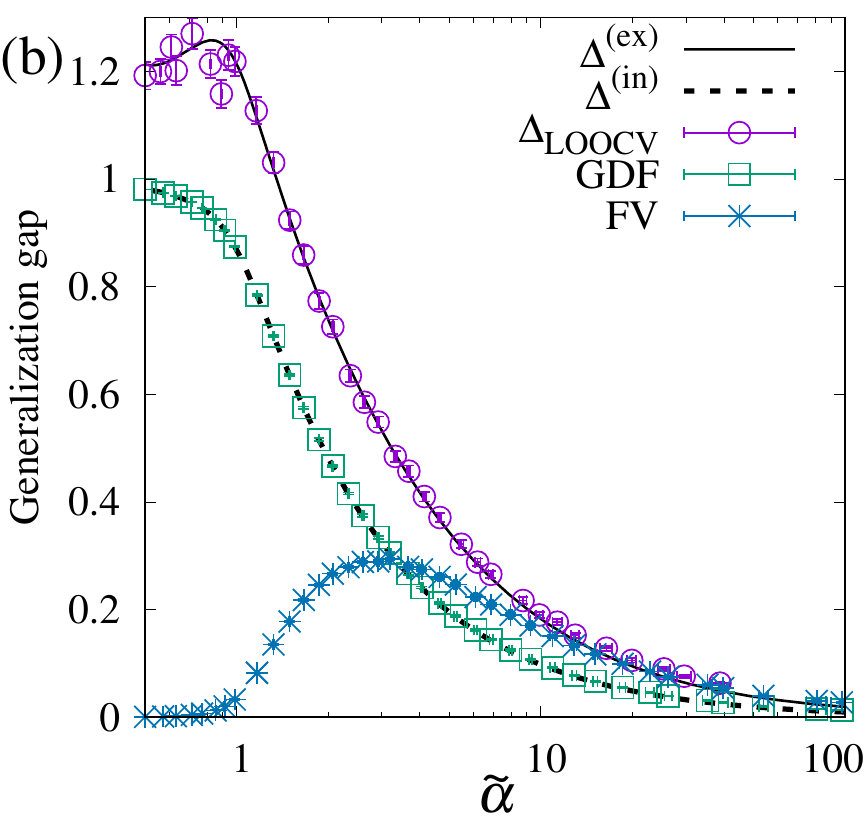}
\end{minipage}
\caption{Generalization gaps for the linear regression with elastic net penalty
at $M = 200$, $N = 400$, $\lambda_2 = 0.01$ and $\sigma = 1$
for (a) $\rho = 0.1$ and (b) $\rho = 1$,
where $\lambda_1$ is controlled to set 
$\widetilde{\alpha}=M\slash ||\widehat{\bm{x}}||_0$ at each point.
$\Delta_{\mathrm{LOOCV}}$ by CD algorithm,
GDF, and FV calculated by GAMP 
are denoted by $\circ$, $\square$, and $*$, respectively.
The value of each point is averaged over a 100-times numerical simulation, and the bar on the point indicates the standard error.
The solid and dashed lines denote the analytical results obtained using the replica method
for the generalization gap of 
the extra-sample and in-sample prediction errors, respectively.}
\label{fig:EN_gap}
\end{figure}

We show the results of GAMP for the linear regression with elastic net penalty
as an example where 
the estimate cannot be described by the hat matrix.
The estimates have zero components induced by the 
regularization parameter $\lambda_1$ of (\ref{eq:def_EN_penalty}).
Hence, 
we prepare sparse coefficient $\bm{x}^{(0)}$ in the numerical simulation.
We identically and independently generate the components in
$\bm{x}^{(0)}\in\mathbb{R}^N$ according to
the Gaussian distribution with mean zero and variance 1, and 
then, we choose $N(1-\rho)$-components and set them to zero, where
$\rho\in[0,1]$.
Using this sparse coefficient $\bm{x}^{(0)}$, 
we set $\bm{\theta}^{(0)}=\frac{1}{\sqrt{N}}\bm{F}\bm{x}^{(0)}$, and 
the components of data $\bm{y}$
are generated as $y_\mu\sim{\cal N}(\theta_\mu^{(0)},\sigma^2)$.
When $\lambda_1>0$, the estimate $\widehat{\bm{x}}({\cal D})$ has zero components,
hence the number of the parameter is given by $||\bm{x}||_0$,
where $||\cdot||_0$ is the $\ell_0$ norm.
Therefore, we set $\tilde{\alpha}=M\slash ||\bm{x}||_0$
distinguishing from $\alpha=M\slash N$, and the 
parameter regime where 
the number of the model parameters is larger than data dimension
corresponds to $\tilde{\alpha}>1$.

For the elastic net regularization, 
we cannot derive the exact form of the estimate
as with the ridge regression.
Therefore, we compare our result with the LOOCV error calculated 
by the coordinate descent (CD) algorithm \cite{sparse_book} considered to lead to reasonable solutions.
For the details of the CD algorithm, see \ref{sec:app_CD}.

Figs. \ref{fig:EN_pre} and \ref{fig:EN_gap} show 
the prediction errors and corresponding generalization gaps 
at $M = 200$, $N = 400$, $\sigma = 1$, and $\lambda_2 = 0.01$, respectively,
for linear regression with 
the elastic net regularization under Gaussian noise.
In the figs., the regularization parameter $\lambda_1$
runs from $[0,2]$ for (a) $\rho = 0.1$ and $[0,3]$ for (b) $\rho = 1$.
The solid and dashed lines represent the analytical result obtained using the replica method for the extra-sample and in-sample prediction errors, respectively.
As with the ridge regression, the FV tends to underestimate the prediction error and the
generalization gap for the small $\widetilde{\alpha}$ region.

\subsection{From FV to the LOOCV error}

We discuss the discrepancy between FV and $\Delta_{\mathrm{LOOCV}}$.
Here, we start with FV and not GDF
because 
FV corresponds to the first-order approximation
of the generalization gap of LOOCV error,
as WIC is equivalent to TIC.
We approach LOOCV error by removing assumptions imposed in the 
derivation of TIC, one by one.

\subsubsection{Differences in metrics}

In deriving TIC, we imposed assumption {\bf A3},
which corresponds to the ignorance of the difference between
$a^\prime(\bm{\theta}(\bm{F},\bm{x}({\cal D}_{\setminus\mu})))$ and 
$a^\prime(\bm{\theta}(\bm{F},\bm{x}({\cal D})))$
in the case of GLM.
This assumption is supposed to be inadequate in the region
where $N$ is sufficiently large because differences between the 
estimates under the full sample and that under the LOO sample are accumulated 
by $N$-components.
We consider the correction of the 
first-order approximation of the LOOCV error 
in the framework of GAMP by
removing assumption {\bf A3}.
We denote FV without assumption {\bf A3} as 
cFV(corrected FV), which is defined by 
\begin{eqnarray}
\mathrm{cFV}({\cal D})=\frac{1}{M}\mathrm{Tr}\left({\cal J}^{\varphi}(\widehat{\bm{x}}({\cal D}))^{-1}{\cal I}^{\setminus\cdot}
(\{\widehat{\bm{x}}({\cal D}_{\setminus\mu})\})\right).
\end{eqnarray}
Using the expression by GAMP, we obtain
\begin{eqnarray}
{\cal I}_{ij}^{\setminus\cdot}(\{\widehat{\bm{x}}({\cal D}_{\setminus\mu})\})&=\frac{1}{MN}\sum_{\mu=1}^MF_{\mu i}F_{\mu j}
(y_\mu-a^\prime(\widehat{\theta}_\mu^{\setminus\mu}))^2,
\end{eqnarray}
and hence $\mathrm{cFV}$ is given by
\begin{eqnarray}
\mathrm{cFV}^{(\mathrm{GAMP})}&=
&=\frac{1}{M}\sum_{\mu=1}^M(y_\mu-a^{\prime}(\widehat{\theta}_\mu^{\setminus\mu}))^2{\chi_\theta}_\mu.
\end{eqnarray}
From (\ref{eq:theta_mu}), we know the relationship between 
$\widehat{\theta}_\mu$ and $\widehat{\theta}^{\setminus\mu}_\mu$.

%

\begin{figure}
\begin{minipage}{0.495\hsize}
\centering
\includegraphics[width=2.5in]{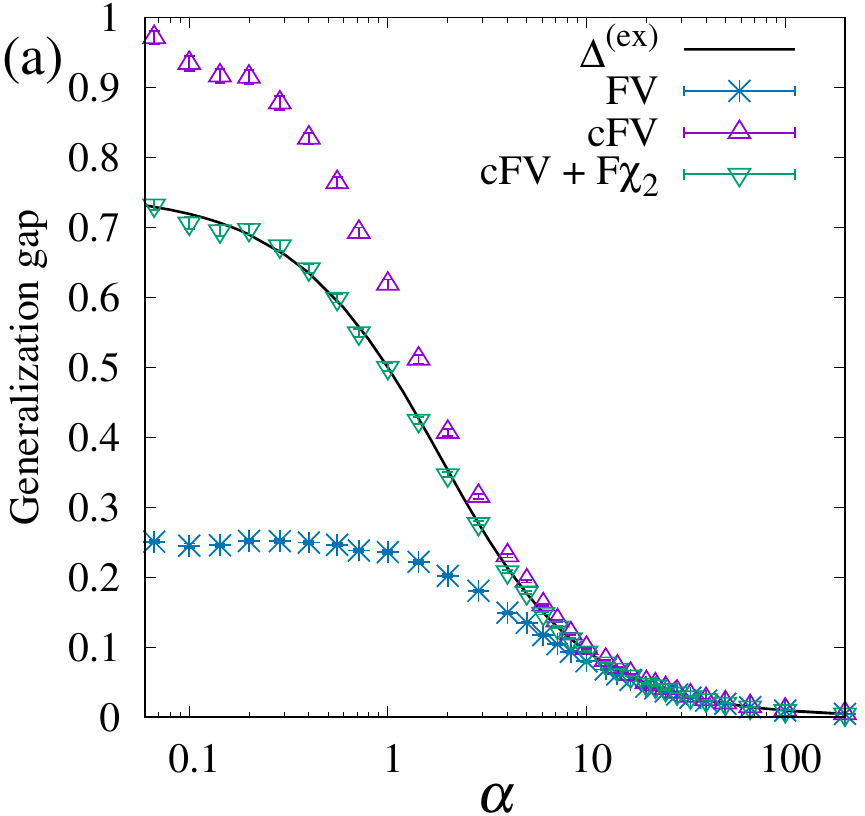}
\end{minipage}
\begin{minipage}{0.495\hsize}
\includegraphics[width=2.5in]{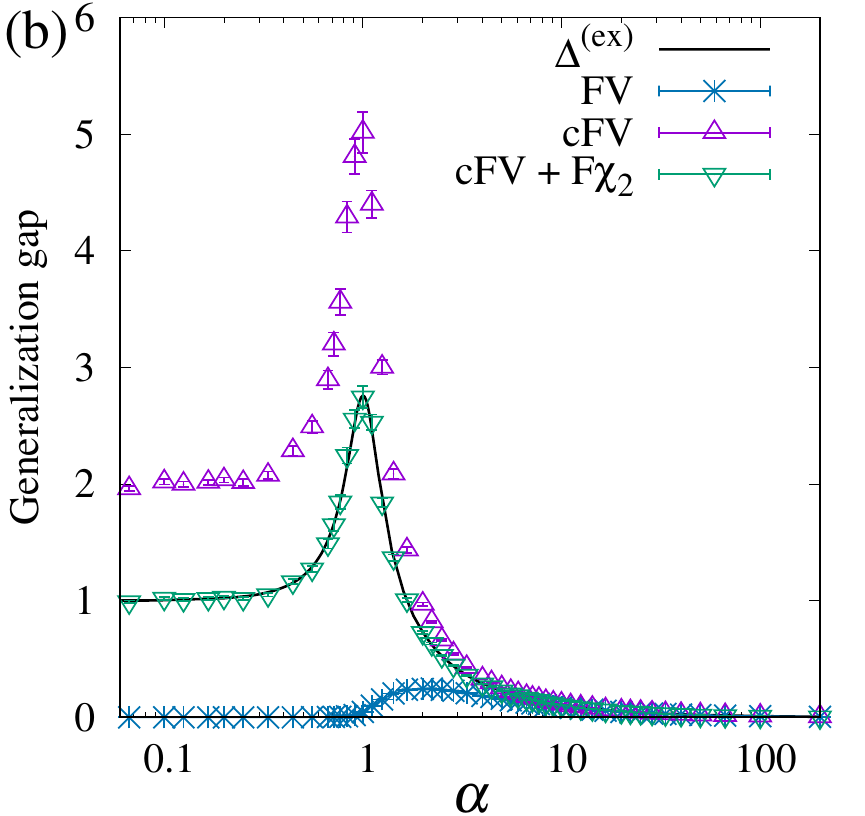}
\end{minipage}
\caption{Generalization gaps for ridge regression 
at $M = 200$ and $\sigma = 1$ for (a) $\lambda = 1$ and (b) $\sigma = 0.01$.
FV ($*$), cFV ($\triangle$), and cFV+F$\chi_2$ ($\triangledown$) by GAMP and $\Delta^{(\mathrm{ex})}$ obtained using the replica method.
The values of the each points are averaged over 100 data samples ${\cal D}$,
and the bars on the points represent standard error.
}
\label{fig:Ridge_correct}
\end{figure}
\begin{figure}
\begin{minipage}{0.495\hsize}
\centering
\includegraphics[width=2.5in]{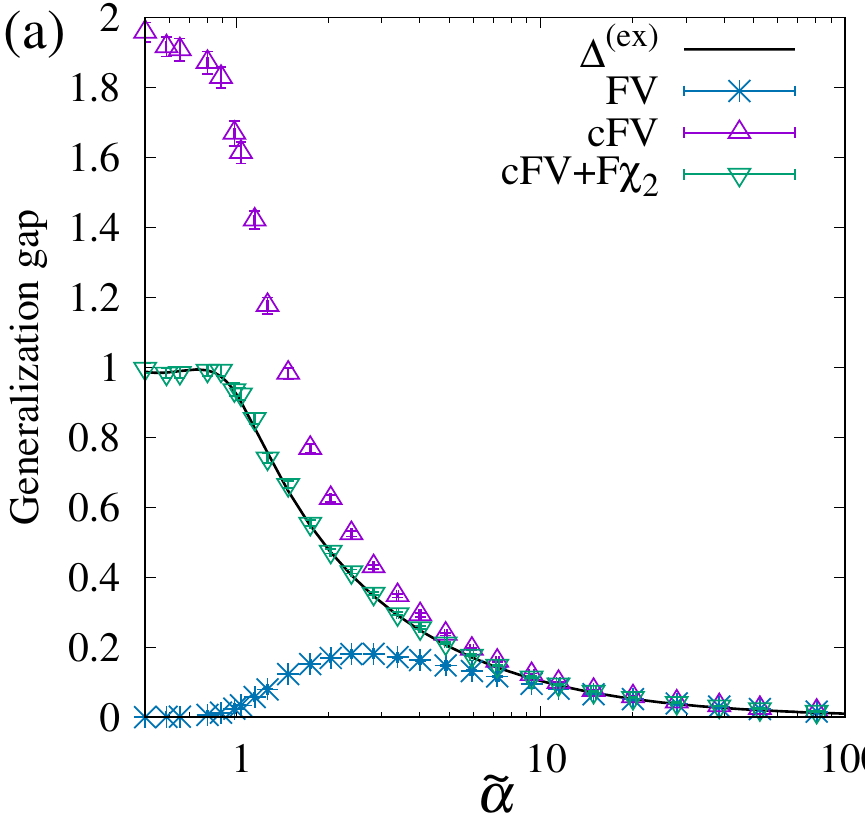}
\end{minipage}
\begin{minipage}{0.495\hsize}
\includegraphics[width=2.5in]{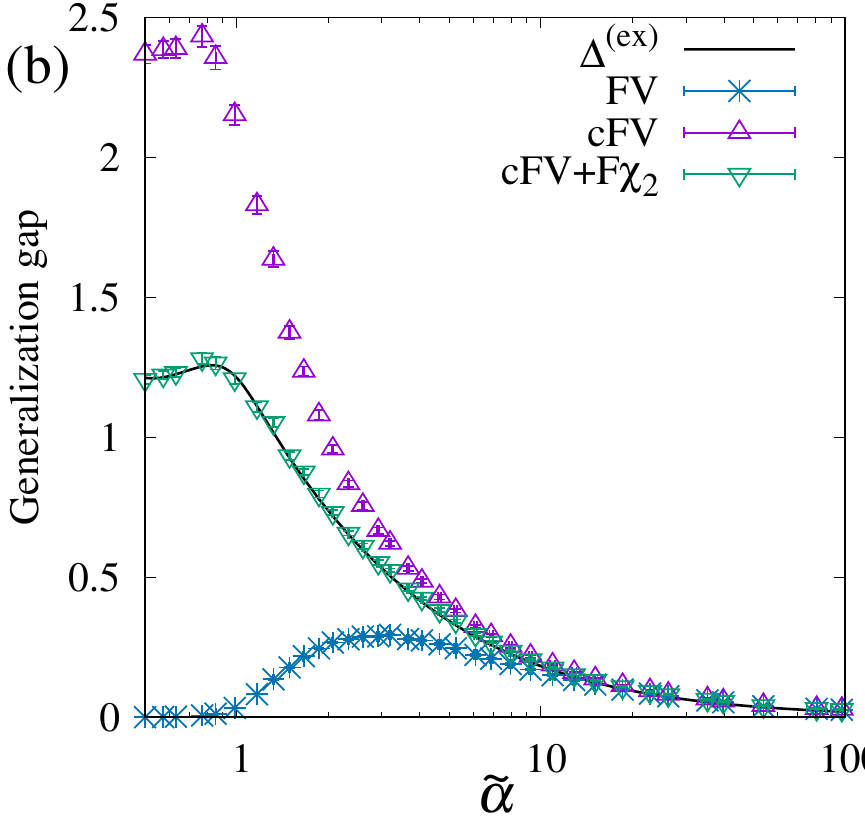}
\end{minipage}
\caption{Generalization gaps for the linear regression with elastic net penalty
at $M = 200$, $N = 400$, $\sigma = 1$, and $\lambda_2 = 0.01$ for (a) $\rho = 0.1$
and (b) $\rho = 1$.
The values of $\lambda_1$ are controlled to satisfy $\widetilde{\alpha}=M\slash 
||\widehat{\bm{x}}||_0$,
and FV ($*$), cFV ($\triangle$), and cFV+F$\chi_2$ ($\triangledown$) by GAMP and $\Delta^{(\mathrm{ex})}$ by replica method are also shown.
The values of each points are averaged over 100 data samples ${\cal D}$,
and the bars on the points denote the standard error.
}
\label{fig:EN_correct}
\end{figure}

In Figs. \ref{fig:Ridge_correct} and \ref{fig:EN_correct},
the behavior of the cFV is shown for ridge regression and linear regression with
the elastic net penalty, respectively.
The setting of the numerical simulation for Figs. \ref{fig:Ridge_correct} and 
\ref{fig:EN_correct} are the same as that explained in the previous subsection.
cFV overestimates $\Delta_{\mathrm{LOOCV}}$ for a sufficiently small $\alpha$; however, 
the cFV qualitatively approaches the behavior of $\Delta_{\mathrm{LOOCV}}$; i.e. 
cFV has a cusp around $\alpha=1$ in Fig.\ref{fig:Ridge_correct}(b),
and the minimum at $\alpha<1$ in \Fref{fig:EN_correct}(b), 
as with $\Delta^{(\mathrm{ex})}$.

\subsubsection{Higher order expansion of cross validation error}

We relax assumption {\bf A1} for further correction,
and consider higher-order terms 
with respect to $\widehat{x}_i({\cal D}_{\setminus\mu})-\widehat{x}_i({\cal D})$.
When we remove the assumption {\bf A1},
the expression (\ref{eq:diff_cavity_final}) 
and (\ref{eq:lnf_expand}) needs to be corrected.
For example, when we
consider the second order of $\widehat{x}_i({\cal D}_{\setminus\mu})-\widehat{x}_i({\cal D})$,
its value 
should be determined to satisfy 
\begin{eqnarray}
\nonumber
&\frac{1}{M}\sum_{k=1}^N\frac{\partial^2}{\partial x_k\partial x_j}\varphi({\cal D},\bm{x})\Big|_{\bm{x}=\widehat{\bm{x}}({\cal D}_{\setminus\nu})}
(\widehat{x}_k({\cal D}_{\setminus\nu})-\widehat{x}_k({\cal D}))\\
\nonumber
&+\frac{1}{2M}\sum_{k\ell}\frac{\partial^3}{\partial x_\ell\partial x_k\partial x_j}
\varphi({\cal D},\bm{x})\Big|_{\bm{x}=\widehat{\bm{x}}({\cal D}_{\setminus\nu})}
(\widehat{x}_k({\cal D}_{\setminus\nu})-\widehat{x}_k({\cal D}))
(\widehat{x}_\ell({\cal D}_{\setminus\nu})-\widehat{x}_\ell({\cal D}))\\
&=\frac{1}{M}\frac{\partial}{\partial x_j}\ln f(y_\nu|\theta_\nu(\mathrm{f}_\nu,\bm{x})),
\label{eq:cavity_expansion_2nd}
\end{eqnarray}
as shown in \ref{sec:app_expansion_JtoK},
and the value differs from that given by (\ref{eq:diff_cavity_final}) in general.

For simplicity,
we impose an assumption in the penalty term as
\begin{description}
\item[A4] Higher-order derivative of the penalty term 
$h(x)$ than the second order is zero for any $x$.
\end{description}
Penalties considered in this paper, i.e. elastic net penalty, 
satisfies the 
condition {\bf A4} for any regularization parameter $\lambda_1$ and $\lambda_2$.
Higher order derivatives of
$\varphi$ than the second order is ignorable when we focus on the Gaussian likelihood with penalty satisfying {\bf A4},
and hence, 
(\ref{eq:cavity_expansion_2nd}) is reduced to  (\ref{eq:diff_cavity_final}).
Therefore, when removing the assumption {\bf A1} under the assumpgion {\bf A4},
the correction of (\ref{eq:diff_cavity_final}) does not need to be considered, 
and only that of (\ref{eq:lnf_expand}) needs to be considered.
By continuing the expansion of (\ref{eq:lnf_expand}) up to the second order, and 
substituting the expression (\ref{eq:diff_cavity_final}),
we obtain the second order term of the generalization gap denoted by $\mathrm{F}\chi_2$ as
\begin{eqnarray}
\mathrm{F}\chi_2&
&=-\frac{1}{2M}\sum_{\mu=1}^Ma^{\prime\prime}(\widehat{\theta}_\mu^{\setminus\mu})\left(y_\mu-a^\prime(\widehat{\theta}_\mu^{\setminus\mu})\right)^2{\chi_\theta}_\mu^2,
\end{eqnarray}
where we used the relationship 
\begin{eqnarray}
\beta\left(\langle x_ix_j\rangle_{\beta}-\langle x_i\rangle_{\beta}\langle x_j\rangle_{\beta}\right)=\frac{1}{M}\left({{\cal J}^{\varphi}}^{-1}\right)_{ij},
\label{eq:J_vs_variance}
\end{eqnarray}
(\ref{eq:chi_theta_def}),
and (\ref{eq:GLM_theta_variance}).
In the case of Gaussian likelihood with penalties satisfying {\bf A4},
we obtain the GAMP-based estimator of the generalization gap for $\mathrm{err}_{\mathrm{LOOCV}}$ as
\begin{eqnarray}
\widehat{\Delta}_{\mathrm{LOOCV}}=\mathrm{cFV}+\mathrm{F}\chi_2.
\label{eq:gap_CV_Gauss}
\end{eqnarray}
Similarly, the $k$-th order term denoted by $\mathrm{F}\chi_k$ 
under the assumptions {\bf A2} and {\bf A4}
is given by
\begin{eqnarray}
\mathrm{F}\chi_k
&=-\frac{1}{k!M}\sum_{\mu=1}^Ma^{\prime(k)}(\widehat{\theta}_\mu^{\setminus\mu})\left(y_\mu-a^\prime(\widehat{\theta}_\mu^{\setminus\mu})\right)^k{\chi_\theta}_\mu^k,
\label{eq:gap_CV_Gauss_high}
\end{eqnarray}
where $a^{\prime(k)}(\widehat{\theta}_\mu^{\setminus\mu})$
denotes the $k$-th order derivative.
For the Gaussian likelihood, $\mathrm{F}\chi_k$ for $k\geq 3$ are zero
as $a^{\prime(k)}(\theta)$ for $k\geq 3$ are zero for any $\theta$.

The behaviour of $\widehat{\Delta}_{\mathrm{LOOCV}}$ (\ref{eq:gap_CV_Gauss})
by GAMP
is shown in Figs.\ref{fig:Ridge_correct} and \ref{fig:EN_correct}
for the ridge regression and linear regression with elastic net penalty,
respectively.
As the second-order term is negative, $\widehat{\Delta}_{\mathrm{LOOCV}}$
decreases from cFV and well matches with the generalization gap 
for the extra-sample prediction error.

In the discussion so far, we did not mention assumption {\bf A2}.
For Gaussian likelihood, ${\cal J}(\bm{x})$ do not depend on the 
argument, and therefore, we do not need to consider the adequacy of assumption {\bf A2}.

\subsubsection{Correction of the Hessian}

We consider the removal of assumption {\bf A1} and {\bf A2}
for general likelihoods of GLM.
For the case, correction of (\ref{eq:diff_cavity_final})
and (\ref{eq:gap_CV_Gauss_high}) are required.
As shown in \ref{sec:app_expansion_JtoK},
(\ref{eq:diff_cavity_final}) is corrected by removing assumptions {\bf A1} and  
{\bf A2} as follows:
\begin{eqnarray}
\bm{x}({\cal D}_{\setminus\nu})-\bm{x}({\cal D})
=({\cal K}^{\setminus\nu})^{-1}\frac{1}{M}\frac{\partial}{\partial\bm{x}}\ln f(y_\nu|\theta_\nu(\mathbf{f}_\nu,\bm{x}))\Big|_{\bm{x}=\widehat{\bm{x}}({\cal D}_{\setminus\nu})},
\end{eqnarray}
where ${\cal K}^{\setminus\mu}$ for $\mu\in\bm{M}$
is defined by
\begin{eqnarray}
{\cal K}^{\setminus\nu}=-{\cal J}+\frac{d_\nu\mathbf{f}_\nu^{\top}\mathbf{f}_\nu}{NM}
\end{eqnarray}
with
\begin{eqnarray}
d_\nu=\frac{a^{\prime}(\widehat{\theta}_\nu)-a^{\prime}(\widehat{\theta}_\nu^{\setminus\nu})}{\widehat{\theta}_{\nu}-\widehat{\theta}_{\nu}^{\setminus\nu}}
-a^{\prime\prime}(\widehat{\theta}_\nu).
\end{eqnarray}
When the difference between $\widehat{\theta}_{\nu}$
and $\widehat{\theta}_{\nu}^{\setminus\nu}$ is sufficiently small,
$d_\nu$ corresponds to $a^{\prime\prime}(\widehat{\theta}_\nu^{\setminus\nu})-a^{\prime\prime}(\widehat{\theta}_\nu)$.
Then, for Gaussian likelihood
where $a^{\prime\prime}(\theta)=1$ for any $\theta$,
${\cal K}^{\setminus\mu}={\cal J}$ holds for any 
$\mu\in\bm{M}$.
Using the matrix ${\cal K}^{\setminus\mu}$,
the $k$-th order term of the generalization gap is corrected as
\begin{eqnarray}
\mathrm{F}\chi_k=-\frac{1}{k!M}\sum_{\mu=1}^M
a^{\prime(k)}(\widehat{\theta}_\mu^{\setminus\mu})
\left(y_\mu-a^\prime(\widehat{\theta}_\mu^{\setminus\mu})\right)^k
\left({\chi_\theta}_\mu^{\setminus\mu}\right)^k,
\end{eqnarray}
where we define
\begin{eqnarray}
{\chi_\theta}_\mu^{\setminus\mu}=\frac{\beta}{N}\sum_{i,j}F_{\mu i}F_{\mu j}({{\cal K}^{\setminus\mu}}^{-1})_{ij}.
\end{eqnarray}
As explained in \ref{sec:app_GLM_cumulant},
the $k$-th derivative of $a(\theta)$ corresponds to the $k$-th cumulant of the
distribution $f(y|\theta)$ in case of the exponential family distribution.
We denote the $k$-th cumulant of the distribution $f(\cdot|\theta)$ as $\kappa_k[f(\cdot|\theta)]$, and obtain the following expression:
\begin{eqnarray}
\nonumber
\sum_{k=2}^\infty \mathrm{F}\chi_k
&=-\frac{1}{M}\sum_{\mu=1}^M\sum_{k=2}^\infty\frac{1}{k!}\kappa_k[f(\cdot|\widehat{\theta}_\mu^{\setminus\mu})]\left\{{\chi_\theta}_\mu^{\setminus\mu}\left(y_\mu-a^\prime(\widehat{\theta}_\mu^{\setminus\mu})\right)\right\}^k\\
\nonumber
&=-\frac{1}{M}\sum_{\mu=1}^M{\cal C}_{f(\cdot|\widehat{\theta}_\mu^{\setminus\mu})}\left({\chi_\theta}_\mu^{\setminus\mu}(y_\mu-a^\prime(\widehat{\theta}_\mu^{\setminus\mu}))\right)\\
&\hspace{1.5cm}+\frac{1}{M}\sum_{\mu=1}^Ma^\prime(\widehat{\theta}_\mu^{\setminus\mu})
{\chi_\theta}_\mu^{\setminus\mu}\left(y_\mu-a^\prime(\widehat{\theta}_\mu^{\setminus\mu})\right),
\end{eqnarray}
where 
${\cal C}_{f(\cdot|\theta)}(t)$ is the cumulant generating function;
\begin{eqnarray}
{\cal C}_{f(\cdot|\theta)}(t)&=\ln\int dy \exp(ty)f(y|\theta)
=a(\theta+t)-a(\theta).
\end{eqnarray}
In summarize, we obtain the following form of the 
estimator of the generalization gap for LOOCV error:
\begin{eqnarray}
\widehat{\Delta}_{\mathrm{LOOCV}}=\mathrm{F}\chi \mathrm{s},
\label{eq:gap_final}
\end{eqnarray}
where
\begin{eqnarray}
\nonumber
\mathrm{F}\chi\mathrm{s}&=\mathrm{cFV}
-\frac{1}{M}\sum_{\mu=1}^M\Big\{a\left(\widehat{\theta}_\mu^{\setminus\mu}+{\chi_\theta}_\mu^{\setminus\mu}(y_\mu-a^\prime(\widehat{\theta}_\mu^{\setminus\mu}))\right)\\
&\hspace{3.0cm}-a(\widehat{\theta}_\mu^{\setminus\mu})-a^\prime(\widehat{\theta}_\mu^{\setminus\mu})
{\chi_\theta}_\mu^{\setminus\mu}\left(y_\mu-a^\prime(\widehat{\theta}_\mu^{\setminus\mu})\right)\Big\}.
\label{eq:Fchi_all_final}
\end{eqnarray}
The remaining task is the 
computation of ${\cal K}^{\setminus\mu}$ and 
associated ${{\chi}_\theta}_\mu^{\setminus\mu}$ for $\mu\in\bm{M}$.
We skip this $M$-times computation using the Sherman--Morrison formula 
(\ref{eq:Sherman-Morrison}).
Using the formula, we obtain
\begin{eqnarray}
{\chi_\theta}_\mu^{\setminus\mu}=\frac{{\chi_\theta}_\mu}{1+d_\mu{\chi_\theta}_\mu},
\end{eqnarray}
hence we do not need the computation of the inverse of matrix for each $\mu\in\bm{M}$.

\begin{figure}
\begin{minipage}{0.495\hsize}
\includegraphics[width=2.5in]{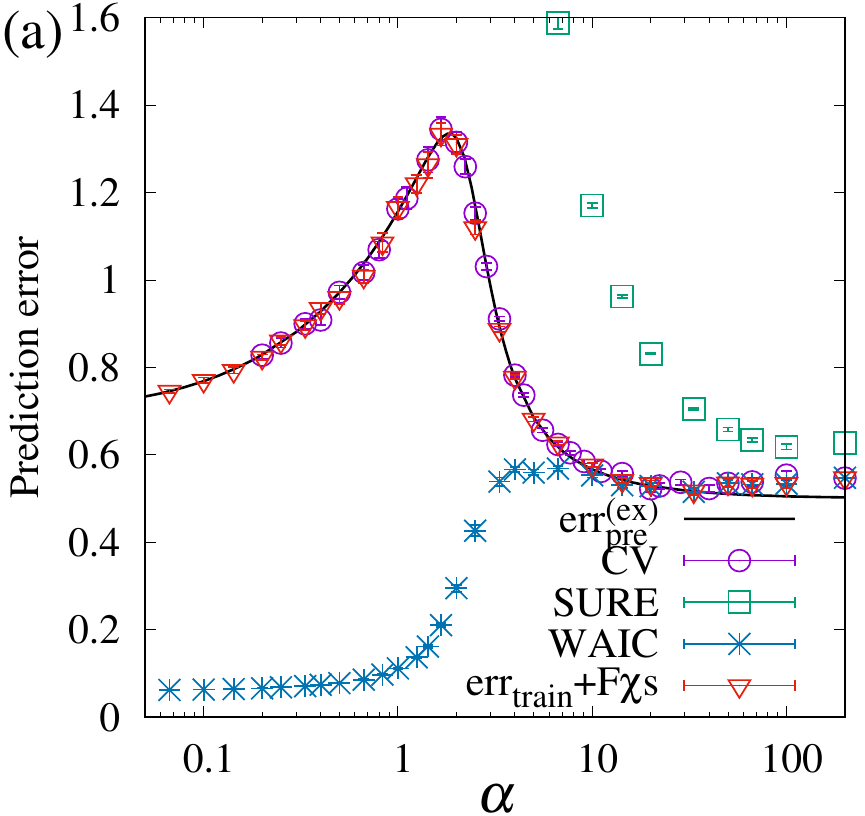}
\end{minipage}
\begin{minipage}{0.495\hsize}
\includegraphics[width=2.5in]{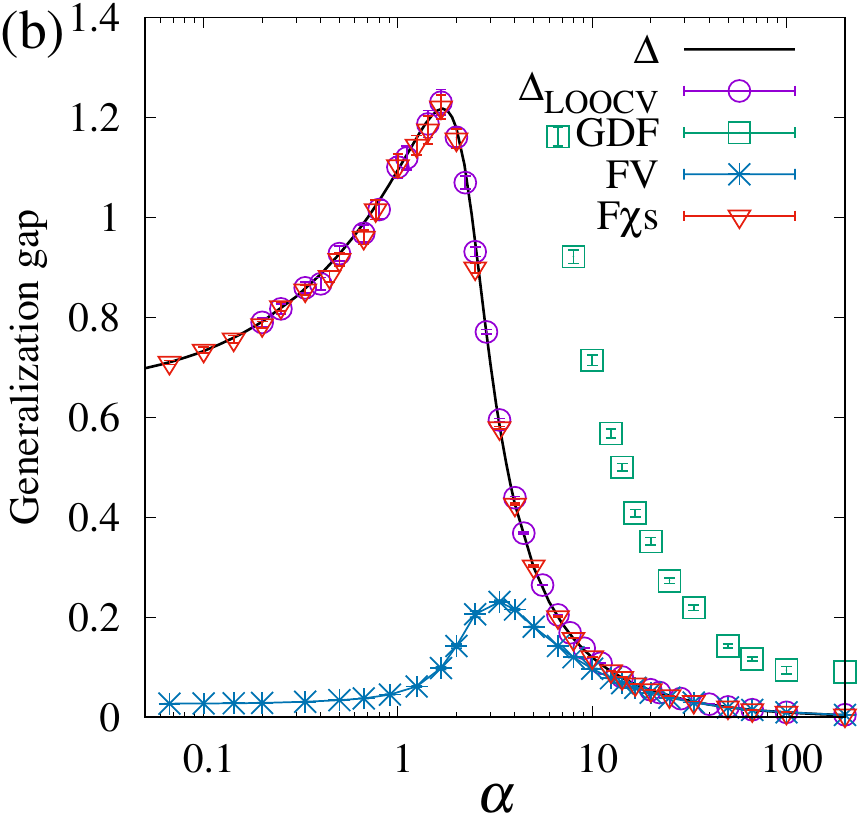}
\end{minipage}
\caption{(a) Prediction errors and (b) generalization gaps for logistic regression
with $\ell_2$ penalty
at $N = 200$, $\lambda = 0.01$, and $\sigma = 1$.
LOOCV error ($\mathrm{err}_{\mathrm{LOOCV}}$) calculated by CD algorithm, 
SURE, WAIC, and $\mathrm{err}_{\mathrm{train}}+\mathrm{F}\chi \mathrm{s}$
calculated by GAMP
are denoted by $\bigcirc$, $\square$, $*$ and $\triangledown$ in (a), respectively,
and corresponding generalization gaps are shown in (b) with the same markers.
The solid line denotes the analytically derived results obtained using the 
replica method.}
\label{fig:logi_Gauss}
\end{figure}

In Fig. \ref{fig:logi_Gauss},
the GAMP-based estimators of
(a) prediction error and (b) generalization gap 
are shown for the 
logistic regression with the $\ell_2$ penalty at $N = 200$ and $\lambda = 0.01$,
where $a^{\prime(n)}(\theta)$ for $n\geq 3$ have finite values.
We set the true value of the regression coefficient $\bm{x}^{(0)}$ 
and generate the components according to the Gaussian distribution
with mean 0 and variance 1.
The data $\bm{y}\in[0,1]^M$ is generated as 
$y_\mu=\mathbb{I}(\theta_\mu^{(0)}+\varepsilon_\mu)$,
where $\bm{\theta}^{(0)}=\frac{1}{\sqrt{N}}\bm{F}\bm{x}^{(0)}$ and
$\varepsilon_\mu\sim{\cal N}(0,\sigma^2)$.
For comparison, we show the estimates obtained by CD algorithm denoted by
$\bigcirc$,
and the analytical results of the extra-sample prediction error 
obtained using the replica method by solid line.
CD algorithm for the GLM including logistic regression is explained in \ref{sec:app_CD_GLM}.
As shown in Fig.\ref{fig:logi_Gauss}, 
the GAMP-based generalization gap given by (\ref{eq:gap_final})
shows good agreement with the results by CD algorithm and replica method 
for any $\alpha$.

We summarize the procedure for obtaining the 
estimator for the generalization gap of 
the LOOCV error in Algorithm \ref{alg:AMP_CV},
where the notation $\mathrm{GAMP}({\cal D})$ in the algorithm
denotes to obtain the estimates under given data ${\cal D}$
by running the GAMP of Algorithm \ref{alg:AMP}.
\begin{algorithm}[h]
\caption{Generalization gap for CV by GAMP}
\label{alg:AMP_CV}
\begin{algorithmic}[1]
\Require {${\cal D}=\{\bm{y}$, $\bm{F}\}$, and $\{\widehat{\bm{x}},\widehat{\bm{\theta}},\bm{s}_\theta\}\gets\mathrm{GAMP}({\cal D})$}
\Ensure {$\widehat{\Delta}_{\mathrm{LOOCV}}$}
\For {$\mu=1,\ldots,M$}
\State{$\widehat{\theta}_\mu^{\setminus\mu}\gets\widehat{\theta}_\mu-{s_{\theta}}_\mu(y_\mu-a^\prime(\widehat{\theta}_\mu))$}
\State{${\chi_\theta}_\mu\gets\displaystyle\frac{{s_\theta}_\mu}{1+{s_\theta}_\mu a^{\prime\prime}(\widehat{\theta}_\mu)}$}
\State{$d_\mu\gets\displaystyle\frac{a^{\prime}(\widehat{\theta}_\mu)-a^{\prime}(\widehat{\theta}_\mu^{\setminus\mu})}{\widehat{\theta}_{\mu}-\widehat{\theta}_{\mu}^{\setminus\mu}}
-a^{\prime\prime}(\widehat{\theta}_\mu)$}
\State{${\chi_\theta}_\mu^{\setminus\mu}\gets\displaystyle\frac{{\chi_\theta}_\mu}{1+d_\mu{\chi_\theta}_\mu}$}
\EndFor
\State{$\mathrm{cFV}\gets \displaystyle\frac{1}{M}\sum_{\mu=1}^M
(y_\mu-a^\prime(\widehat{\theta}_\mu^{\setminus\mu}))^2{\chi_\theta}_\mu^{\setminus\mu}$}
\State{$\mathrm{F}\chi\gets-\displaystyle\frac{1}{M}\sum_{\mu=1}^M \left\{a\left(\widehat{\theta}_\mu^{\setminus\mu}+(y_\mu-a^\prime(\widehat{\theta}_\mu^{\setminus\mu})){\chi_\theta}_\mu^{\setminus\mu}
\right)-a(\widehat{\theta}_\mu^{\setminus\mu})
-a^\prime(\widehat{\theta}_\mu^{\setminus\mu})(y_\mu-a^\prime(\widehat{\theta}_\mu^{\setminus\mu})){\chi_\theta}_\mu^{\setminus\mu}
\right\}$} 
\State{$\widehat{\Delta}_{\mathrm{LOOCV}}\gets \mathrm{cFV}+\mathrm{F}\chi$}
\end{algorithmic}
\end{algorithm}

\subsubsection{Calculation of $\mathrm{cFV}$ and $\mathrm{F}\chi\mathrm{s}$
without GAMP}

The problem of GAMP is that it does not converge when the stability 
condition (\ref{eq:AT_GAMP_final}) is not satisfied.
In the parameter region where GAMP does not converge,
there could be a local minima of 
the exponential order with respect to $N$, as indicated by the correspondence with the de Almeida--Thouless instability \cite{AT}. However, even in such cases,
one can obtain the GAMP-inspired estimator without using GAMP
at each minimum obtained by the arbitrary algorithm
that can converge to a local minimum.
The characteristic variables of GAMP are 
$s_i$, ${s_\theta}_\mu$, and $\partial_{\widehat{\theta}}g_{\mathrm{out}}$,
in the sense that general algorithms do not contain these variables.
If we obtain $s_i$ numerically, the other variables 
${s_\theta}_\mu$, ${\chi_\theta}_\mu$, and ${\chi_\theta}_\mu^{\setminus\mu}$
can be obtained, and further, $\partial_{\widehat{\theta}}g_{\mathrm{out}}$
can also be obtained. Thus, to mimic the calculation of $s_i$
is important to extend the GAMP-based estimator for other algorithms.
This calculation of $s_i$ is resolved using the relationship
(\ref{eq:J_vs_variance}).

\begin{algorithm}[h]
\caption{GAMP-inspired generalization gap without GAMP}
\label{alg:AMP_CV_without_AMP}
\begin{algorithmic}[1]
\Require {${\cal D}=\{\bm{y}$, $\bm{F}\}$, and $\{\widehat{\bm{x}},\widehat{\bm{\theta}}\}\gets\mathrm{arbitrary~algorithm}({\cal D})$}
\Ensure {$\widehat{\Delta}_{\mathrm{LOOCV}}$}
\State{$\bm{D}\gets\mathrm{diag}(a^{\prime\prime}(\widehat{\bm{\theta}}))$}
\State{$\bm{\chi}\gets\displaystyle\left(\frac{1}{N}\bm{F}^\top\bm{D}\bm{F}+\frac{\partial^2}{\partial\bm{x}\partial\bm{x}^\top}h(\bm{x})\right)^{-1}$}
\Comment{For sparse estimators, use $\bm{\chi}^{{\cal L}}$.}
\For {$i=1,...,N$}
\State{$s_i\gets \chi_{ii}$}
\Comment{For sparse estimators, $s_i\gets\chi^{\cal L}_{ii}$ for $i\in{\cal L}$, otherwise  $s_i\gets0$.}
\EndFor
\For {$\mu=1,\ldots,M$}
\State{${s_\theta}_\mu\gets\displaystyle\frac{1}{N}\sum_{i=1}^NF_{\mu i}^2s_{i}$}
\State{${\chi_\theta}_\mu\gets\displaystyle\frac{{s_\theta}_\mu}{1+{s_\theta}_\mu a^{\prime\prime}(\theta_\mu)}$}
\State{$\widehat{\theta}_\mu^{\setminus\mu}\gets\widehat{\theta}_\mu-{s_{\theta}}_\mu(y_\mu-a^\prime(\widehat{\theta}_\mu))$}
\State{$d_\mu\gets\displaystyle\frac{a^{\prime}(\widehat{\theta}_\mu)-a^{\prime}(\widehat{\theta}_\mu^{\setminus\mu})}{\widehat{\theta}_{\mu}-\widehat{\theta}_{\mu}^{\setminus\mu}}
-a^{\prime\prime}(\widehat{\theta}_\mu)$}
\State{${\chi_\theta}_\mu^{\setminus\mu}\gets\displaystyle\frac{{\chi_\theta}_\mu}{1+d_\mu{\chi_\theta}_\mu}$}
\EndFor
\State{$\mathrm{cFV}\gets \displaystyle\frac{1}{M}\sum_{\mu=1}^M
(y_\mu-a^\prime(\widehat{\theta}_\mu^{\setminus\mu}))^2{\chi_\theta}_\mu^{\setminus\mu}$}
\State{$\mathrm{F}\chi\gets-\displaystyle\frac{1}{M}\sum_{\mu=1}^M \left\{a\left(\widehat{\theta}_\mu^{\setminus\mu}+(y_\mu-a^\prime(\widehat{\theta}_\mu^{\setminus\mu})){\chi_\theta}_\mu^{\setminus\mu}
\right)-a(\widehat{\theta}_\mu^{\setminus\mu})
-a^\prime(\widehat{\theta}_\mu^{\setminus\mu})(y_\mu-a^\prime(\widehat{\theta}_\mu^{\setminus\mu})){\chi_\theta}_\mu^{\setminus\mu}
\right\}$} 
\State{$\widehat{\Delta}_{\mathrm{LOOCV}}\gets \mathrm{cFV}+\mathrm{F}\chi$}
\end{algorithmic}
\end{algorithm}

In Algorithm \ref{alg:AMP_CV_without_AMP},
we showed the procedure to obtain the GAMP-based estimator without
using GAMP.
Here, we consider that the algorithm returns the estimates 
$\widehat{\bm{x}}$ and $\widehat{\bm{\theta}}$
corresponding to the input data ${\cal D}$.
The calculation of $s_i$ is implemented in the 2nd line of the Algorithm \ref{alg:AMP_CV_without_AMP},
where $\chi_{ii}$ corresponds to $s_i$.
In the case of sparse estimators which contains zero components,
such as the regression under the $\ell_1$ or the elastic net regularization,
it is natural to set $s_i=0$ for $i$ with $\widehat{x}_i=0$,
assuming that the set of zero components do not change
against the small perturbation on data.
We denote the support set as ${\cal L}$,
where $x_i\neq 0$ for $i\in{\cal L}$;
otherwise, $x_i=0$.
For the calculation of $s_i$ for $i\in{\cal L}$,
we define the submatrix of the predictor $\bm{F}$ as $\bm{F}^{{\cal L}}$
that comprises row vectors in ${\cal L}$ \cite{Obuchi-Kabashima}.
We define the Hessian matrix for the support set as
\begin{eqnarray}
{\cal H}^{\cal L}=\frac{\partial^2}{\partial{\bm{x}^{\cal L}}\partial{\bm{x}^{\cal L}}^\top}
\left\{\ln f\left(\bm{y}\Big|\frac{1}{\sqrt{N}}\bm{F}^{{\cal L}}\bm{x}^{\cal L}\right)-\phi(\bm{x}^{\cal L})\right\}
\end{eqnarray}
and the inverse $\bm{\chi}^{\cal L}=({\cal H}^{\cal L})^{-1}$; we
set 
\begin{eqnarray}
s_i=\left\{
\begin{array}{ll}
\chi_{ii}^{\cal L} & i\in{\cal L}\\
0 & \mathrm{otherwise}.
\end{array}
\right.
\end{eqnarray}

\subsection{Correction of cavity-variance for correlated predictors}

The correlation between predictors is neglected in GAMP, 
as shown in assumption {\bf a1} and {\bf a2},
and therefore, there is no mathematical support for the 
accuracy of the estimators for the non-i.i.d. predictors.
Here, we consider two non-i.i.d. predictors
and apply the GAMP-based calculation of the 
estimators for the generalization gap.
\begin{itemize}
\item 
Correlated Gaussian predictor:
The row vectors of the predictor matrix are
generated by ${\cal N}(\bm{0},{\cal V})$
with a covariance matrix ${\cal V}\in\mathbb{R}^{N\times N}$.
We set the component of ${\cal V}$ as ${\cal V}_{ij}=\sigma_d^{-|i-j|}$
with $\sigma_d>0$; i.e., the power law decay in the distance.

\item Rank deficient predictor:
We generate i.i.d. Gaussian predictor
$\bm{G}\in\mathbb{R}^{M\times N}$ and apply 
singular value decomposition as
$\bm{G}=\bm{U}\bm{S}\bm{V}^\top$,
where $\bm{U}\in\mathbb{R}^{M\times r}$ and $\bm{V}\in\mathbb{R}^{r\times N}$ 
are the set of left and right singular vectors, respectively,
and $r=\min(M,N)$.
The matrix $\bm{S}\in\mathbb{R}^{r\times r}$ is a diagonal matrix 
whose $(i,i)$ component is the $i$-th singular value.
We choose singular values with fraction $r\times (1-\rho_F)$,
and set them as 0.
We denote the matrix of singular values with the
zero components as $\widetilde{\bm{S}}$,
and construct the rank-deficient predictor matrix by 
$\bm{F}=\bm{U}\widetilde{\bm{S}}\bm{V}$
with the normalization $\sum_{\mu=1}^MF_{\mu i}^2=M$
for any $i$.
\end{itemize}
Under these predictor matrix, we consider the same generating rule of the data $\bm{y}$
as with the previous sections.

Fig. \ref{fig:lambda001_corrGauss} shows
the generalization gaps at $N = 200$, $\lambda=0.01$ and $\sigma=1$ for 
(a) correlated Gaussian predictors at $\sigma_d = 0.5$ 
and (b) rank-deficient predictors at $\rho_F = 0.9$
under the 
ridge regression.
As the likelihoods are Gaussian distribution,
$\mathrm{cFV}+\mathrm{F}\chi_2$ is theoretically sufficient for the description of the
generalization gap,
but it overestimates the generalization gap $\Delta_{\mathrm{LOOCV}}$ 
in the small $\alpha$ region.
Fig. \ref{fig:lambda001_EN_corr} shows the generalization gap for the 
linear regression elastic net penalty at $N=400$ and $\lambda_2 = 0.01$.
As with Fig. \ref{fig:lambda001_corrGauss},
the GAMP-based estimator overestimates the actual value of the $\Delta_{\mathrm{LOOCV}}$.

\begin{figure}
\begin{minipage}{0.495\hsize}
\centering
\includegraphics[width=2.5in]{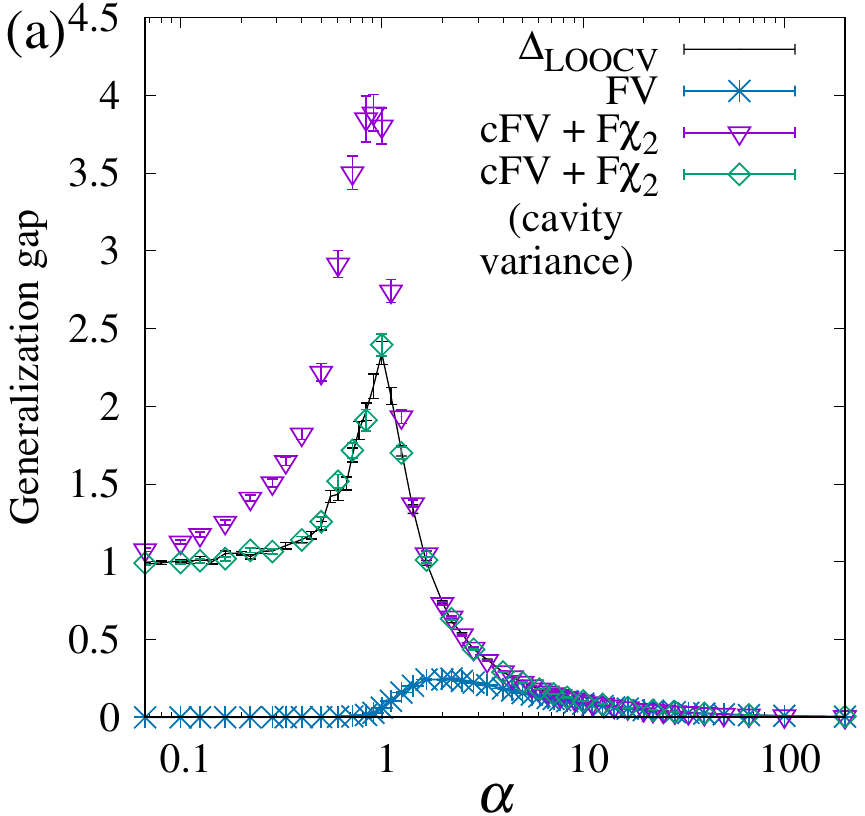}
\end{minipage}
\begin{minipage}{0.495\hsize}
\includegraphics[width=2.5in]{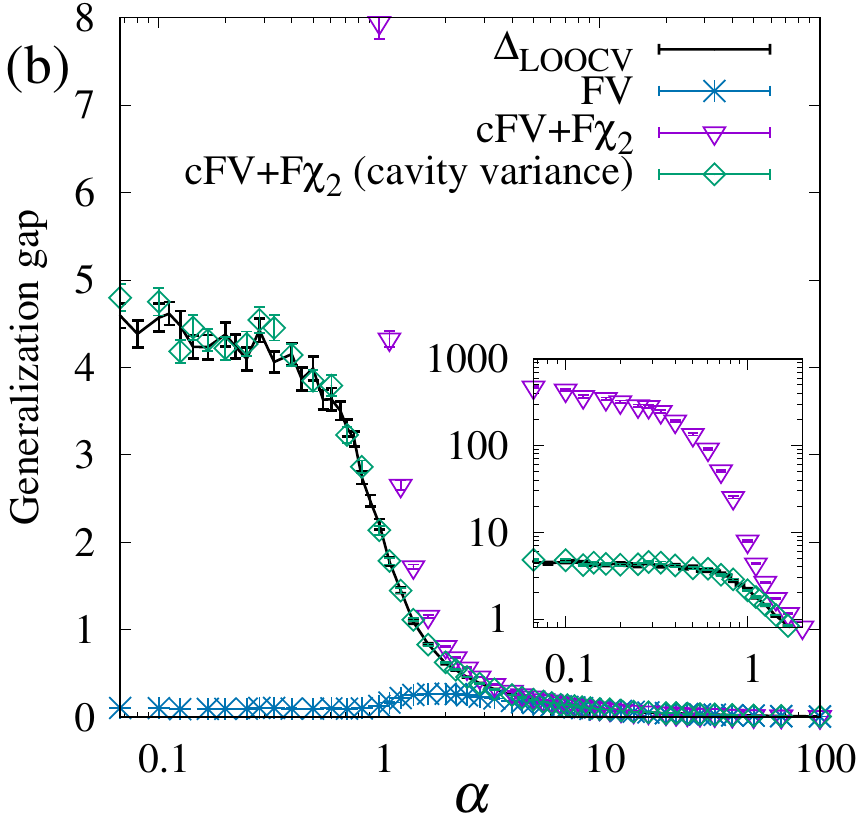}
\end{minipage}
\caption{Generalization gaps under the ridge regression
at $M = 200$, $\sigma = 1$, and $\lambda = 0.01$. 
$\Delta_{\mathrm{LOOCV}}$ 
in the linear algebraic form is denoted by a  
solid line, and FV ($*$), cFV+F$\chi_2$ ($\triangledown$), 
and cFV+cF$\chi_2$ with cavity variance ($\diamond$) calculated by GAMP
are shown.
(a) and (b) are for the correlated Gaussian predictor with $\sigma_d = 0.5$, 
and rank-deficient predictor with $\rho_F = 0.9$, respectively.
Each data point is averaged over 100 samples of ${\cal D}$,
and the bars on the points represent standard errors.}
\label{fig:lambda001_corrGauss}
\end{figure}
\begin{figure}
\begin{minipage}{0.495\hsize}
\centering
\includegraphics[width=2.5in]{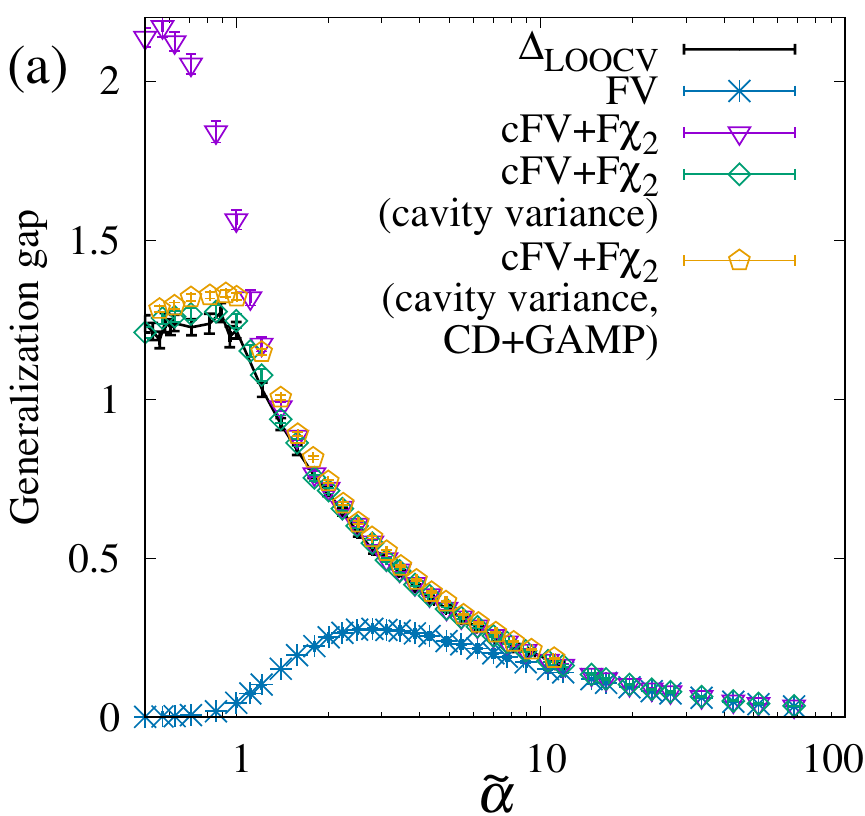}
\end{minipage}
\begin{minipage}{0.495\hsize}
\includegraphics[width=2.5in]{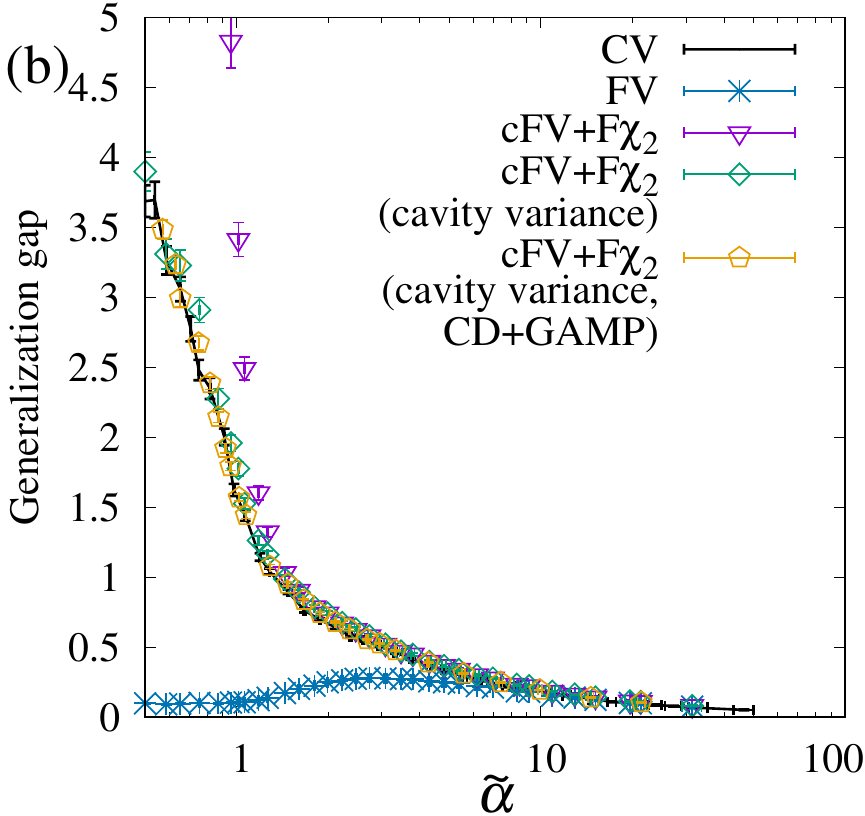}
\end{minipage}
\caption{Generalization gaps under linear regression
with elastic net regularization
at $M = 200$, $N=400$, $\lambda_2 = 0.01$, $\sigma = 1$, and $\rho = 1$.
LOOCV obtained using the CD algorithm is denoted by
the solid line; FV ($*$), cFV+F$\chi_2$ ($\triangledown$), 
and cFV+cF$\chi_2$ with cavity variance ($\diamond$) calculated by GAMP
are shown.
The GAMP-inspired generalization gap using the estimates given by
CD algorithm is denoted by pentagons.
(a) and (b) are for the correlated Gaussian predictor with $\sigma_d = 0.5$, 
and the rank-deficient predictor with $\rho_F = 0.9$, respectively,
and $\lambda_1$ is controlled to set $\widetilde{\alpha}=M\slash||\widehat{\bm{x}}||_0$.
Each data point is averaged over 100 samples of ${\cal D}$,
and bars on the points represent the standard errors.}
\label{fig:lambda001_EN_corr}
\end{figure}

A part of the discrepancies between 
$\widehat{\Delta}_{\mathrm{LOOCV}}$ and $\Delta_{\mathrm{LOOCV}}$
attributed to the non-i.i.d.-ness in predictor matrices
is improved by considering the differences between the variables
${s_\theta}_\mu$ and ${s_\theta}_\mu^{\setminus\mu}$,
which we call the cavity variance, 
and it cannot be determined in the framework of GAMP.
As shown in (\ref{eq:s_theta}),
we ignore their difference as it scales as $O(N^{-1})$ for the 
random predictors satisfying assumptions {\bf a1} and {\bf a2}; 
however, for the correlated predictors, this is not the case.
GAMP do not have closed equations to determine ${s_\theta}^{\setminus\mu}$, and 
therefore, we  heuristically derive the form of ${s_\theta}_\mu^{\setminus\mu}$
using the relationship (\ref{eq:J_vs_variance}).

By definition,
\begin{eqnarray}
{s_\theta}_\mu^{\setminus\mu}=\frac{\beta}{N}\sum_{ij}F_{\mu i}F_{\mu j}
\left(\langle x_ix_j\rangle_{\beta}^{\setminus\mu}-\langle x_i\rangle_{\beta}^{\setminus\mu}\langle x_j\rangle_{\beta}^{\setminus\mu}\right),
\end{eqnarray}
where $\langle\cdot\rangle_\beta^{\setminus\mu}$
denotes the posterior average under the LOO sample ${\cal D}_{\setminus\mu}$,
\begin{eqnarray}
p_\beta(\bm{x}|{\cal D}_{\setminus\mu})
\propto \prod_{\nu\in\bm{M}\neq\mu} f^\beta(y_\nu|\bm{\theta}(\bm{x}))\phi^\beta(\bm{x}).
\end{eqnarray}
As with (\ref{eq:J_vs_variance}), we obtain
\begin{eqnarray}
\beta\left(\langle x_ix_j\rangle^{\setminus\mu}_{\beta}-\langle x_i\rangle^{\setminus\mu}_{\beta}\langle x_j\rangle^{\setminus\mu}_{\beta}\right)=
\frac{1}{M}({{\cal J}^{\varphi\setminus\mu}(\widehat{\bm{x}}({\cal D}_{\setminus\mu}))}^{-1})_{ij},
\end{eqnarray}
where
\begin{eqnarray}
{\cal J}_{ij}^{\varphi\setminus\mu}(\widehat{\bm{x}}({\cal D}_{\setminus\mu}))
=-\frac{1}{M}
\frac{\partial^2}{\partial x_i\partial x_j}
\left\{\sum_{\nu\in\bm{M}\setminus\mu}\ln f(y_\nu|\bm{\theta}(\bm{x}))-h(\bm{x})\right\}
\Big|_{\bm{x}=\widehat{\bm{x}}({\cal D}_{\setminus\mu})}.
\end{eqnarray}
We avoid the $\mu$-times computation of ${{\cal J}^{\varphi\setminus\mu}}^{-1}$
for $\mu\in\bm{M}$ by utilizing 
the Sherman--Morrison formula (\ref{eq:Sherman-Morrison}).
Setting $\widetilde{\bm{F}}=\sqrt{\bm{D}(\bm{\theta}^{\setminus\nu})}\bm{F}$,
where $\bm{D}(\bm{\theta}^{\setminus\nu})$ is the diagonal matrix whose $(\mu,\mu)$ component is 
$\sqrt{a^{\prime\prime}(\theta_\mu^{\setminus\nu})}$,
we obtain
\begin{eqnarray}
\nonumber
{s_\theta}_\nu^{\setminus\nu}&={\chi_\theta}_\nu
+\frac{\frac{1}{N}||\widetilde{\textbf{f}}_\nu(\frac{1}{N}\widetilde{\bm{F}}^{\top}\widetilde{\bm{F}}+\lambda\bm{I}_N)^{-1}\textbf{f}_\nu^{\top}||_2^2}
{1-\frac{1}{N}\widetilde{\textbf{f}}_\nu(\frac{1}{N}\widetilde{\bm{F}}^{\top}\widetilde{\bm{F}}+\lambda\bm{I}_N)^{-1}\widetilde{\textbf{f}}_\nu^{\top}}\\
&
=\frac{{\chi_\theta}_\nu}{1-a^{\prime\prime}(\theta_\nu^{\setminus\nu}){\chi_\theta}_\nu}.
\label{eq:cavity_variance_final}
\end{eqnarray}
When we put 
${\chi_{\theta}}_\nu$ derived GAMP, (\ref{eq:GLM_theta_variance}),
into (\ref{eq:cavity_variance_final}),
we obtain the equality ${s_\theta}_\nu^{\setminus\nu}={s_\theta}_\nu$,
and hence, the computation of the cavity variance should be
based on Algorithm \ref{alg:AMP_CV_without_AMP}.

\begin{figure}
\begin{minipage}{0.495\hsize}
\centering
\includegraphics[width=2.5in]{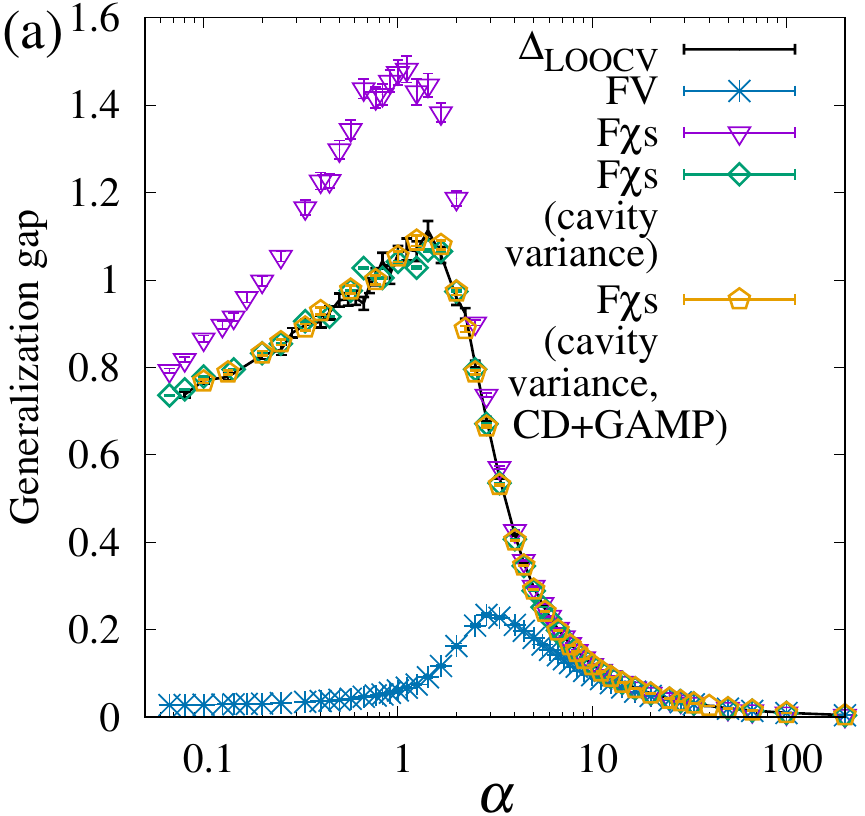}
\end{minipage}
\begin{minipage}{0.495\hsize}
\includegraphics[width=2.5in]{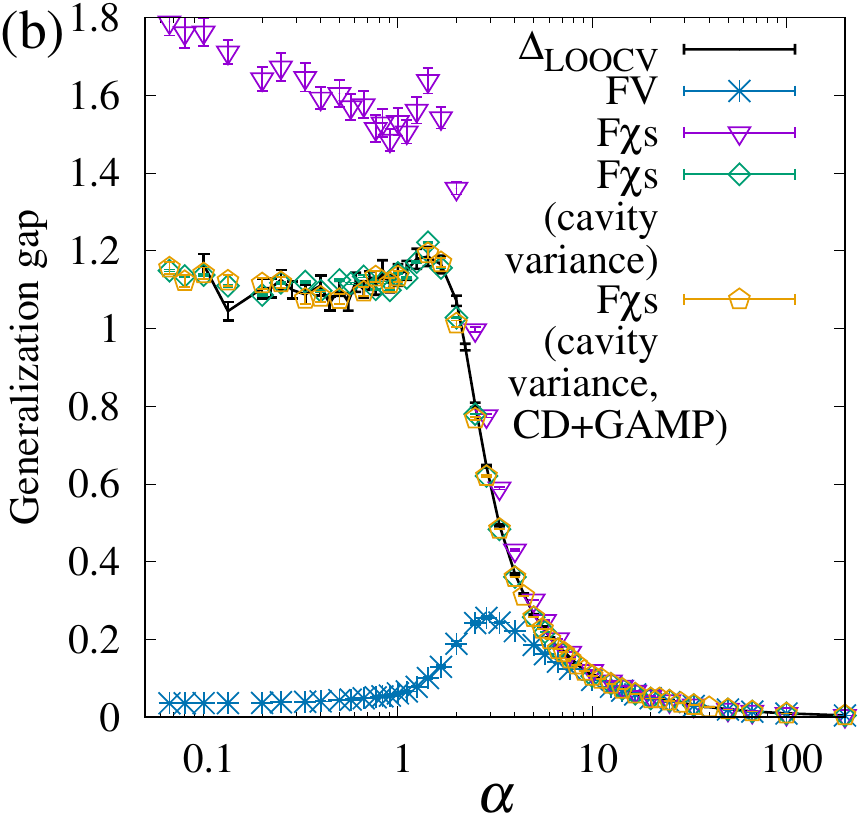}
\end{minipage}
\caption{Generalization gaps under the logistic regression
with $\ell_2$ regularization
at $M = 200$, $\lambda = 0.01$, and $\sigma = 1$.
LOOCV obtained by the CD algorithm is denoted by
a solid line; FV ($*$), F$\chi$s ($\triangledown$), 
and F$\chi$s with cavity variance ($\diamond$) calculated by GAMP
are shown.
The GAMP-inspired generalization gap using the estimates given by CD algorithm
is denoted by pentagons.
(a) and (b) are for the correlated Gaussian predictor with $\sigma_d = 0.5$, 
and the rank-deficient predictor with $\rho_F = 0.9$, respectively.
Each data point is averaged over 100 samples of ${\cal D}$,
and the bars on the points represent standard errors.}
\label{fig:lambda001_Logi_corr}
\end{figure}

The behaviours of the estimators for the generalization gap 
considering the cavity variance at $N = 200$
are shown in Figs. \ref{fig:lambda001_corrGauss},
\ref{fig:lambda001_EN_corr},
and \ref{fig:lambda001_Logi_corr}
for the ridge regression, linear regression with elastic net penalty,
and logistic regression with $\ell_2$ penalty, respectively.
In these examples, the 
accuracy of the estimators are significantly improved by
considering the cavity variance.

The procedure to obtain the generalization gap 
considering the cavity variance is summarised in Algorithm \ref{alg:AMP_CV_without_AMP_cavity}.

\begin{algorithm}[h]
\caption{GAMP-inspired generalization gap without GAMP (with cavity variance)}
\label{alg:AMP_CV_without_AMP_cavity}
\begin{algorithmic}[1]
\Require {${\cal D}=\{\bm{y}$, $\bm{F}\}$, and $\{\widehat{\bm{x}},\widehat{\bm{\theta}}\}\gets\mathrm{arbitrary~algorithm}({\cal D})$}
\Ensure {$\widehat{\Delta}_{\mathrm{LOOCV}}$}
\State{$\bm{D}\gets\mathrm{diag}(a^{\prime\prime}(\bm{\theta}))$}
\State{$\bm{\chi}\gets\displaystyle\left(\frac{1}{N}\bm{F}^\top\bm{D}\bm{F}+\frac{\partial^2}{\partial\bm{x}\partial\bm{x}^\top}h(\bm{x})\right)^{-1}$}
\Comment{For the sparse estimators, use $\bm{\chi}^{{\cal L}}$.}
\For{$i=1,\ldots,N$}
\State{$s_i\gets \chi_{ii}$}
\Comment{For sparse estimators, $s_i\gets\chi^{\cal L}_{ii}$ for $i\in{\cal L}$, otherwise  $s_i\gets0$.}
\EndFor
\For{$\mu=1,\ldots,M$}
\State{${s_\theta}_\mu\gets\displaystyle\frac{1}{N}\sum_{i=1}^NF_{\mu i}^2\chi_{ii}$}
\State{${\chi_\theta}_\mu\gets\displaystyle\frac{{s_\theta}_\mu}{1+{s_\theta}_\mu a^{\prime\prime}(\theta_\mu)}$}
\State{${s_{\theta}}_\mu^{\setminus\mu}\gets\displaystyle\frac{{\chi_{\theta}}_\mu}{1-a^{\prime\prime}(\widehat{\theta}_\mu^{\setminus\mu}){\chi_{\theta}}_\mu}$}
\State{$\widehat{\theta}_\mu^{\setminus\mu}\gets\widehat{\theta}_\mu-{s_{\theta}}_\mu^{\setminus\mu}(y_\mu-a^\prime(\widehat{\theta}_\mu))$}
\State{$d_\mu\gets\displaystyle\frac{a^{\prime}(\widehat{\theta}_\mu)-a^{\prime}(\widehat{\theta}_\mu^{\setminus\mu})}{\widehat{\theta}_{\mu}-\widehat{\theta}_{\mu}^{\setminus\mu}}
-a^{\prime\prime}(\widehat{\theta}_\mu)$}
\State{${\chi_\theta}_\mu^{\setminus\mu}\gets\displaystyle\frac{{\chi_\theta}_\mu}{1+d_\mu{\chi_\theta}_\mu}$}
\EndFor
\State{$\mathrm{cFV}\gets \displaystyle\frac{1}{M}\sum_{\mu=1}^M
(y_\mu-a^\prime(\widehat{\theta}_\mu^{\setminus\mu}))^2{\chi_\theta}_\mu^{\setminus\mu}$}
\State{$\mathrm{F}\chi\gets-\displaystyle\frac{1}{M}\sum_{\mu=1}^M \left\{a\left(\widehat{\theta}_\mu^{\setminus\mu}+(y_\mu-a^\prime(\widehat{\theta}_\mu^{\setminus\mu})){\chi_\theta}_\mu^{\setminus\mu}
\right)-a(\widehat{\theta}_\mu^{\setminus\mu})
-a^\prime(\widehat{\theta}_\mu^{\setminus\mu})(y_\mu-a^\prime(\widehat{\theta}_\mu^{\setminus\mu})){\chi_\theta}_\mu^{\setminus\mu}
\right\}$} 
\State{$\widehat{\Delta}_{\mathrm{LOOCV}}\gets \mathrm{cFV}+\mathrm{F}\chi$}
\end{algorithmic}
\end{algorithm}

\section{Replica method}
\label{sec:replica}

For a part of predictor matrices, 
the prediction error and generalization gap can be analytically
derived by using replica method.
In this section, we show the analysis for the 
random predictor matrix $\bm{F}$ whose components are 
independently and identically distributed according to the Gaussian distribution.
In such cases, it is shown that the 
typical trajectory of the time evolution of the GAMP,
termed as state evolution,
corresponds to the saddle point equations of the 
replica method
\cite{Krzakala2012}.
The rigorousness of the state evolution for GAMP 
has been proved under i.i.d. Gaussian data \cite{Javanmard2013},
and correlated Gaussian data \cite{Loureiro2022}.

In the replica analysis,
we derive the expectation value of the errors
with respect to ${\cal D}=\{\bm{y},\bm{F}\}$,
to understand the typical property
under the given ${\cal D}$ \cite{zdeborova2016statistical}.
The basis for the analysis is the free energy density, which is defined as
\begin{eqnarray}
\Psi=-\lim_{\beta\to\infty}\frac{1}{N\beta}\mathbb{E}_{{\cal D}}[\ln Z_{\beta}({\cal D})].
\end{eqnarray}
The 
training error relates to the 
free energy density at zero temperature,
and as discussed in sec.\ref{sec:ISCV_and_FV},
the expression of FV by the replica method
can be derived from the partition function
by introducing auxiliary variables.

In the analysis, we consider the same setting as the numerical simulations 
by GAMP summarized in sec.\ref{sec:FV_correct},
where the data $\bm{y}$ is generated by
a certain rule governed by the `true' parameter 
$\bm{\theta}^{(0)}=\bm{F}\bm{x}^{(0)}\slash\sqrt{N}$ with a
`true' regression coefficient $x_i^{(0)}$.
Here, we consider that the true regression coefficient is generated by 
Bernoulli--Gaussian distribution
\begin{eqnarray}
\phi^{(0)}(x^{(0)})=(1-\rho)\delta(x^{(0)})+\frac{\rho}{\sqrt{2\pi\sigma_x^2}}\exp\left(-\frac{(x^{(0)})^2}{2\sigma_x^2}\right),
\end{eqnarray}
where $\rho\in[0,1]$ is the fraction of nonzero components in the true coefficient.
We assume that the model parameter
is subjected to 
additive 
noise $\bm{\varepsilon}\in\mathbb{R}^M$; hence we denote 
$\bm{y}={\cal G}(\bm{\theta}^{(0)}+\bm{\varepsilon})$,
where ${\cal G}$ is an arbitrary function that describe the data generative process.
The expectation of the free energy density
with respect to $\bm{y},\bm{\varepsilon},\bm{F}$ and $\bm{x}^{(0)}$
is implemented by replica method
based on the following identity
\begin{eqnarray}
\mathbb{E}[\ln Z_{\beta}({\cal D})]=
\lim_{n\to 0}\frac{\partial}{\partial n}\mathbb{E}_{\cal D}[Z_{\beta}^n({\cal D})].
\label{eq:replica}
\end{eqnarray}
We focus on the $N\to\infty$ and $M\to\infty$ limit 
which keep $M\slash N=\alpha\sim O(1)$.
For the sake of convenience, we define free entropy density 
for the finite $\beta$ and finite $n$ as
\begin{eqnarray}
\Phi_\beta(n)=\frac{1}{N}\ln \mathbb{E}_{{\cal D}}[Z_\beta^n({\cal D})].
\end{eqnarray}
Assuming $n$ is a positive integer, we express the $n$-th power
by introducing the $n$-replicated system as
\begin{eqnarray}
\nonumber
\Phi_\beta(n)&=\frac{1}{N}\ln\int d\bm{\varepsilon}P_\varepsilon(\bm{\varepsilon})
\int d\bm{F}P_F(\bm{F})\int d\bm{x}^{(0)}
\phi^{(0)}(\bm{x}^{(0)})\\
&\times\prod_{a=1}^n\int d\bm{x}^{(a)}
f^\beta\left({\cal G}(\bm{\theta}^{(0)}+\bm{\varepsilon})|\bm{\theta}^{(a)}\right)\phi^\beta(\bm{x}^{(a)})
\end{eqnarray}
where we set $\bm{\theta}^{(a)}=\frac{1}{\sqrt{N}}\bm{Fx}^{(a)}$.
For implementing the integrals, 
we introduce the subshells with respect to the macroscopic variables
$\{q^{(ab)}\}$ and $\{m^{(a)}\}$
by inserting trivial identities 
\begin{eqnarray}
\nonumber
1&=\int dq^{(ab)}\delta\left(q^{(ab)}-\frac{1}{N}\sum_{i=1}^Nx_{i}^{(a)}x_{i}^{(b)}\right)\\
&=\int dq^{(ab)}d\widehat{q}^{(ab)}\exp\left(-Nq^{(ab)}\widehat{q}^{(ab)}+\widehat{q}^{(ab)}\sum_{i=1}^Nx_i^{(a)}x_i^{(b)}\right)\\
\nonumber
1&=\int dm^{(a)}\delta\left(m^{(a)}-\frac{1}{N}\sum_{i=1}^Nx_{i}^{(0)}x_{i}^{(a)}\right)\\
&=\int dm^{(a)}d\widehat{m}^{(a)}\exp\left(-Nm^{(a)}\widehat{m}^{(a)}+\widehat{m}^{(a)}\sum_{i=1}^Nx_i^{(0)}x_i^{(a)}\right)
\end{eqnarray}
for all pairs of $(a,b)~(a,b\in\bm{n})$ and $a\in\bm{n}$,
where $\bm{n}=\{1,\ldots,n\}$,
and $\{\widehat{q}^{(ab)}\}$ and $\{\widehat{m}^{(a)}\}$
are conjugate variables for the integral expressions of the delta functions.
The expectation with respect to $\bm{F}$,
whose components are independently and identically
distributed according to Gaussian distribution with mean 0 and variance 1,
within the subshells lead to
\begin{eqnarray}
\mathbb{E}_{\bm{F}|q^{(ab)}}[\theta^{(a)}_{\mu}\theta^{(b)}_{\mu}]&=q^{(ab)}\label{eq:theta_corr_q}\\
\mathbb{E}_{\bm{F}|m^{(a)}}[\theta^{(a)}_{\mu}\theta^{(0)}_{\mu}]&=m^{(a)}\label{eq:theta_corr_m}\\
\mathbb{E}_{\bm{F}}[\theta^{(0)}_{\mu}\theta^{(0)}_{\mu}]&=\rho\sigma_x^2.
\label{eq:theta_corr_0}
\end{eqnarray}
Thus, in the subshell,
the joint distribution of $\widetilde{\bm{\theta}}_\mu\equiv
\{\theta_\mu^{(0)},\theta_\mu^{(1)},\cdots\theta_\mu^{(n)}\}$ is given by
\begin{eqnarray}
P_\theta(\tilde{\bm{\theta}}_\mu|{\cal Q})=\frac{1}{\sqrt{2\pi|{\cal Q}|}}\exp\left(-\frac{1}{2}\widetilde{\bm{\theta}}_\mu^{\top}{\cal Q}^{-1}\widetilde{\bm{\theta}}_\mu\right)
\end{eqnarray}
for any $\mu\in\bm{M}$,
where ${\cal Q}\in\mathbb{R}^{(n+1)\times (n+1)}$
is given by
\begin{eqnarray}
{\cal Q}_{ab}=\left\{\begin{array}{c c}
   q^{(ab)} & a\neq 0~\mbox{and}~b\neq 0\\
   m^{(a)} & a\neq 0~\mbox{and}~b= 0\\
   m^{(b)} & b= 0~\mbox{and}~b\neq 0 \\
   \rho\sigma_x^2 & a=b=0
\end{array}\right..
\end{eqnarray}
Using the joint distribution with respect to $\widetilde{\bm{\theta}}_\mu$,
we obtain
\begin{eqnarray}
\Phi_\beta&(n)=\frac{1}{N}\ln\int \prod_{a\leq b}dq^{(ab)}d\widehat{q}^{(ab)}\prod_{a=1}^ndm^{(a)}d\widehat{m}^{(a)}\\
\nonumber
&\times\exp\left(-N\sum_{a\leq b}q^{(ab)}\widehat{q}^{(ab)}-N\sum_{a=1}^nm^{(a)}\hat{m}^{(a)}\right)\\
\nonumber
&\times\left[\int d\widetilde{\bm{\theta}}\int d\varepsilon P_\varepsilon(\varepsilon)P_\theta(\widetilde{\bm{\theta}}|{\cal Q})
\prod_{a=1}^nf\left({\cal G}(\theta^{(0)}+\varepsilon)|\theta^{(a)}\right)\right]^M\\
\nonumber
&\times\left[\int \!\!d\widetilde{\bm{x}}\!\prod_{a=1}^n\!\phi(x^{(a)})\phi^{(0)}(x^{(0)})
\!\exp\!\!\left(\sum_{a\leq b}\widehat{q}^{(ab)}x^{(a)}x^{(b)}\!+\!\!\sum_{a=1}^n\widehat{m}^{(a)}x^{(0)}x^{(a)}\!\!\right)\!\right]^N,
\end{eqnarray}
where $\widetilde{\bm{x}}=\{x^{(0)},x^{(1)},\cdots x^{(n)}\}$
and 
$P(\varepsilon)$ represents the
distribution of the noise $\varepsilon$.
By applying the saddle point equation,
we obtain
\begin{eqnarray}
\Phi_\beta(n)=\mathop{\mathrm{extr}}_{\{q^{(ab)}\},\{m^{(a)}\},\{\widehat{q}^{(ab)}\},\{\widehat{m}^{(a)}\}}\left\{\Phi_\beta^{(C)}(n)+\alpha\Phi^{(E)}_\beta(n)+\Phi^{(S)}_\beta(n)\right\},
\label{eq:free_entropy_n}
\end{eqnarray}
where
\begin{eqnarray}
\Phi_\beta^{(C)}(n)&=-\sum_{a\leq b}q^{(ab)}\widehat{q}^{(ab)}-\sum_{a=1}^nm^{(a)}\widehat{m}^{(a)}\\
\Phi_\beta^{(E)}(n)&=\ln\int d\varepsilon P_\varepsilon(\varepsilon)\int d\widetilde{\bm{\theta}} P_\theta(\widetilde{\bm{\theta}}|{\cal Q})
\prod_{a=1}^nf^{\beta}\left({\cal G}(\theta^{(0)}+\varepsilon)|\theta^{(a)}\right)\label{eq:Phi_E}\\
\nonumber
\Phi_\beta^{(S)}(n)&=\ln\int d\widetilde{\bm{x}}\phi^{(0)}(x^{(0)})\phi^\beta(\widetilde{\bm{x}}^{(a)})\\
&\hspace{2.0cm}\times\exp\left(\sum_{a\leq b}\widehat{q}^{(ab)}x^{(a)}x^{(b)}
+\sum_{a=1}^n\widehat{m}^{(a)}x^{(0)}x^{(a)}\right).
\end{eqnarray}

\subsection{Replica symmetric assumption}

For further calculations, we introduce 
replica symmetric (RS) assumption as
\begin{eqnarray}
(q^{(ab)},\widehat{q}^{(ab)})&=\left\{ \begin{array}{ll}
\left(Q,-\displaystyle\frac{\widehat{Q}}{2}\right)  &  {\rm if}~a= b\\
(q,\widehat{q})  &  {\rm if}~a\neq b\\
\end{array}, \right.
\label{eq:RS_q}\\
(m^{(a)},\widehat{m}^{(a)})&=(m,\widehat{m})~{\rm for~any}~a.
\label{eq:RS_m}
\end{eqnarray}
RS assumption restricts the extremum of  
(\ref{eq:free_entropy_n}) to be in the space where
(\ref{eq:RS_q}) and (\ref{eq:RS_m}) hold.
Therefore, 
it is always necessary to examine the validity of the solution under the RS
assumptions.
The RS formula at zero-temperature
derived here has been proven to be exact for convex likelihoods in \cite{Loureiro2021},
and  that for Bayes-optimal inference was proven to be rigorous in \cite{Barbier2019}.
The RS assumption is known as to be equivalent with the assumption introduced
in the cavity approach \cite{beyond}.

For the integral of $\theta$ in (\ref{eq:Phi_E}) under the RS assumption, it is convenient to
introduce the transformation 
using independent Gaussian random variables $w^{(a)}~(a\in\bm{n})$,
$z$, and $v$ as
\begin{eqnarray}
\theta^{(a)}=\sqrt{Q-q}w^{(a)}+\sqrt{q}z,~(a\neq 0)\\
\theta^{(0)}=\sqrt{\rho\sigma_x^2-\frac{m^2}{q}}v+\frac{m}{\sqrt{q}}z,
\end{eqnarray}
where this transformation preserves the relationships
(\ref{eq:theta_corr_q})--(\ref{eq:theta_corr_0}).
We denote the functions $\Phi_\beta^{(C)}(n)$, $\Phi_\beta^{(E)}(n)$,
and $\Phi_\beta^{(S)}(n)$
under the RS assumption as $\Phi_\beta^{(C,\mathrm{RS})}(n)$, $\Phi_\beta^{(E,\mathrm{RS})}(n)$,
and $\Phi_\beta^{(S,\mathrm{RS})}(n)$, respectively.
Using them, we obtain
\begin{eqnarray}
\nonumber
\Psi_\beta^{(C,\mathrm{RS})}&\equiv-\lim_{n\to 0}\frac{\partial}{\partial n}\Phi_\beta^{(C,\mathrm{RS})}(n)\\
&=\widehat{m}m-\frac{\widehat{Q}Q+\widehat{q}q}{2}\\
\nonumber
\Psi_\beta^{(E,\mathrm{RS})}&\equiv-\lim_{n\to 0}\frac{\partial}{\partial n}\Phi_\beta^{(E,\mathrm{RS})}(n)\\
&={-\int d\varepsilon P_{\varepsilon}(\varepsilon)\int Dz\int Dv\ln
\widehat{\xi}^{(\mathrm{RS})}_\beta(v,z,\varepsilon)}\\
\nonumber
\Psi_\beta^{(S,\mathrm{RS})}&\equiv-\lim_{n\to 0}\frac{\partial}{\partial n}\Phi_\beta^{(S,\mathrm{RS})}(n)\\
&=-\int dx^{(0)}\phi^{(0)}(x^{(0)})\int Dz\ln\xi^{(\mathrm{RS})}_\beta(z,x^{(0)}),
\end{eqnarray}
and the free energy under the RS assumption is given by
\begin{eqnarray}
\Psi^{(\mathrm{RS})}=\lim_{\beta\to\infty}\frac{1}{\beta}\left\{\Psi_\beta^{(C,\mathrm{RS})}+\alpha\Psi_\beta^{(E,\mathrm{RS})}+\Psi_\beta^{(S,\mathrm{RS})}\right\},
\end{eqnarray}
where we set 
\begin{eqnarray}
\nonumber
&\widehat{\xi}^{(\mathrm{RS})}_\beta(v,z)={\int Dw}\\
&\hspace{1.0cm}{\times f^\beta\left({\cal G}\left(\sqrt{\sigma_{0}^2-\frac{m^2}{q}}v+\frac{m}{\sqrt{q}}z+\varepsilon\right)\Big|\sqrt{Q-q}w+\sqrt{q}z\right)}\\
&\xi^{(\mathrm{RS})}_\beta(z,x^{(0)})=\int dx\phi^\beta(x)
\exp\left(-\frac{\widehat{Q}+\hat{q}}{2}x^2+(\sqrt{\hat{q}}z+\hat{m}x^{(0)})x\right),
\end{eqnarray}
{and $\sigma_0^2=\rho\sigma_x^2$.}

\subsubsection{$\beta\to\infty$}

At $\beta\to\infty$, 
we implement some integrals by introducing the saddle point method.
We scale the variables as $\chi=\beta(Q-q)\sim O(1)$
and $\widehat{Q}+\widehat{q}=\beta\widehat{\Theta}$, 
$\widehat{q}=\beta^2\widehat{\chi}$,
and $\widehat{m}=\beta\widehat{\mu}$.
Then, we obtain
\begin{eqnarray}
\lim_{\beta\to\infty}\frac{1}{\beta}\Psi_\beta^{(C,\mathrm{RS})}=\widehat{\mu}\mu-\!\frac{\widehat{\Theta}Q-\widehat{\chi}\chi}{2}.
\end{eqnarray}
Next, we introduce variable $u=\sqrt{Q-q}w+\sqrt{q}z$ and replace $w$ with it as
$w=(u-\sqrt{q}z)\slash\sqrt{Q-q}$.
Further, by applying the transformation from $(v,z)$ to $(\nu,\zeta)$,
which keeps $dvdz=d\nu d\zeta$, as
\begin{eqnarray}
\sqrt{\sigma_0^2-\frac{m^2}{Q}}v+\frac{m}{\sqrt{Q}}z&=\sigma_0\nu\\
-\frac{m}{\sqrt{Q}}v+\sqrt{\sigma_0^2-\frac{m^2}{Q}}z&=
\sigma_0\zeta,
\end{eqnarray}
hence $\sqrt{Q}z=(m\nu+\sqrt{Q\sigma_0^2-m^2}\zeta)\slash\sigma_0$.
With these variables, we obtain
\begin{eqnarray}
\lim_{\beta\to\infty}\frac{1}{\beta}\ln\widehat{\xi}^{(\mathrm{RS})}_\beta&=\widehat{f}_\xi^{(\mathrm{RS})*}(\nu,\zeta,\varepsilon),\label{eq:xi_hat_beta}
\end{eqnarray}
where $\widehat{f}_\xi^{(\mathrm{RS})*}(\nu,\zeta,\varepsilon)=\widehat{f}_\xi^{(\mathrm{RS})}(u_{\mathrm{RS}}^*(\nu,\zeta,\varepsilon);\nu,\zeta,\varepsilon)$ with
\begin{eqnarray}
\nonumber
&\widehat{f}_\xi^{(\mathrm{RS})}(u;\nu,\zeta,\varepsilon)\\
&=-\frac{(u-(m\nu+\sqrt{Q\sigma_0^2-m^2}\zeta)(\sigma_0)^{-1})^2}{2\chi}
+\ln f\left({\cal G}\left(\sigma_0\nu+\varepsilon\right)|u\right)\label{eq:f_xi_hat_beta}\\
&u_{\mathrm{RS}}^*(\nu,\zeta,\varepsilon)=\mathop{\mathrm{argmax}}_u\widehat{f}_\xi(u;\nu,\zeta,\varepsilon)
\label{eq:replica_u_saddle}
\end{eqnarray}
In the case of GLM, $u^*(\nu,\zeta)$ should satisfy
\begin{eqnarray}
\frac{u_{\mathrm{RS}}^*-(m\nu+\sqrt{Q\sigma_0^2-m^2}\zeta)\sigma_0^{-1}}{\chi}={\cal G}\left(\sigma_0\nu+\varepsilon\right)-a^\prime(u^*_{\mathrm{RS}}).
\end{eqnarray}

For $\Psi_\beta^{(S,\mathrm{RS})}$,
under the scaling of the parameters at $\beta\to\infty$, 
we obtain
\begin{eqnarray}
\lim_{\beta\to\infty}\frac{1}{\beta}\ln\xi^{(\mathrm{RS})}_\beta(z,x^{(0)})&=f^{(\mathrm{RS})*}_\xi(z,x^{(0)})
\end{eqnarray}
where $f^{(\mathrm{RS})*}_\xi(z,x^{(0)})=f^{(\mathrm{RS})}_\xi(x_{\mathrm{RS}}^*(z,x^{(0)});z,x^{(0)})$
with
\begin{eqnarray}
f^{(\mathrm{RS})}_\xi(x;z,x^{(0)})&=-\frac{\widehat{\Theta}}{2}x^2+(\sqrt{\widehat{\chi}}z+\widehat{\mu}x^{(0)})x+\ln\phi(x)\\
x_{\mathrm{RS}}^*(z,x^{(0)})&=\mathop{\mathrm{argmax}}_xf^{(\mathrm{RS})}_\xi(x;z,x^{(0)}).
\label{eq:replica_one_body_x}
\end{eqnarray}
In summary, the free energy density under RS assumption is given by
\begin{eqnarray}
\nonumber
\Psi^{(\mathrm{RS})}&=\mathop{\mathrm{extr}}_{\Omega,\widehat{\Omega}}\Big[\widehat{\mu}m-\!\frac{\widehat{\Theta}Q-\widehat{\chi}\chi}{2}-\int dx^{(0)}\phi^{(0)}(x^{(0)})\int Dz f_\xi^{(\mathrm{RS})*}(z,x^{(0)})\\
&\hspace{1.0cm}-\alpha{\int D\zeta  D\nu d\varepsilon P_\varepsilon(\varepsilon)\widehat{f}_\xi^{(\mathrm{RS})*}(\nu,\zeta,\varepsilon)}\Big]
\end{eqnarray}
where $\Omega=\{\mu,Q,\chi\}$, 
$\widehat{\Omega}=\{\widehat{\mu},\widehat{\Theta},\widehat{\chi}\}$,
and $\mathrm{extr}_{\Omega,\widehat{\Omega}}$
denotes the extremization with respect to $\Omega$ and $\widehat{\Omega}$.

Hereafter, we consider the Gaussian noise with mean 0 and variance $\sigma^2$.
In this case, we can implement the integral with respect to $\varepsilon$,
and the resulting form of (\ref{eq:f_xi_hat_beta}) is given as 
\begin{eqnarray}
\nonumber
&\widehat{f}_\xi^{(\mathrm{RS})}(u;\nu,\zeta)\\
&=-\frac{(u-(m\nu+\sqrt{Q\sigma_T^2-m^2}\zeta)(\sigma_T)^{-1})^2}{2\chi}
+\ln f\left({\cal G}\left(\sigma_T\nu\right)|u\right)
\end{eqnarray}
where $\sigma_T^2=\sigma_0^2+\sigma^2$.
The corresponding maximizer is given as
$\widehat{f}_\xi^{(\mathrm{RS})*}(\nu,\zeta)=\widehat{f}_\xi^{(\mathrm{RS})}(
u_{\mathrm{RS}}^*(\nu,\zeta),\nu,\zeta)$
with $u_{\mathrm{RS}}^*=\mathrm{argmax}_u\widehat{f}_\xi^{(\mathrm{RS})}(u;\nu,\zeta)$.

\subsubsection{Saddle point equations and simplification of the free energy density}

From the extremization conditions of the free energy density,
we obtain the saddle point equations for macroscopic quantities as
\begin{eqnarray}
Q&=\int dx^{(0)}\phi^{(0)}(x^{(0)})\int Dz(x_{\mathrm{RS}}^*(z,x^{(0)}))^2\\
\chi&=\frac{1}{\sqrt{\widehat{\chi}}}\int dx^{(0)}\phi^{(0)}(x^{(0)})\int Dz\frac{\partial x_{\mathrm{RS}}^*(z,x^{(0)})}{\partial z}\\
m&=\int dx^{(0)}\phi^{(0)}(x^{(0)})\int Dzx_{\mathrm{RS}}^*(z,x^{(0)})x^{(0)},
\end{eqnarray}
and for conjugate variables
\begin{eqnarray}
\widehat{\Theta}&=\alpha\int D\zeta D\nu~\frac{a^{\prime\prime}(u_{\mathrm{RS}}^*(\nu,\zeta))}{1+\chi a^{\prime\prime}(u_{\mathrm{RS}}^*(\nu,\zeta))}\label{eq:Q_hat}\\
\widehat{\chi}&=\alpha\int D\zeta D\nu~({\cal G}(\sigma_T\nu,\varepsilon)-a^\prime(u_{\mathrm{RS}}^*(\nu,\zeta)))^2\label{eq:chi_hat}\\
\widehat{\mu}&=\alpha\int D\zeta D\nu 
\frac{{\cal G}^\prime(\sigma_T\nu)}{1+\chi a^{\prime\prime}(u_{\mathrm{RS}}^*(\nu,\zeta))}.\label{eq:m_hat}
\end{eqnarray}
The derivations of (\ref{eq:Q_hat})-(\ref{eq:m_hat})
are shown in \ref{sec:app_saddle}.

Using the saddle point equations,
we can simplify $\Psi_\beta^{(S,\mathrm{RS})}$ and $\Psi_\beta^{(E,\mathrm{RS})}$ as 
\begin{eqnarray}
\frac{1}{\beta}\Psi_\beta^{(S,\mathrm{RS})}&=\frac{\widehat{\Theta}Q}{2}-\widehat{\chi}\chi-\widehat{\mu}m
-\int dx^{(0)}\phi^{(0)}(x^{(0)})\int Dz\ln\phi(x_{\mathrm{RS}}^*(z,x^{(0)}))\\
\frac{1}{\beta}\Psi_\beta^{(E,\mathrm{RS})}&=\frac{\chi\widehat{\chi}}{2}-\alpha\int D\zeta D\nu \ln f({\cal G}(\sigma_T\nu)|u_{\mathrm{RS}}^*(\nu,\zeta)),
\end{eqnarray}
and the free energy density as
\begin{eqnarray}
\nonumber
\Psi^{\mathrm{RS}}&=-\alpha\int D\zeta D\nu \ln f({\cal G}(\sigma_T\nu)|u_{\mathrm{RS}}^*(\nu,\zeta))\\
&\hspace{2.0cm}-\int dx^{(0)}\phi^{(0)}(x^{(0)})\int Dz\ln\phi(x_{\mathrm{RS}}^*(z,x^{(0)})).
\label{eq:F_simplify}
\end{eqnarray}
The first and the second terms of (\ref{eq:F_simplify})
correspond to the expectation of the training error and 
regularization term, respectively,
hence we can obtain the training error from the free energy density 
(\ref{eq:F_simplify}), by discarding the contribution from the penalty term.

\subsubsection{Relationship with GAMP}

The calculation by the replica analysis has 
correspondence with GAMP.
The equations for determining $u^*_{\mathrm{RS}}$ and $x^*_{\mathrm{RS}}$ in replica method,
given by (\ref{eq:replica_u_saddle}) and (\ref{eq:replica_one_body_x}), respectively,
have the equivalent form
with (\ref{eq:theta_max}) and (\ref{eq:GAMP_x_max}), which appear in GAMP.
In table \ref{table:GAMP_replica_correspondence}, we summarize the 
correspondence of the quantities appear in the
replica method and GAMP.
The difference between GAMP and the replica method is that 
GAMP is defined for one realization of data,
and meanwhile replica method considers the typical property with respect to the
data distribution.
For example, as shown in table \ref{table:GAMP_replica_correspondence},
the estimate $\widehat{x}_i$ in GAMP is defined on given data ${\cal D}$,
and has site-dependency.
In the replica method, the estimate $x^*_{\mathrm{RS}}$ is defined on random variables
$z$ and $x^{(0)}$.
These random variables effectively represent the randomness
induced by $\bm{\varepsilon}$, $\bm{F}$, and $\bm{x}^{(0)}$, and the estimates in GAMP 
and the replica method are expected to be statistically equivalent.

\begin{table}[htbp]
\centering
\begin{tabular}{c|c}
\hline
GAMP & Replica (RS) \\
\hline \hline
$\widehat{x}_i$  &  $x_{\mathrm{RS}}^*(z,x^{(0)})$ \\
\hline
$s$ & $\chi$ \\
\hline
$\theta^*_\mu$ & $u_{\mathrm{RS}}^*(\nu,\zeta)$ \\
\hline
${s_\theta}$ & $\chi$ \\
\hline
$\widehat{\theta}_\mu^{\setminus\mu}$ & $(m\nu-\sqrt{Q\sigma_T^2-m^2}\zeta)\sigma_T^{-1}$ \\
\hline
$\Sigma$ & $\widehat{\Theta}^{-1}$ \\
 \hline
$\mathrm{m}_i\Sigma^{-1}$ & $\sqrt{\widehat{\chi}}z+\widehat{\mu}x^{(0)}$\\
  \hline
$\frac{1}{M}\sum_{\mu=1}^M(\check{g}_{\mathrm{out}})_\mu^2$ & $\widehat{\chi}$\\
  \hline
$-\frac{1}{M}\sum_{\mu=1}^M(\check{\partial}_{\hat{\theta}}g_{\mathrm{out}})_\mu$ & $\widehat{\Theta}$\\
\hline
\end{tabular}
\caption{Correspondence between GAMP and replica method under RS assumption.}
\label{table:GAMP_replica_correspondence}
\end{table}

\subsection{Analytical form of extra-sample prediction error}

The training error is naturally derived from the free energy density 
as shown in (\ref{eq:F_simplify}),
but we need additional computation for deriving the extra-sample prediction error.
Following the definition of the extra-sample prediction error,
we consider new predictor $\mathbf{f}_{\mathrm{new}}$ and define
\begin{eqnarray}
\theta^0_{\mathrm{new}}&=\frac{1}{\sqrt{N}}\mathbf{f}_{\mathrm{new}}\bm{x}^{(0)}
\label{eq:theta_new_0}\\
\theta_{\mathrm{new}}&=\frac{1}{\sqrt{N}}\mathbf{f}_{\mathrm{new}}\widehat{\bm{x}}({\cal D}),
\label{eq:theta_new}
\end{eqnarray}
and set the test data as $y_{\mathrm{new}}={\cal G}(\theta_{\mathrm{new}}^0+\varepsilon)$.
The expectation value of the 
extra-sample prediction error is defined by
\begin{eqnarray}
\mathbb{E}_{{\cal D},\mathbf{f}_{\mathrm{new}},\varepsilon}[\mathrm{err}_{\mathrm{pre}}({\cal D})]=-\mathbb{E}_{{\cal D},\mathbf{f}_{\mathrm{new}},\varepsilon}[\ln f(y_{\mathrm{new}}|\theta_{\mathrm{new}}(\mathbf{f}_{\mathrm{new}},\widehat{\bm{x}}({\cal D})))].
\end{eqnarray}
Here, we consider that the each component of 
the new predictor $\mathbf{f}_{\mathrm{new}}$
is distributed according to the Gaussian distribution with mean 0 and variance 1.
For the expectation in the framework of replica method 
under RS assumption, we introduce Gaussian random variables $v$ and $w$
with mean zero and variance 1
and set 
\begin{eqnarray}
\theta^0_{\mathrm{new}}&=\sqrt{\sigma_x^2-\frac{m^2}{Q}}v+\frac{m}{\sqrt{Q}}w\\
\theta_{\mathrm{new}}&=\sqrt{Q}w,
\end{eqnarray}
to satisfy the conditions
\begin{eqnarray}
\mathbb{E}_{\mathbf{f}_{\mathrm{new}}}[\theta^0_{\mathrm{new}}\theta_{\mathrm{new}}]&=m\\
\mathbb{E}_{\mathbf{f}_{\mathrm{new}}}[\theta^0_{\mathrm{new}}\theta_{\mathrm{new}}^{0}]&=\rho\sigma_x^2\\
\mathbb{E}_{\mathbf{f}_{\mathrm{new}}}[\theta_{\mathrm{new}}\theta_{\mathrm{new}}]&=Q.
\end{eqnarray}
Using (\ref{eq:theta_new_0}) and (\ref{eq:theta_new}), 
the extra-sample prediction error for GLM is given by
\begin{eqnarray}
\nonumber
\mathbb{E}_{{\cal D}}[\mathrm{err}_{\mathrm{pre}}^{(\mathrm{in})}({\cal D})]&=-\int DvDw \Big\{{\cal G}\left(\sqrt{\sigma_T^2-\frac{m^2}{Q}}v+\frac{m}{\sqrt{Q}}w\right)\sqrt{Q}w\\
&\hspace{1.0cm}-a(\sqrt{Q}w)+b\left({\cal G}\left(\sqrt{\sigma_T^2-\frac{m^2}{Q}}v+\frac{m}{\sqrt{Q}}w\right)\right)\Big\}.
\end{eqnarray}
For the linear regression, we obtain
\begin{eqnarray}
\mathbb{E}_{{\cal D}}\left[\mathrm{err}_{\mathrm{pre}}^{(\mathrm{ex})}({\cal D})\right]&=
\frac{Q-2m+\sigma_T^2}{2},
\end{eqnarray}
and for the logistic regression, we obtain
\begin{eqnarray}
\nonumber
\mathbb{E}_{{\cal D}}\left[\mathrm{err}_{\mathrm{pre}}^{(\mathrm{ex})}({\cal D})\right]&=-\int Dw\left\{ \frac{\sqrt{Q}}{2}\frac{\partial}{\partial w}\mathrm{erfc}\left(-\frac{\frac{m}{\sqrt{Q}}w}{\sqrt{2(\sigma_T^2-\frac{m^2}{Q})}}\right)-a(\sqrt{Q}w)\right\}\\
&=-\frac{m}{\sqrt{2\pi\sigma_T^2}}+\int Dw a(\sqrt{Q}w).
\end{eqnarray}
The behaviors of the extra-sample 
prediction error and corresponding generalization gap by replica method are shown in 
Figs. \ref{fig:Ridge_pre}-\ref{fig:logi_Gauss}.
The results by replica method coincides with that by GAMP.

\subsection{Analytical forms of estimators for the generalization gap}

\subsubsection{GDF by replica method under RS assumption}
\label{sec:GDF_replica}

The data $y$ is replaced with $\sigma_T\nu$ in the replica analysis, 
hence GDF is defined by the derivative of $u_{\mathrm{RS}}^*(\nu)$ with respect to $\sigma_T\nu$.
From the solution of $u_{\mathrm{RS}}^*(\nu)$ (\ref{eq:replica_u_saddle}),
we obtain
\begin{eqnarray}
\frac{\partial u_{\mathrm{RS}}^*(\nu,\zeta)}{\partial(\sigma_T\nu)}=\frac{\chi}{1+\chi a^{\prime\prime}(u_{\mathrm{RS}}^*(\nu,\zeta))},
\end{eqnarray}
and hence the expectation of GDF by replica method is derived as
\begin{eqnarray}
\mathbb{E}_{{\cal D}}[\mathrm{gdf}]=\int D\nu D\zeta \frac{\chi}{1+\chi a^{\prime\prime}(u_{\mathrm{RS}}^*(\nu,\zeta))}.
\end{eqnarray}
This form is consistent with GDF derived by GAMP (\ref{eq:GDF_GAMP}):
${s_\theta}$ in GAMP corresponds to $\chi$ in the replica analysis,
and the summation with respect to $\theta_\mu^*$ is replaced with the 
integral with respect to $u_{\mathrm{RS}}^*(\nu,\zeta)$,
as shown in Table \ref{table:GAMP_replica_correspondence}.
The behaviors of GDF and corresponding in-sample prediction errors are 
shown in Figs.\ref{fig:Ridge_pre}-\ref{fig:logi_Gauss}.
The GDF by replica method coincident with the 
expectation value of GDF calculated by GAMP.

\subsubsection{FV by replica method under RS assumption}

For deriving FV in the framework of the replica method,
we introduce the auxiliary variable $\gamma$,
following the similar discussion with GAMP, as
\begin{eqnarray}
\nonumber
{\cal T}_\gamma^{(\mathrm{RS})}(\gamma)&=\gamma\ln f({\cal G}(\sigma_T\nu)|u^*(\gamma))\\
&\hspace{1.0cm}-\frac{(u_{\gamma}^*(\gamma)-(m\nu+\!\!\sqrt{Q\sigma_x^2-m^2}\zeta)\sigma_T^{-1})^2}{2\chi}\\
\nonumber
u_\gamma^*(\gamma,\nu,\zeta)&=\mathop{\mathrm{argmax}}_{u}\Big\{\gamma\ln f({\cal G}(\sigma_T\nu)|u)\\
&\hspace{1.0cm}-\frac{(u-(m\nu+\!\!\sqrt{Q\sigma_x^2-m^2}\zeta)\sigma_T^{-1})^2}{2\chi}\Big\}.
\end{eqnarray}
Using these functions, the expectation of FV by the replica method is given by
\begin{eqnarray}
\nonumber
\mathbb{E}_{{\cal D}}[\mathrm{FV}]&=
\int D\nu D\zeta\frac{\partial^2}{\partial\gamma^2}{\cal T}_\gamma^{(\mathrm{RS)}}(\gamma)\Big|_{\gamma=1}\\
&=\int D\nu D\zeta\frac{({\cal G}(\sigma_T\nu)-u_{\mathrm{RS}}^*(\nu,\zeta))^2\chi}{1+\chi a^{\prime\prime}(u_{\mathrm{RS}}^*(\nu,\zeta))}.
\label{eq:FV_replica}
\end{eqnarray}
Eq.(\ref{eq:FV_replica}) can be interpreted as the expectation value of FV under GAMP
from the correspondence shown in Tabel \ref{table:GAMP_replica_correspondence}.

\subsection{Instability of RS solution}
\label{sec:Replica_AT}

The RS solutions discussed so far lose 
local stability when subjected to perturbations 
that break the symmetry between replicas in a certain parameter region. 
This phenomenon is known as de Almeida-Thouless (AT) instability \cite{AT} 
and occurs when the following conditions hold;
\begin{eqnarray}
\nonumber
&\alpha\int Dzdx^{(0)}\phi^{(0)}(x^{(0)})
\left(\frac{\partial x_{\mathrm{RS}}^*(z,x^{(0)})}{\partial\sqrt{\widehat{\chi}}z}\right)^2\\
&\times\int D\nu D\zeta \left(
\frac{\partial}{\partial \sqrt{Q-\frac{m^2}{\sigma_T^2}}\zeta}
\left({\cal G}\left(\sigma_T\nu\right)-a^\prime(u_{\mathrm{RS}}^*(\nu,\zeta))\right)\right)^2>1
.
\label{eq:AT_replica}
\end{eqnarray}
The derivation of the AT instability condition is explained in \ref{sec:app_1RSB},
by introducing the solution under the 1 step replica symmetry breaking assumption.
Based on the correspondence between AMP and the RS saddle point 
discussed so far and summarized in Table \ref{table:GAMP_replica_correspondence},
the AT instability condition (\ref{eq:AT_replica})
is coincident with the local instability condition of
the GAMP fixed point.
Similar correspondence has been shown in the analysis for CDMA \cite{Kabashima}. 

\subsection{Specific forms of the saddle point equations}

\subsubsection{Gaussian likelihood}

The saddle point equations for the conjugate variables can be simplified
in case of the Gaussian likelihood.
The maximizer $u^*_{\mathrm{RS}}$ of (\ref{eq:replica_u_saddle}) is given by 
\begin{eqnarray}
u_{\mathrm{RS}}^*(\nu,\zeta)=\frac{\chi\sigma_T\nu+(m\nu+\!\!\sqrt{Q\sigma_T^2-m^2}\zeta)\sigma_T^{-1}}{1+\chi}
\end{eqnarray}
hence we obtain
\begin{eqnarray}
\nonumber
\int D\zeta D\nu\ln\widehat{\xi}_{\beta}^{(\mathrm{RS})}&=
\int D\zeta D\nu\left(-\frac{(\sigma_T\nu-(m\nu+\!\!\sqrt{Q\sigma_T^2-m^2}\zeta)\sigma_T^{-1})^2}{2(1+\chi)}\right)\\
&=-\frac{Q-2m+\sigma_T^2}{2(1+\chi)}.
\end{eqnarray}
Using (\ref{eq:xi_hat_Gauss}), we obtain the saddle point equations
\begin{eqnarray}
\widehat{\Theta}&=\frac{\alpha}{1+\chi}\label{eq:theta_hat_Gauss}\\
\widehat{\chi}&=\frac{\alpha(Q-2m+\sigma_T^2)}{(1+\chi)^2}\label{eq:xi_hat_Gauss}
\\
\widehat{m}&=\frac{\alpha}{1+\chi}.
\end{eqnarray}
Further, the AT instability condition is given by
\begin{eqnarray}
\frac{\alpha}{(1+\chi)^2}\int Dzdx^{(0)}\phi^{(0)}(x^{(0)})
\left(\frac{\partial x_{\mathrm{RS}}^*(z,x^{(0)})}{\partial\sqrt{\widehat{\chi}}z}\right)^2>1
.
\label{eq:AT_replica_Gauss}
\end{eqnarray}



\subsubsection{Elastic net penalty}

The estimate $x_{\mathrm{RS}}^*$ (\ref{eq:replica_one_body_x}) for the elastic net penalty
is given by
\begin{eqnarray}
x_{\mathrm{RS}}^*=\frac{\sqrt{\widehat{\chi}}z+\widehat{\mu}x^{(0)}-\mathrm{sgn}(\sqrt{\widehat{\chi}}z+\widehat{\mu}x^{(0)})\lambda_1}{\widehat{\Theta}+\lambda_2}
\mathbb{I}\left(|\sqrt{\widehat{\chi}}z+\widehat{\mu}x^{(0)}|>\lambda_1\right),
\end{eqnarray}
hence the saddle point equations under the 
elastic net penalty are given by
\begin{eqnarray}
\nonumber
Q&=\frac{1}{(\widehat{\Theta}+\lambda_2)^2}\int d\sigma_\rho P_\rho(\sigma_\rho)
\\
&\hspace{1.0cm}\times\left\{\sigma_\rho^2\left(1+2\Theta(\sigma_\rho)^2\right)\mathrm{erfc}\left(\Theta(\sigma_\rho)\right)-\frac{2\sigma_\rho}{\sqrt{\pi}}\Theta(\sigma_\rho)\right\}
\label{eq:EN_saddle_Q}\\
m&=\frac{\rho\sigma_x\widehat{\mu}}{\widehat{\Theta}+\lambda_2}
\mathrm{erfc}\left(\frac{\lambda_1}{\sqrt{2(\widehat{\chi}+\widehat{\mu}^2\sigma_x^2)}}\right)
\label{eq:EN_saddle_m}\\
\chi&=\frac{\widehat{\rho}}{\widehat{\Theta}+\lambda_2},
\label{eq:EN_saddle_chi}
\end{eqnarray}
where
\begin{eqnarray}
\Theta(\sigma_\rho)&=\frac{\lambda_1}{\sqrt{2}\sigma_\rho},\\
\widehat{\rho}&=\int d\sigma_\rho P_\rho(\sigma_\rho)
\mathrm{erfc}\left(\Theta(\sigma_\rho)\right)\\
P_\rho(\sigma_\rho)&=(1-\rho)\delta(\sigma_\rho-\sqrt{\widehat{\chi}})+\rho
\delta(\sigma_\rho-\sqrt{\widehat{\chi}+\widehat{\mu}^2\sigma_x^2}),
\end{eqnarray}
and $\widehat{\rho}$ corresponds 
$\mathbb{E}_{{\cal D}}[||\widehat{\bm{x}}({\cal D})||_0]\slash N$,
namely expected value of the fraction of non-zero component
in the estimate.
For the Gaussian likelihood, the
AT instability condition for the elastic net penalty is given by
\begin{eqnarray}
\frac{\alpha}{\widehat{\rho}}\left(\frac{\chi}{1+\chi}\right)^2>1.
\label{eq:AT_Gauss_en}
\end{eqnarray}
As discussed in sec.\ref{sec:GDF_replica}, 
the term $\chi\slash(1+\chi)$ corresponds to the expectation value of GDF.

\subsubsection{$\ell_2$ penalty}

For the $\ell_2$ penalty at $\rho=1$,
the saddle point equation is given by
\begin{eqnarray}
Q&=\frac{\widehat{\chi}+\widehat{\mu}^2\sigma_x^2}{(\widehat{\Theta}+\lambda)^2}\\
m&=\frac{\widehat{\mu}\sigma_x^2}{\widehat{\Theta}+\lambda}\\
\chi&=\frac{1}{\widehat{\Theta}+\lambda}\label{eq:Ridge_saddle_chi},
\end{eqnarray}
where we set $\lambda_1=0$, $\widehat{\rho}=1$ and $\lambda_2=\lambda$ of 
the saddle point equations for the elastic net penalty (\ref{eq:EN_saddle_Q})-(\ref{eq:EN_saddle_chi}).
For the Gaussian likelihood,
the saddle point equation of $\chi$ can be solved by using 
(\ref{eq:Ridge_saddle_chi}) and (\ref{eq:xi_hat_Gauss}).
In particular for $\lambda=0$ case,
we can solve the variable $\chi$
as 
\begin{eqnarray}
\chi=\frac{1}{\alpha-1}.
\label{eq:chi_ridgeless}
\end{eqnarray}
This expression indicates that $\chi\to\infty$ as $\alpha$ approaches 1 from above.
Using (\ref{eq:chi_ridgeless}), we obtain
\begin{eqnarray}
\mathbb{E}_{{\cal D}}[\mathrm{gdf}]=\frac{\chi}{1+\chi}=\frac{1}{\alpha},
\label{eq:GDF_ridgeless}
\end{eqnarray}
which corresponds to the GF or GDF using the hat matrix.

For the Gaussian likelihood with $\ell_2$ penalty of 
$\lambda=0$ case, we term ridgeless limit, the following theorem 
with respect to the AT instability and the series expansion, which is the basis for the
derivation of WAIC as (\ref{eq:T_expansion}), holds.
\begin{thm}
In the ridgeless limit, the instability of the RS solution 
and the divergence of the series expansion of ${\cal T}(\eta)$ at $\eta=1$
arise at $\alpha<1$.
\end{thm}
\begin{proof}
From (\ref{eq:AT_Gauss_en}) and (\ref{eq:GDF_ridgeless}),
the AT instability condition for the ridgeless limit is given by
\begin{eqnarray}
\frac{1}{\alpha}>1.
\label{eq:AT_ridgeless}
\end{eqnarray}
Therefore, 
AT instability always appears in $\alpha<1$
for $\lambda=0$.

Putting the expression (\ref{eq:GDF_ridgeless})
to the convergence condition of the series expansion for the derivation of WAIC
given by (\ref{eq:Gauss_conv_rand})
with the correspondence between $\chi$ and $s_\theta$,
we obtain
\begin{eqnarray}
\frac{\eta}{\alpha}>1,
\end{eqnarray}
hence the expansion at $\eta=1$
is inappropriate when ${\alpha}^{-1}>1$.
Therefore, in the ridgeless limit,
the AT instability corresponds to the 
inadequacy of the series expansion for deriving WAIC.
\end{proof}

\subsubsection{$\ell_1$ penalty}

In case of $\ell_1$ penalty,
which corresponds to $\lambda_2=0$ of the elastic net penalty,
we obtain 
\begin{eqnarray}
\chi=\frac{\widehat{\rho}}{\widehat{\Theta}}.
\end{eqnarray}
For the Gaussian likelihood with $\ell_1$ penalty termed LASSO,
we obtain the expression using (\ref{eq:theta_hat_Gauss}),
\begin{eqnarray}
\chi=\frac{\widehat{\rho}}{\alpha-\widehat{\rho}},
\end{eqnarray}
and hence 
\begin{eqnarray}
\mathbb{E}_{{\cal D}}[\mathrm{gdf}]=\frac{\chi}{1+\chi}=\frac{\widehat{\rho}}{\alpha}.
\label{eq:GDF_L1_RS}
\end{eqnarray}
Eq. (\ref{eq:GDF_L1_RS}) indicates that the following theorem hold.
\begin{thm}
For LASSO,
the number of the parameters divided by $M$, 
namely $||\widehat{\bm{x}}({\cal D})||_0\slash M$,
is an unbiased estimator of the covariance penalty (\ref{eq:def_cov}).
Hence, the $C_p$ criterion is equivalent to AIC, and they are unbiased estimators of the 
in-sample prediction error of LASSO \cite{L1_GDF}.
\end{thm}

As with the ridgeless limit, we can show that following theorem for LASSO.
\begin{thm}
For LASSO,
the instability of the RS solution 
and the divergence of the series expansion of ${\cal T}(\eta)$ at $\eta=1$
arise at $\alpha<\widehat{\rho}$.
\end{thm}
\begin{proof}
From (\ref{eq:AT_Gauss_en}) and (\ref{eq:GDF_L1_RS}),
AT instability arises when 
\begin{eqnarray}
\frac{\widehat{\rho}}{\alpha}>1.
\end{eqnarray}
Substituting (\ref{eq:GDF_L1_RS}) to the radius of convergence of the series expansion 
(\ref{eq:T_expansion}), we obtain
\begin{eqnarray}
\frac{\eta\widehat{\rho}}{\alpha}>1.
\end{eqnarray}
hence the invalidity of the series expansion at $\eta=1$ 
and the instability of the RS solution 
arise at $\alpha<\widehat{\rho}$.
\end{proof}


\section{Summary and Conclusion}
\label{sec:summary}

We discussed prediction errors, generalization gaps and their estimators
in GLM. 
We derived the form of $C_p$ and WAIC using GAMP, and we
discussed the relationship between them.
We showed that GDF and FV, which are the estimators of the generalization gap,
can be expressed by the variance of the model parameter in the framework of GAMP.

We focused on the differences between the LOOCV error and WAIC or TIC at the 
finite $\alpha$ region.
The first-order expansion of the LOOCV error with respect to the 
difference between the estimates under the full and LOO samples
is insufficient for describing the LOOCV error
for the finite $\alpha$.
Thus, we derived the difference between the LOOCV error and WAIC
within the framework of GAMP by considering the higher order terms; 
we showed that the difference can be 
described by the cumulant generating function of the likelihood.
The resultant form can be calculated under full data samples,
without considering LOO samples.
Further, we extended GAMP-based estimators of the prediction errors 
to deal with correlated predictors and showed that these estimators exhibit a good agreement with 
LOOCV errors calculated by the linear algebraic form or coordinate descent algorithm.

We show that GAMP or replica method describes
the prediction accuracy of the model 
by using the variance of the parameter at equilibrium, 
which relationship can be regarded as an exemplification of
fluctuation-response relationship 
in the prediction based modeling \cite{Watanabe_FDT}.
The concept of the prediction-based modeling is 
compatible with that in statistical physics.
The derivation of such a relationship in this study
is restricted to the Gaussian predictors,
whose components independently and identically 
distributed according to the Gaussian distribution.
However, this setting is not as restrictive as it appears.
In fact, recent studies have shown Gaussian universality for 
GLMs with a large class of features,
wherein their asymptotic behaviors can be computed to leading order 
under a simpler model with Gaussian features
\cite{Loureiro2021,Goldt2022,Montanari2022,Gerace2022}.
Utilizing this correspondence,
the statistical physics-based method is expected to be 
a valuable tool for analyzing the statistical methods to find limitation
in the conventional statistical theory, analyze the $\alpha\sim O(1)$ region, 
and construct efficient algorithms.


\ack
The author thanks Yukito Iba for discussion on the information criteria
as a fluctuation-response relationship and 
teaching me regarding reference \cite{Iba-Yano2022}.
The author also thanks Yoshiyuki Kabashima 
for the helpful discussions and comments.
This work was partially supported by JST PRESTO Grant Number
JPMJPR19M2, Japan.

\appendix

\section{Expansion of the leave-one-out cross validation error}
\label{sec:app_CV}

The maximizer of $\varphi({\cal D},\bm{x})$
denoted by $\widehat{\bm{x}}({\cal D})$ satisfies
\begin{eqnarray}
\frac{\partial}{\partial x_j}\varphi({\cal D},\bm{x})\Big|_{\bm{x}=\widehat{\bm{x}}({\cal D})}=0
\label{eq:full}
\end{eqnarray}
for any $j\in\bm{N}$.
Meanwhile, the maximizer of $\sum_{\mu\in\bm{M}\setminus\nu}\ln f
(y_\mu|\theta_\mu(\mathbf{f}_{\mu},\bm{x}))-h(\bm{x})$,
denoted by $\widehat{\bm{x}}({\cal D}_{\setminus\nu})$, satisfies
\begin{eqnarray}
\frac{\partial}{\partial x_j}\left\{\sum_{\mu\in\bm{M}\setminus\nu}\ln f(y_\mu|\theta_\mu(\mathbf{f}_{\mu},\bm{x}))-h(\bm{x})\right\}\Big|_{\bm{x}=\widehat{\bm{x}}({\cal D}_{\setminus\nu})}=0
\end{eqnarray}
for any $j\in\vm{N}$. Hence, the following relationship holds:
\begin{eqnarray}
\frac{\partial}{\partial x_j}\varphi({\cal D},\bm{x})\Big|_{\bm{x}=\widehat{\bm{x}}({\cal D}_{\setminus\nu})}=
\frac{\partial}{\partial x_j}\ln f(y_\nu|\theta_\nu(\mathbf{f}_\nu,\bm{x}))\Big|_{\bm{x}=\widehat{\bm{x}}({\cal D}_{\setminus\nu})}.
\label{eq:leave_one_out}
\end{eqnarray}
Discarding (\ref{eq:full}) from (\ref{eq:leave_one_out}), we obtain
\begin{eqnarray}
\frac{\partial\varphi({\cal D},\bm{x})}{\partial x_j}\Big|_{\bm{x}=\widehat{\bm{x}}({\cal D}_{\setminus\nu})}\!\!\!\!-\frac{\partial\varphi({\cal D},\bm{x})}{\partial x_j}\Big|_{\bm{x}=\widehat{\bm{x}}({\cal D})}
\!\!\!\!\!=\frac{\partial}{\partial x_j}\ln f(y_\nu|\theta_\nu(\mathbf{f}_\nu,\bm{x}))\Big|_{\bm{x}=\widehat{\bm{x}}({\cal D}_{\setminus\nu})}.
\label{eq:app_partial_diff}
\end{eqnarray}
Under the assumption {\bf A1} that the difference 
$\widehat{\bm{x}}({\cal D})-\widehat{\bm{x}}({\cal D}_{\setminus\nu})$ is sufficiently small,
we expand 
the L.H.S. up to the first order as
\begin{eqnarray}
\nonumber
&\frac{\partial}{\partial x_j}\varphi({\cal D},\bm{x})\Big|_{\bm{x}=\widehat{\bm{x}}({\cal D}_{\setminus\nu})}-\frac{\partial}{\partial x_j}\varphi({\cal D},\bm{x})\Big|_{\bm{x}=\widehat{\bm{x}}({\cal D})}\\
&=\sum_{k=1}^N\frac{\partial^2}{\partial x_k\partial x_j}\varphi({\cal D},\bm{x})\Big|_{\bm{x}=\widehat{\bm{x}}({\cal D}_{\setminus\nu})}(\widehat{x}_k({\cal D}_{\setminus\nu})-\widehat{x}_k({\cal D})).
\end{eqnarray}
Therefore, for any $j\in\bm{N}$ and $\nu\in\bm{M}$, 
\begin{eqnarray}
-\!\sum_{k=1}^N{\cal J}^{\varphi}_{kj}(\widehat{\bm{x}}({\cal D}_{\setminus\nu}))
(\widehat{x}_k({\cal D}_{\setminus\nu})\!-\!\widehat{x}_k({\cal D}))\!=\!\frac{1}{M}\frac{\partial\ln f(y_\nu|\theta_\nu(\mathbf{f}_\nu,\bm{x}))}{\partial x_j}\Big|_{\bm{x}=\widehat{\bm{x}}({\cal D}_{\setminus\nu})},
\end{eqnarray}
where we introduce the coefficient $1\slash M$ for convenience
and define
\begin{eqnarray}
{\cal J}^{\varphi}_{kj}(\widehat{\bm{x}})=-\frac{1}{M}\frac{\partial^2}{\partial x_k\partial x_j}\varphi({\cal D},\bm{x})\Big|_{\bm{x}=\widehat{\bm{x}}}.
\end{eqnarray}
With assumption {\bf A2}, we assume that ${\cal J}_{kj}(\widehat{\bm{x}}({\cal D}_{\setminus\nu}))={\cal J}_{kj}
(\widehat{\bm{x}}({\cal D}))$
holds for any $\nu\in\bm{M}$ and $i,j\in\bm{N}$, and 
we obtain
\begin{eqnarray}
\widehat{\bm{x}}({\cal D}_{\setminus\nu})-\widehat{\bm{x}}({\cal D})=-({\cal J}^{\varphi}(\widehat{\bm{x}}({\cal D}))^{-1}
\frac{1}{M}\frac{\partial}{\partial\bm{x}}\ln f(y_\nu|\bm{\theta}(\bm{x}))\Big|_{\bm{x}=\widehat{\bm{x}}({\cal D}_{\setminus\nu})}.
\label{eq:app_diff_expand}
\end{eqnarray}

\subsection{Further expansion}
\label{sec:app_expansion_JtoK}

We continue the expansion of 
(\ref{eq:app_partial_diff}) for higher orders.
For simplicity, we consider that 
the case where the higher order derivative of the penalty term $h(x)$
is greater than the second order is zero as assumption {\bf A4}.
Under this assumption, we obtain 
\begin{eqnarray}
\nonumber
&\frac{\partial}{\partial x_j}
\varphi({\cal D},\bm{x})\Big|_{\bm{x}=\widehat{\bm{x}}({\cal D}_{\setminus\nu})}
-\frac{\partial}{\partial x_j}
\varphi({\cal D},\bm{x})\Big|_{\bm{x}=\widehat{\bm{x}}({\cal D})}\\
\nonumber
&=\sum_{k=1}^N(\widehat{x}_k({\cal D}_{\setminus\nu})-\widehat{x}_k({\cal D}))\Big\{\frac{\partial^2\varphi({\cal D},\bm{x})}{\partial x_k\partial x_j}\Big|_{\bm{x}=\widehat{\bm{x}}({\cal D}_{\setminus\nu})}
\\
\nonumber
&+\sum_{n=2}^{\infty}\frac{(-1)^{n}}{n!}\!\!\sum_{i_1=1}^N\!\!\cdots\!\!\!\!\sum_{i_{n-1}=1}^N\!\!\frac{\partial^{n+1}\varphi({\cal D},\bm{x})}{\partial x_{i_1}\!\cdots\partial x_{i_{n-1}}\partial x_k\partial x_j}\Big|_{\bm{x}=\widehat{\bm{x}}({\cal D}_{\setminus\nu})}\!\!\prod_{\ell=1}^{n-1}(\widehat{x}_{i_\ell}({\cal D}_{\setminus\nu})\!-\!\widehat{x}_{i_\ell}({\cal D}))\!\Big\}\\
\nonumber
&=\sum_{k=1}^N(\widehat{x}_k({\cal D}_{\setminus\nu})\!-\!\widehat{x}_k({\cal D}))\left\{\frac{\partial^2\varphi}{\partial x_k\partial x_j}\Big|_{\bm{x}=\widehat{\bm{x}}({\cal D}_{\setminus\nu})}\!\!\!\!\!\!\!\!\!\!\!\!-\!\frac{F_{\nu k}F_{\nu j}}{N}\!\sum_{n=2}^\infty\!\frac{(\widehat{\theta}_{\nu}\!-\!\widehat{\theta}_{\nu}^{\setminus\nu})^{n}}{n!}a^{\prime(n+1)}(\widehat{\theta}_\nu^{\setminus\nu})
\!\right\}\\
\nonumber
&=\sum_{k=1}^N(\widehat{x}_k({\cal D}_{\setminus\nu})\!-\!\widehat{x}_k({\cal D}))\left\{\!\frac{\partial^2\varphi}{\partial x_k\partial x_j}\Big|_{\bm{x}=\widehat{\bm{x}}({\cal D}_{\setminus\nu})}\hspace{-0.8cm}-\frac{F_{\nu k}F_{\nu j}}{N}
\!\left(\!\frac{a^{\prime}(\widehat{\theta}_\nu)\!-\!a^{\prime}(\widehat{\theta}_\nu^{\setminus\nu})}{\widehat{\theta}_{\nu}-\widehat{\theta}_{\nu}^{\setminus\nu}}
\!-\!a^{\prime\prime}(\widehat{\theta}_\nu^{\setminus\nu})
\!\right)\!\right\},
\label{eq:expansion_app}
\end{eqnarray}
where the arguments of $\varphi$ are abbreviated, and 
$\widehat{\theta}_\mu^{\setminus\nu}\sim\widehat{\theta}_\mu$
holds for $\mu\neq\nu$. 
Note that in (\ref{eq:expansion_app}),
we consider all possible combinations of $n$-th order terms
for any integer, that is, we are implicitly assuming that $N\to\infty$.
By defining the matrix 
\begin{eqnarray}
{\cal K}^{\setminus\nu}={\cal J}^{\varphi}(\widehat{\bm{x}}({\cal D}_{\setminus\nu}))-\frac{\mathbf{f}_\nu^{\top}\mathbf{f}_\nu}{NM}
\left(\frac{a^{\prime}(\widehat{\theta}_\nu)-a^{\prime}(\widehat{\theta}_\nu^{\setminus\nu})}{\widehat{\theta}_{\nu}-\widehat{\theta}_{\nu}^{\setminus\nu}}
-a^{\prime\prime}(\widehat{\theta}_\nu)\right)
\end{eqnarray}
for $\nu\in\bm{M}$, we obtain
\begin{eqnarray}
\bm{x}({\cal D}_{\setminus\nu})-\bm{x}({\cal D})
=-({\cal K}^{\setminus\nu})^{-1}\frac{1}{M}\frac{\partial}{\partial\bm{x}}
\ln f(y_\nu|\theta_\nu(\mathbf{f}_\nu,\bm{x}))\Big|_{\bm{x}=\widehat{\bm{x}}({\cal D}_{\setminus\nu})}.
\end{eqnarray}
For the Gaussian likelihood,
${\cal K}^{\setminus\nu}={\cal J}$ holds 
for any $\nu\in\bm{M}$.

%

\section{Generalized linear models}
\label{sec:app_GLM}

\subsection{Logistic regression with Bernoulli distribution}

The logistic regression on the 
Bernoulli distribution is defined for binary data $y_\mu\in\{0,1\}$ as
\begin{eqnarray}
\ln f(\bm{y}|\bm{\pi})=\sum_{\mu=1}^M\left\{y_\mu\log\pi_\mu+(1-y_\mu)\log(1-\pi_\mu)\right\}
\label{eq:bernoulli}
\end{eqnarray}
where the Bernoulli parameter $\pi_\mu$ provides the mean of $y_\mu$,
and the probability that $y_\mu$ is 1.
In the logistic regression, the model parameter
$\bm{\theta}=\frac{1}{\sqrt{N}}\bm{F}\bm{x}$ 
is connected to the Bernoulli parameter as  
\begin{eqnarray}
\pi_\mu=\ell^{-1}({\theta}_\mu)=\frac{1}{1+\exp(-\theta_\mu)},
\label{eq:pi_theta}
\end{eqnarray}
using the link function $\ell(\cdot)$. Therefore, 
$\ell(\pi_\mu)=\ln\left(\pi_\mu\slash(1-\pi_\mu)\right)$.
Using \eref{eq:pi_theta},
the distribution \eref{eq:bernoulli} is transformed into the form
of the exponential family distribution
\begin{eqnarray}
\ln f(\bm{y}|\bm{\theta})=\bm{\theta}^{\top}\bm{y}-\sum_{\mu=1}^M \log(1+e^{\theta_\mu}) ,
\end{eqnarray}
where functions $a$ and $b$ respectively correspond to
\begin{eqnarray}
a(\bm{\theta})=\log\prod_{\mu=1}^M(1+e^{\theta_\mu}),~~~
b(\bm{y})=0.
\end{eqnarray}
The maximization problem in GAMP (\ref{eq:theta_max})
for the logistic regression is given by 
\begin{eqnarray}
\theta_\mu^*=\mathop{\mathrm{argmax}}_{\theta_\mu}\left\{-\theta_\mu(1-y_\mu)-\log(1+e^{-\theta_\mu})-\frac{(\theta_\mu-\widehat{\theta}_\mu^{\setminus\mu})^2}{2{s_\theta}_\mu}\right\},
\end{eqnarray}
and the estimate $\theta_\mu^*$ satisfies 
\begin{eqnarray}
\frac{\theta_\mu^*-\widehat{\theta}_\mu^{\setminus\mu}}{{s_\theta}_\mu}=y_\mu-\frac{1}{1+e^{-\theta_\mu^*}}.
\label{eq:u_star_logistic}
\end{eqnarray}
The estimated mean of $y_\mu$ is given by $\pi_\mu^*\equiv(1+e^{-\theta_\mu^*})^{-1}$, and the variance is $\pi_\mu^*(1-\pi_\mu^*)$; therefore, we get
\begin{eqnarray}
(\check{g}_{\rm out})_\mu&=y_\mu-\pi_\mu^*,~~~
(\partial_{\widehat{\theta}}\check{g}_{\rm out})_\mu&=\frac{\pi_\mu^*(1-\pi_\mu^*)}{1+{s_\theta}_\mu\pi_\mu^*(1-\pi_\mu^*)}.
\end{eqnarray}

\subsection{Poisson regression}

Poisson distribution is defined for $y_\mu\in\mathbb{N}$, where $\mathbb{N}$ denotes the set of natural numbers including zero as
\begin{eqnarray}
\ln f(\bm{y}|\bm{\lambda})=\sum_{\mu=1}^M \left\{y_\mu\ln\lambda_\mu-\lambda_\mu-\log y_\mu!\right\},
\label{eq:poisson}
\end{eqnarray}
where $\lambda_\mu$ denotes the mean of $y_\mu$.
In Poisson regression, the model parameter $\bm{\theta}$ is 
connected to the mean as 
\begin{eqnarray}
\lambda_\mu = \ell^{-1}(\theta_\mu)=\exp(\theta_\mu),
\end{eqnarray}
and hence, $\ell(\lambda_\mu)=\log(\lambda_\mu)$.
This link function reduces Poisson distribution to the exponential family's form as 
\begin{eqnarray}
a(\bm{\theta})=\sum_{\mu=1}^M\exp(\theta_\mu),~~~
b(\bm{y})=-\sum_{\mu=1}^M\log y_\mu!.
\end{eqnarray}
The maximization problem in GAMP (\ref{eq:theta_max})
for the Poisson regression is given by
\begin{eqnarray}
\theta_\mu^*=\max_{\theta}\left\{y_\mu \theta_\mu-\exp(\theta_\mu)-\frac{(u_\mu-\widehat{\theta}^{\setminus\mu}_\mu)^2}{2{s_\theta}_\mu}\right\},
\end{eqnarray}
and the solution $\theta^*_\mu$ satisfies the condition
\begin{eqnarray}
\frac{\theta_\mu^*-\widehat{\theta}^{\setminus\mu}_\mu}{{s_\theta}_\mu}=y_\mu-\lambda_\mu^*,
\end{eqnarray}
where $\lambda_\mu^*\equiv\exp(u_\mu^*)$ is the estimated average.
Therefore, we obtain
\begin{eqnarray}
(\check{g}_{\mathrm{out}})_\mu&=y_\mu-\lambda_\mu^*,~~~
(\check{\partial}_{\widehat{\theta}}{g}_{\mathrm{out}})_\mu&=\frac{\lambda_\mu^*}{1+{s_\theta}_\mu\lambda_\mu^*}.
\end{eqnarray}

\subsection{Exponential regression}

Exponential distribution with mean $\lambda_\mu$
is defined for $y_\mu\in[0,\infty)$ as
\begin{eqnarray}
\ln f(\bm{y}|\bm{\eta})=\sum_{\mu=1}^M\left\{-\frac{y_\mu}{\eta_\mu}-\log(\eta_\mu)\right\}.
\end{eqnarray}
In the exponential regression, the model parameter
$\bm{\theta}$ relates to the mean as 
\begin{eqnarray}
-\eta_\mu=\ell^{-1}(\theta_\mu),
\end{eqnarray}
and hence, $\ell(\theta_\mu)=-\eta_\mu^{-1}$,
where we introduce the factor $-1$ for convenience.
Using the parameter $\bm{\theta}$,
the exponential distribution is rewritten as an exponential family 
with functions $a$ and $b$ given by
\begin{eqnarray}
a(\bm{\theta})&=-\sum_{\mu=1}^M\log(-\theta_\mu),~~~
b(\bm{y})&=0.
\end{eqnarray}
The maximization problem in GAMP (\ref{eq:theta_max}) 
for exponential regression is given by
\begin{eqnarray}
\theta_\mu^*=\max_{\theta}\left\{y_\mu \theta_\mu+\ln(-\theta_\mu)-\frac{(\theta_\mu-\widehat{\theta}^{\setminus\mu}_\mu)^2}{2{s_\theta}_\mu}\right\},
\end{eqnarray}
and the solution satisfies
\begin{eqnarray}
\frac{\theta_\mu^*-\widehat{\theta}^{\setminus\mu}_\mu}{{s_\theta}_\mu}=y_\mu+\frac{1}{\theta_\mu}.
\end{eqnarray}
The estimated mean and variance 
corresponds to $\eta_\mu^*\equiv-(\theta_\mu^*)^{-1}$
and $(\theta_\mu^*)^{-2}$, respectively; hence
we obtain
\begin{eqnarray}
(\check{g}_{\rm out})_\mu=y_\mu-\eta_\mu^*,~~~
(\partial_{\widehat{\theta}}\check{g}_{\rm out})_\mu=\frac{(\theta_\mu^*)^{-2}}{1+{s_\theta}_\mu(\theta_\mu^*)^{-2}}.
\end{eqnarray}

\subsection{Cumulant generating function for exponential family distribution}
\label{sec:app_GLM_cumulant}

The moment-generating function for the exponential family distribution is given by
\begin{eqnarray}
\nonumber
{\cal M}_{{\cal L}(\theta)}(t)&=\int dy \exp(ty)f(y|\theta)\\
&=\frac{\exp(a(\theta+t))}{\exp(a(\theta))}\int dy\exp((t+\theta)y-a(t+\theta)+b(y)).
\end{eqnarray}
Hence, the cumulant generating function ${\cal C}_{{\cal L}(\theta)}(t)=\ln{\cal M}_{{\cal L}(\theta)}(t)$ is given by
\begin{eqnarray}
{\cal C}_{{\cal L}(\theta)}(t)=a(\theta+t)-a(\theta).
\end{eqnarray}
The $k$-th cumulant for the distribution $f(y|\theta)$
denoted by $\kappa_k[f(\cdot|\theta)]$ is given by
\begin{eqnarray}
\kappa_k[f(\cdot|\theta)]=\frac{\partial^k}{\partial t^k}{\cal C}_{f(\cdot|\theta)}(t)\Big|_{t=0}=
\frac{\partial^k}{\partial\theta^k}a(\theta),
\end{eqnarray}
hence the derivative of $a$ corresponds to the cumulant.

\section{Derivation of the updating rule of AMP}
\label{sec:app_AMP}

\subsection{General form of belief propagation}

The likelihood depends on the variable $\bm{x}$ 
through the parameter $\bm{\theta}=\frac{1}{\sqrt{N}}\bm{F}\bm{x}$.
Considering the input message to $i$,
we focus on $\theta^{-i}_\mu\equiv{\theta}_\mu-\frac{1}{\sqrt{N}}F_{\mu i}x_i$
and obtain the probability distribution of $\theta_\mu^{-i}$ using the messages.
For a variable transformation,
we introduce the identity for all $\mu\in\bm{M}$ as
\begin{equation}
1=\int d\theta_{\mu}^{-i}\delta\left(\theta_{\mu}^{-i}-\frac{1}{\sqrt{N}}\sum_{j\neq i}^NF_{\mu j}x_{j}\right).
\end{equation}
The distribution of $\theta_{\mu}^{-i}$
at step $t$, defined as $\mathrm{Pr}[\theta_{\mu}^{-i}=\theta]=W^{-i^{(t)}}_\mu(\theta)$, is given by
\begin{eqnarray}
\nonumber
W_{\mu}^{-i^{(t)}}(\theta)&=\int d\bm{x}_{\setminus i}\delta\left(\theta-\frac{1}{\sqrt{N}}\sum_{j\neq i}F_{\mu j}x_{j}\right)\prod_{j\in\bm{N}\backslash i}
\xi^{(t)}_{j\to\mu}(x_{j})\\
\nonumber
&=\frac{1}{2\pi}\int d\tilde{\theta}e^{-\sqrt{-1}\tilde{\theta}\theta}\prod_{j\in\bm{N}\setminus i}\int dx_{j}\exp\left(\sqrt{-1}\tilde{\theta}\frac{F_{\mu j}x_{j}}{\sqrt{N}}\right)
\xi^{(t)}_{j\to\mu}(x_{j})\\
\nonumber
&\simeq\frac{1}{2\pi}\int d\tilde{\theta}e^{-\sqrt{-1}\tilde{\theta}\theta}\exp\left(\sqrt{-1}\tilde{\theta}\widehat{\theta}_{\mu}^{\setminus\mu,-i^{(t)}}
-\frac{\widetilde{\theta}^2}{2}\beta^{-1}s_{\theta_\mu}^{\backslash\mu, -i^{(t)}}\right)\\
&=\frac{1}{\sqrt{2\pi s^{\backslash\mu,-i(t)}_{\theta_\mu}}}\exp\left(-\frac{\beta(\theta-\widehat{\theta}_{\mu}^{\setminus\mu,-i^{(t)}})^2}{2s_{\theta_\mu}^{\backslash\mu, -i^{(t)}}}\right),
\end{eqnarray}
where the integral representation of the delta function with the variable $\tilde{\theta}$ is introduced,
and we define
\begin{eqnarray}
\widehat{\theta}_{\mu}^{\setminus\mu, -i^{(t)}}&=\frac{1}{\sqrt{N}}\sum_{j\neq i}{F}_{\mu j}\widehat{x}^{(t)}_{j\to\mu}\\
s^{\setminus\mu, -i^{(t)}}_{\theta_\mu}&=\frac{1}{N}\sum_{j\neq i}
{F}^{2}_{\mu j}{s}^{(t)}_{j\to\mu}.
\end{eqnarray}
Using the distribution $W_{\mu}^{-i^{(t)}}(\theta)$,
the integral in the input message 
is
operated as  
\begin{eqnarray}
\nonumber
\widetilde{\xi}^{(t)}_{\mu\to i}(x_{i})&\propto
\int d\theta_{\mu}d\theta_\mu^{-i}
f^\beta(y_{\mu}|\theta_{\mu})
\delta\left(\theta_{\mu}-\theta_\mu^{-i}-\frac{F_{\mu i}x_{i}}{\sqrt{N}}\right)
W_{\mu}^{-i^{(t)}}(\theta_\mu^{-i})\\
\nonumber
&=\int d\theta_{\mu}
f^\beta(y_{\mu}|\theta_{\mu})
\exp\left(-\frac{\beta(\theta_{\mu}-\frac{F_{\mu i}x_{i}}{\sqrt{N}}-\widehat{\theta}_{\mu}^{\setminus\mu,-i^{(t)}})^2}{2s^{\setminus\mu,- i^{(t)}}_{\theta_\mu}}\right)
\\
\nonumber
&=\int d\theta_{\mu}f^\beta(y_{\mu}|\theta_{\mu})
\exp\left(-\frac{\beta(\theta_{\mu}-\widehat{\theta}_{\mu}^{\setminus\mu,-i^{(t)}})^2}{2s^{\setminus\mu,-i^{(t)}}_{\theta_\mu}}\right)\\
\nonumber
&\times\Big\{1+\frac{\beta F_{\mu i}x_{i}(\theta_{\mu}-\widehat{\theta}_{\mu}^{\setminus\mu,-i^{(t)}})}{\sqrt{N}s^{\setminus\mu,-i^{(t)}}_{\theta_\mu}}+\frac{\beta^2F^2_{\mu i}x^2_{i}(\theta_{\mu}-\widehat{\theta}_{\mu}^{\setminus\mu,-i^{(t)}})^2}{2N(s^{\setminus\mu,-i^{(t)}}_{\theta_\mu})^2}\\
&\hspace{1.0cm}-\frac{\beta F^2_{\mu i}x^2_{i}}{2Ns^{\setminus\mu,-i^{(t)}}_{\theta_\mu}}+O(N^{-3\slash 2})\Big\}.
\label{eq:input_dist_tmp}
\end{eqnarray}
Dividing by a constant does not change the proportionality, and therefore, 
we divide both sides of (\ref{eq:input_dist_tmp}) by
$\Xi_\beta(y_\mu;s^{\setminus\mu,-i^{(t)}}_{\theta_\mu},\widehat{\theta}_{\mu}^{\setminus\mu,-i^{(t)}})$ defined as
\begin{eqnarray}
\Xi_\beta(y;{s_{\theta}},\widehat{\theta})
&=\sqrt{\displaystyle\frac{\beta}{2\pi s_\theta}}\int d\theta
f_\beta(y|\theta) \exp\left(-\frac{\beta(\theta-\widehat{\theta})^2}{2s_\theta}\right).
\end{eqnarray}
Further, we define the effective posterior distribution after the observation of $y$ as
\begin{eqnarray}
\nonumber
&p_{\mathrm{c},\mu}(\theta_\mu;{s_{\theta}}_\mu,\widehat{\theta}_\mu^{\setminus\mu},y_\mu)\\
&\hspace{0.5cm}=\frac{1}{\Xi_\beta(y_\mu;{s_{\theta}}_\mu,\widehat{\theta}_\mu^{\setminus\mu})}\sqrt{\displaystyle\frac{\beta}{2\pi s_\theta}}f_\beta(y_\mu|\theta_\mu) \exp\left(-\frac{\beta(\theta_\mu-\widehat{\theta}_\mu^{\setminus\mu})^2}{2{s_\theta}_\mu}\right),
\end{eqnarray}
and 
\begin{eqnarray}
g_{\mathrm{out}}({s_{\theta}}_\mu,\widehat{\theta}_\mu^{\setminus\mu},y_\mu)&=\frac{\partial}{\partial\hat{\theta}_\mu^{\setminus\mu}}\ln \Xi_\beta(y_\mu;{s_{\theta}}_\mu,\widehat{\theta}_\mu^{\setminus\mu})=\Big\langle\frac{\beta(\theta_\mu-\widehat{\theta}_\mu^{\setminus\mu})}{{s_\theta}_\mu}\Big\rangle_{\theta_\mu}\label{eq:g_out_final}\\
\partial_{\hat{\theta}} g_{\mathrm{out}}({s_{\theta}}_\mu,\widehat{\theta}_\mu^{\setminus\mu},y_\mu)&=\frac{\partial}{\partial\hat{\theta}_\mu^{\setminus\mu}}g_{\mathrm{out}}({s_{\theta}}_\mu,\widehat{\theta}_\mu^{\setminus\mu},y_\mu),
\label{eq:g_out_p_final}
\end{eqnarray}
where $\langle\cdot\rangle_{\theta}$ denotes the expectation with respect to $\theta$ 
according to the distribution $p_{\mathrm{c},\mu}(\theta;{s_{\theta}},\widehat{\theta},y)$.
The following relationship exists:
\begin{eqnarray}
\partial_{\hat{\theta}}g_{\mathrm{out}}({s_{\theta}}_\mu,\widehat{\theta}_\mu^{\setminus\mu},y_\mu)&\!=\!-\frac{\beta}{{s_{\theta}}_\mu}
+\!\left\langle\!\!\left(\displaystyle\frac{\beta(\theta_\mu-\widehat{\theta}_\mu^{\setminus\mu})}{{s_{\theta}}_\mu}\right)^2\right\rangle_{\theta_\mu}\!\!\!\!\!-g_{\mathrm{out}}^2({s_{\theta}}_\mu,\widehat{\theta}_\mu^{\setminus\mu},y_\mu).
\label{eq:dg_out_relation}
\end{eqnarray}
Using Eqs.(\ref{eq:g_out_final})-(\ref{eq:dg_out_relation})
and ignoring $O(N^{-3\slash 2})$ terms,
we obtain
\begin{eqnarray}
\nonumber
\widetilde{\xi}_{\mu\to i}^{(t)}(x_{i})&\propto
1+\frac{F_{\mu i}x_{i}}{\sqrt{N}}
g_{\mathrm{out}}(s_{\theta_\mu}^{\setminus\mu,-i^{(t)}},\widehat{\theta}_\mu^{\setminus\mu,-i^{(t)}},y_\mu)\\
\nonumber
&+\frac{X^2_{i\ell}F^2_{\mu i}}{2N}
\Big\langle\frac{\beta^2(\theta_{\mu}-\widehat{\theta}_{\mu}^{\setminus\mu,-i^{(t)}})^2}{(s^{\setminus\mu,-i^{(t)}}_{\theta_\mu})^2}\Big\rangle_{\theta_\mu}-\frac{\beta x^2_{i}}{2Ns^{\setminus\mu,-i^{(t)}}_{\theta_\mu}}F^2_{\mu i}\\
\nonumber
&=1+\frac{F_{\mu i}x_{i}}{\sqrt{N}}
g_{\mathrm{out}}(s_{\theta_\mu}^{\setminus\mu,-i^{(t)}},\widehat{\theta}_\mu^{\setminus\mu,-i^{(t)}},y_\mu)\\
\nonumber
&+\frac{F^2_{\mu i}x^2_{i}}{2N}
g_{\mathrm{out}}^2(s_{\theta_\mu}^{\setminus\mu,-i^{(t)}},\widehat{\theta}_\mu^{\setminus\mu,-i^{(t)}},y_\mu)
+\frac{x^2_{i}F^2_{\mu i}}{2N}
\Big\langle\frac{\beta^2(\theta_{\mu}-\widehat{\theta}_{\mu}^{\setminus\mu,-i^{(t)}})^2}{(s^{\setminus\mu,-i^{(t)}}_{\theta_\mu})^2}\Big\rangle_{\theta_\mu}\\
&-\frac{\beta x^2_{i}}{2Ns^{\setminus\mu,-i^{(t)}}_{\theta_\mu}}F^2_{\mu i}
-\frac{F^2_{\mu i}x^2_{i}}{2N}
g_{\mathrm{out}}^2(s_{\theta_\mu}^{\setminus\mu,-i^{(t)}},\widehat{\theta}_\mu^{\setminus\mu,-i^{(t)}},y_\mu),
\end{eqnarray}
hence
\begin{eqnarray}
\widetilde{\xi}^{(t)}_{\mu\to i}(x_{i})&\propto 
\exp\left\{-\frac{\beta x_{i}^2}{2N}A^{(t)}_{\mu\to i}+x_{i}\frac{\beta B^{(t)}_{\mu\to i}}{\sqrt{N}}\right\},
\label{eq:input_final}
\end{eqnarray}
where
\begin{eqnarray}
A_{\mu\to i}^{(t)}&=-F^2_{\mu i}\beta^{-1}\partial_{\widehat{\theta}}g_{\mathrm{out}}(s_{\theta_\mu}^{\setminus\mu,-i^{(t)}},\widehat{\theta}_\mu^{\setminus\mu,-i^{(t)}},y_\mu)\label{eq:bp_A}\\
B_{\mu\to i}^{(t)}&=F_{\mu i}\beta^{-1}g_{\mathrm{out}}(s_{\theta_\mu}^{\setminus\mu,-i^{(t)}},\widehat{\theta}_\mu^{\setminus\mu,-i^{(t)}},y_\mu).\label{eq:bp_B}
\end{eqnarray}
Eq. (\ref{eq:input_final}) is a Gaussian distribution whose
mean $\widetilde{x}^{(t)}_{\mu\to i}$
and the rescaled variance $\widetilde{s}^{(t)}_{\mu\to i}$
are respectively given by
\begin{eqnarray}
\widetilde{x}^{(t)}_{\mu\to i}&=\frac{\sqrt{N}B^{(t)}_{\mu\to i}}{A^{(t)}_{\mu\to i}},~~~
\widetilde{s}^{(t)}_{\mu\to i}&=\frac{N}{A^{(t)}_{\mu\to i}}.\label{eq:input_stat}
\end{eqnarray}
We immediately obtain
the following expressions of the output messages using Gaussian approximation (\ref{eq:input_final}):  
\begin{eqnarray}
{\xi}^{(t+1)}_{i\to\mu}(x_{i})&\propto
\phi^\beta(x_{i})
\exp\left\{-\sum_{\nu\in\bm{M}\setminus \mu}\frac{\beta}{2\widetilde{s}^{(t)}_{\nu\to i}}
\left(x_{i}^2-2x_{i}\widetilde{x}_{\nu\to i}^{(t)}\right)\right\}.
\label{eq:output_tmp}
\end{eqnarray}
We define the following form of the distribution to express the output messages and marginal distribution in the unified form:
\begin{eqnarray}
{\cal P}_\beta(x;\Sigma,\mathrm{m})=\frac{1}{{\cal Z}}\phi^\beta(x)
\exp\left\{-\frac{\beta(x-\mathrm{m})^2}{2\Sigma}\right\},
\end{eqnarray}
where ${\cal Z}$ represents the normalization constant, and 
denotes its mean and variance as 
$\mathbb{M}_\beta(\Sigma,\mathrm{m})$ and $\beta^{-1}\mathbb{V}_\beta(\Sigma,\mathrm{m})$, 
respectively.
Transforming (\ref{eq:output_tmp}),
the mean and rescaled variance of the 
output message,
denoted by $\widehat{x}_{i\to\nu}$ and
${s}_{i\to\nu}$, respectively,
are given by
\begin{eqnarray}
\widehat{x}_{i\to\nu}^{(t+1)}=\mathbb{M}_\beta
\left({\Sigma}^{(t)}_{i\to\nu},\mathrm{m}^{(t)}_{i\to\nu}\right),~~~
{s}^{(t+1)}_{i\to\nu}=\beta\mathbb{V}_{\beta}
\left({\Sigma}^{(t)}_{i\to\nu},\mathrm{m}^{(t)}_{i\to\nu}\right),\label{eq:output_stat}
\end{eqnarray}
where
\begin{eqnarray}
{\Sigma}^{(t)}_{i\to\mu}=\frac{1}{\sum_{\nu\in\bm{M}\backslash \mu}
(\widetilde{s}^{(t)}_{\nu\to i})^{-1}},~~~
\mathrm{m}^{(t)}_{i\to\mu}=\frac
{\sum_{\nu\in\bm{M}\backslash \mu}(\widetilde{s}^{(t)}_{\nu\to i})^{-1}
\widetilde{x}^{(t)}_{\nu\to i}}{\sum_{\nu\in\bm{M}\backslash \mu}(\widetilde{s}^{(t)}_{\nu\to i})^{-1}}.\label{eq:output_macro_param}
\end{eqnarray}
Similarly, marginal distributions at time $t$ is given by
\begin{eqnarray}
p_{\beta,i}^{(t)}(x_{i})&={\cal P}_{\beta}(x_{i};{\Sigma}_{i}^{(t)},
\mathrm{m}_{i}^{(t)}),
\end{eqnarray}
where
\begin{eqnarray}
{\Sigma}^{(t)}_{i}=\frac{1}{\sum_{\nu=1}^M(\widetilde{s}^{(t)}_{\nu \to i})^{-1}},~~~
\mathrm{m}^{(t)}_{i}=\frac
{\sum_{\nu=1}^M(\widetilde{s}^{(t)}_{\nu\to i})^{-1}\widetilde{x}^{(t)}_{\nu \to i}}{\sum_{\nu=1}^M(\widetilde{s}^{(t)}_{\nu\to i})^{-1}},
\end{eqnarray}
and the estimate and rescaled variance are given by
\begin{eqnarray}
\widehat{x}_{i}^{(t)}=\mathbb{M}_\beta({\Sigma}_{i}^{(t)},\mathrm{m}_{i}^{(t)}),~~~
s_i^{(t)}=\beta\mathbb{V}_\beta({\Sigma}_{i}^{(t)},\mathrm{m}_{i}^{(t)}).
\end{eqnarray}

\subsection{TAP form of GAMP}

An iteration step of BP comprises the updating of 
$2MN$ messages following 
(\ref{eq:bp_A})-(\ref{eq:input_stat}), (\ref{eq:output_stat}), and 
(\ref{eq:output_macro_param}) for $i=1,\cdots,N$ and 
$\mu=1,\cdots,M$. 
Computational cost per iteration is $O(M^2\times N+M\times N^2)$, and therefore, the 
application of BP to large systems are not effective.
However, BP equations can be rewritten in terms of $M+N$ messages 
instead of $2MN$ messages under the assumption 
that matrix $\bm{F}$ is not sparse and all elements scale as $O(1)$.
The BP with reduced messages are known as 
Thouless--Anderson--Palmer (TAP) equations \cite{TAP} 
in the study of the statistical physics for spin glasses. 
The TAP form is equivalent to the BP equations at a sufficiently large $N$, 
and this is called the GAMP algorithm \cite{DMM}. 
The TAP form results in a significant reduction of computational complexity to 
$O(MN)$ per iteration.

First, we ignore $\frac{1}{N}F_{\mu i}^2s_{i\to\mu}\sim O(N^{-1})$ for a sufficiently large system; hence $
{s}_{\theta_\mu}^{\setminus\mu, -i^{(t)}}={s}^{\setminus\mu^{(t)}}_{\theta_\mu}$.
Further, using the relationship
\begin{eqnarray}
\widehat{\theta}_{\mu}^{\setminus\mu^{(t)}}&=\frac{1}{\sqrt{N}}\sum_{j=1}^NF_{\mu j}\widehat{x}_{j\to\mu}^{(t)}
=\widehat{\theta}_{\mu}^{\setminus\mu,-i^{(t)}}+\frac{1}{\sqrt{N}}F_{\mu i}\widehat{x}_{i\to\mu}^{(t)},
\end{eqnarray}
and applying Taylor expansion with respect to $O(N^{-1\slash 2})$ terms,
we obtain
\begin{eqnarray}
g_{\mathrm{out}}(s_{\theta_\mu}^{\setminus\mu,-i^{(t)}},\widehat{\theta}_\mu^{\setminus\mu,-i^{(t)}},y_\mu)&=g_{\mathrm{out}}(s_{\theta_\mu}^{\setminus\mu^{(t)}},\widehat{\theta}_\mu^{\setminus\mu^{(t)}},y_\mu)\label{eq:g_out_Taylor}\\
\nonumber
&-\partial_{\hat{\theta}}g_{\mathrm{out}}({s_{\theta}}^{\setminus\mu^{(t)}}_{\mu},\widehat{\theta}_{\mu}^{\setminus\mu^{(t)}},y_{\mu})\frac{F_{\mu i}\widehat{x}_{i\to\mu}^{(t)}}{\sqrt{N}}+O(N^{-1}).
\end{eqnarray}
For the sake of simplicity, we denote
$(g_{\mathrm{out}}^{(t)})_{\mu}=g_{\mathrm{out}}(s_{\theta_\mu}^{\setminus\mu^{(t)}},\widehat{\theta}_\mu^{\setminus\mu^{(t)}},y_\mu)$
and $(\partial_{\hat{\theta}}g_{\mathrm{out}}^{(t)})_{\mu}=\partial_{\hat{\theta}}g_{\mathrm{out}}(s_{\theta_\mu}^{\setminus\mu^{(t)}},\widehat{\theta}_\mu^{\setminus\mu^{(t)}},y_\mu)$.

From (\ref{eq:g_out_Taylor}), we obtain
\begin{eqnarray}
\mathbb{M}(\Sigma^{(t)}_{i},\mathrm{m}^{(t)}_{i})&=\mathbb{M}(\Sigma^{(t)}_{i\to\mu},\mathrm{m}^{(t)}_{i\to\mu})\\
\nonumber
&+
\frac{\partial}{\partial m^{(t)}_{i\to\mu}}\mathbb{M}(\Sigma^{(t)}_{i\to\mu},\mathrm{m}^{(t)}_{i\to\mu})
\frac{1}{\sqrt{N}}\Sigma^{(t)}_{i\to\mu}B_{\mu\to i}^{(t)}+O(N^{-1})\\
\nonumber
&=\mathbb{M}(\Sigma^{(t)}_{i\to\mu},\mathrm{m}^{(t)}_{i\to\mu})+
\beta \mathbb{V}(\Sigma^{(t)}_{i},\mathrm{m}^{(t)}_{i})\frac{F_{\mu i}}{\sqrt{N}}\beta^{-1}(g_{\mathrm{out}}^{(t)})_{\mu}+O(N^{-1}),
\end{eqnarray}
where we use the relationship (\ref{eq:VtoM})
and ${\Sigma}_{i}^{(t)}={\Sigma}_{i\to\mu}^{(t)}+O(N^{-1})$.
Immediately, we get
\begin{eqnarray}
\widehat{x}_{i\to\mu}^{(t+1)}=\widehat{x}_{i}^{(t+1)}-\frac{1}{\sqrt{N}}{s}_{i}^{(t+1)}F_{\mu i}\beta^{-1}g_{\mathrm{out}}(s_{\theta_\mu}^{\setminus\mu^{(t)}},\widehat{\theta}_\mu^{\setminus\mu^{(t)}},y_\mu),
\end{eqnarray}
and $\widehat{x}_{i}^{(t+1)}$ and $s_i^{(t+1)}$ are obtained from 
$\mathbb{M}_\beta(\Sigma_i^{(t)},\mathrm{m}_i^{(t)})$ and 
$\beta\mathbb{V}_\beta(\Sigma_i^{(t)},\mathrm{m}_i^{(t)})$ where
\begin{eqnarray}
\mathrm{m}_{i}^{(t)}&={\Sigma}_{i}^{(t)}\sum_{\mu=1}^M\left\{
\frac{1}{\sqrt{N}}F_{\mu i}\beta^{-1}(g_{\mathrm{out}}^{(t)})_{\mu}
-\frac{1}{N}F_{\mu i}^{2}\widehat{x}_{i}^{(t)}
\beta^{-1}(\partial_{\widehat{\theta}} g_{\mathrm{out}}^{(t)})_{\mu}\right\}\\
\frac{1}{{\Sigma}_{i}^{(t)}}&=-\frac{1}{N}\sum_{\mu=1}^MF^{2}_{\mu i}\beta^{-1}(\partial_{\widehat{\theta}}g_{\mathrm{out}}^{(t)})_{\mu}.
\end{eqnarray}
Further, at a sufficiently large $N$,
the relationship between the mean and variance of the parameter $\theta_\mu$
in the $\mu$-cavity and full system are presented by
\begin{eqnarray}
\widehat{\theta}_{\mu}^{\setminus\mu^{(t)}}&=
\widehat{\theta}_\mu^{(t)}-\beta^{-1}(g_{\mathrm{out}}^{(t-1)})_{\mu}{s_\theta}_{\mu}^{\setminus\mu^{(t)}}
\label{eq:theta_hat_app}\\
{s_\theta}_{\mu}^{\setminus\mu^{(t)}}&={s_\theta}_{\mu}^{(t)}=\frac{1}{N}
\sum_{j=1}^N
F^{2}_{\mu j}{s}^{(t)}_{j}.
\end{eqnarray}


\section{Coordinate descent algorithm}
\label{sec:app_CD}

Coordinate descent (CD) is an algorithm that can optimise each variable
successively where other variables are fixed.
The CD algorithm is efficient for penalised regression problems, and 
partially for non-convex optimization problems \cite{Breheny,Sakata-Xu}.

\subsection{Gaussian likelihood with penalty}

We focus on the variable $x_i$ and optimise it 
under the fixed variables $x_j=\widehat{x}_j$ where $j\in\bm{N}\setminus i$.
For Gaussian likelihood,
we rewrite the function to be minimised as 
\begin{eqnarray}
\varphi&=-\frac{1}{2}\sum_{\mu=1}^M
\!\left\{-2\frac{y_\mu}{\sqrt{N}}F_{\mu i}x_i
\!+\!\frac{1}{N}F_{\mu i}^2x_i^2\!+\!\frac{2}{\sqrt{N}}F_{\mu i}x_i\widehat{\theta}_\mu^{-i}\!\right\}\!-\!h(x_i)\!+\!\mathrm{const.},
\label{eq:app_CD_i}
\end{eqnarray}
where we set 
$\widehat{\theta}_\mu^{-i}=\frac{1}{\sqrt{N}}\sum_{j\neq i}F_{\mu j}\widehat{x}_j$.
In the case of the elastic net penalty,
the maximiser of (\ref{eq:app_CD_i}) is given by
\begin{eqnarray}
\widehat{x}_i=\frac{b_i-\lambda_1\mathrm{sgn}(b_i)}{\frac{1}{N}\sum_{\mu=1}^MF_{\mu i}^2+\lambda_2}\mathbb{I}\left(|b_i|>\lambda_1\right)
\end{eqnarray}
where
\begin{eqnarray}
b_i=\frac{1}{\sqrt{N}}\sum_{\mu=1}^M(y_\mu-\widehat{\theta}_\mu^{-i}) F_{\mu i}.
\end{eqnarray}
One can obtain the estimate by repeating this procedure for $i\in\bm{N}$
sufficient time steps.

\subsection{Coordinate descent for GLM}
\label{sec:app_CD_GLM}

For general likelihood,
the proximal-Newton iterative approach can be combined with the CD algorithm where the likelihood is repeatedly approximated by a quadratic function
\cite{Lee2014,sparse_book}.
We obtain the following by expanding the log-likelihood up to the second order around the estimate at step $t$, denoted by $\widehat{\bm{x}}^{(t)}$: 
\begin{eqnarray}
\nonumber
&\ln f(\bm{y}|\bm{\theta}(\bm{F},\bm{x}))\\
\nonumber
&\simeq\ln f(\bm{y}|\bm{\theta}(\bm{F},\widehat{\bm{x}}^{(t)}))
+\sum_{i=1}^N\frac{\partial}{\partial x_i}\ln f(\bm{y}|\bm{\theta}(\bm{F},\bm{x}))\Big|_{\bm{x}=\widehat{\bm{x}}^{(t)}}(x_i-\widehat{x}_i^{(t)})\\
\nonumber
&+\frac{1}{2}\sum_{ij}\frac{\partial^2}{\partial x_i\partial x_j}\ln f(\bm{y}|\bm{\theta}(\bm{x}))\Big|_{\bm{x}=\widehat{\bm{x}}^{(t)}}(x_i-\widehat{x}_i^{(t)})(x_j-\widehat{x}_j^{(t)})\\
\nonumber
&=\ln f(\bm{y}|\bm{\theta}(\bm{F},\widehat{\bm{x}}^{(t)}))
+\frac{1}{\sqrt{N}}\sum_{i=1}^N\sum_{\mu=1}^M F_{\mu i}(y_\mu-a^\prime(\theta_\mu^{(t)}))(x_i-\widehat{x}_i^{(t)})\\
&-\frac{1}{2N}\sum_{ij}\sum_{\mu=1}^MF_{\mu i}F_{\mu j}a^{\prime\prime}(\theta_\mu^{(t)})(x_i-\widehat{x}_i^{(t)})(x_j-\widehat{x}_j^{(t)}).
\end{eqnarray}
Hence, the penalized likelihood is approximated by
\begin{eqnarray}
\nonumber
\widetilde{\varphi}^{(t)}(\bm{x})&=-\frac{1}{2}(\bm{x}-\widehat{\bm{x}}^{(t)}-{\bm{W}^{(t)}}^{-1}\bm{R}^{(t)})^{\top}\bm{W}(\bm{x}-\widehat{\bm{x}}^{(t)}-{\bm{W}^{(t)}}^{-1}\bm{R}^{(t)})\\
&\hspace{1.0cm}-h(\bm{x})+\mathrm{const.},
\label{eq:CD_approximation}
\end{eqnarray}
where
\begin{eqnarray}
\bm{R}^{(t)}&=\frac{1}{\sqrt{N}}\bm{F}^{\top}(\bm{y}-a^\prime(\bm{\theta}(\widehat{\bm{x}}^{(t)})))\\
\bm{W}^{(t)}&=\frac{1}{N}\bm{F}^{\top}\bm{D}(\bm{\theta}(\widehat{\bm{x}}^{(t)}))\bm{F},
\end{eqnarray}
and $\bm{D}(\bm{\theta})\in\mathbb{R}^{M\times M}$ is the diagonal matrix
whose $(\mu,\mu)$ component corresponds to $a^{\prime\prime}(\theta_\mu(\widehat{\bm{x}}^{(t)}))$.
We solve (\ref{eq:CD_approximation}) using the CD algorithm.
When the penalty is $\ell_2$ norm,
the solution is given by
\begin{eqnarray}
\hat{\bm{x}}^{(t+1)}=(\bm{W}^{(t)}+\lambda\bm{I}_N)^{-1}\bm{W}^{(t)}(\widehat{\bm{x}}^{(t)}+
{\bm{W}^{(t)}}^{-1}\bm{R}^{(t)}).
\end{eqnarray}

\section{Replica method}

\subsection{Derivation of the saddle point equations for conjugate variables}
\label{sec:app_saddle}

From the extremization condition of the free energy density,
we obtain saddle point equations as
\begin{eqnarray}
\widehat{\Theta}&=-2\alpha\int D\zeta D\nu\frac{\partial}{\partial Q}
\widehat{f}_\xi^{(\mathrm{RS})*}(\nu,\zeta)\\
\widehat{\chi}&=2\alpha\int D\zeta D\nu\frac{\partial}{\partial \chi}
\widehat{f}_\xi^{(\mathrm{RS})*}(\nu,\zeta)\\
\widehat{\mu}&=\alpha\int D\zeta D\nu\frac{\partial}{\partial m}
\widehat{f}_\xi^{(\mathrm{RS})*}(\nu,\zeta),
\end{eqnarray}
where
\begin{eqnarray}
\nonumber
\widehat{f}_\xi^{(\mathrm{RS})*}(\nu,\zeta)&=-\frac{1}{2\chi}
\left(u_{\mathrm{RS}}^*(\nu,\zeta)-\frac{m\nu+\sqrt{Q\sigma_T^2-m^2}\zeta}{\sigma_T}\right)^2\\
&\hspace{3.0cm}+\ln f\left({\cal G}\left(\sigma_T\nu\right)|u_{\mathrm{RS}}^*(\nu,\zeta)\right)
\end{eqnarray}
and in case of GLM, $u_{\mathrm{RS}}^*(\nu,\zeta)$ satisfies
\begin{eqnarray}
\frac{u_{\mathrm{RS}}^*(\nu,\zeta)-(m\nu+\sqrt{Q\sigma_T^2-m^2}\zeta)\sigma_T^{-1}}{\chi}={\cal G}\left(\sigma_T\nu\right)-a^\prime(u_{\mathrm{RS}}^*(\nu,\zeta)).
\label{eq_app:u_GLM}
\end{eqnarray}
For the derivation of the saddle point equations,
it is convenient to derive the expression of 
$\partial u_{\mathrm{RS}}^*(\nu,\zeta)\slash\partial \zeta$ and $\partial u_{\mathrm{RS}}^*(\nu,\zeta)\slash\partial \nu$. 
From \eref{eq_app:u_GLM}, $\partial u_{\mathrm{RS}}^*(\nu,\zeta)\slash\partial \zeta$
satisfies
\begin{eqnarray}
\frac{1}{\chi}\left(\frac{\partial u_{\mathrm{RS}}^*(\nu,\zeta)}{\partial\zeta}-\sqrt{\frac{Q\sigma_T^2-m^2}{\sigma_T^2}}\right)=-\frac{\partial u_{\mathrm{RS}}^*(\nu,\zeta)}{\partial \zeta}a^{\prime\prime}(u_{\mathrm{RS}}^*(\nu,\zeta)).
\end{eqnarray}
Solving this, we obtain
\begin{eqnarray}
\frac{\partial u_{\mathrm{RS}}^*(\nu,\zeta)}{\partial \zeta}=\sqrt{\frac{Q\sigma_T^2-m^2}{\sigma_T^2}}\frac{1}{1+\chi a^{\prime\prime}(u_{\mathrm{RS}}^*(\nu,\zeta))}.
\label{eq_app:dudzeta}
\end{eqnarray}
Similarly, $\partial u^*(\nu,\zeta)\slash\partial \nu$
satisfies
\begin{eqnarray}
\frac{1}{\chi}\left(\frac{\partial u_{\mathrm{RS}}^*(\nu,\zeta)}{\partial \nu}-\frac{m}{\sigma_T}\right)=\frac{\partial {\cal G}(\sigma_T\nu,\varepsilon)}{\partial\nu}-\frac{\partial u_{\mathrm{RS}}^*(\nu,\zeta)}{\partial \nu}a^{\prime\prime}(u_{\mathrm{RS}}^*(\nu,\zeta)),
\end{eqnarray}
and solving this, we obtain
\begin{eqnarray}
\frac{\partial u_{\mathrm{RS}}^*(\nu,\zeta)}{\partial \nu}=\frac{1}{1+\chi a^{\prime\prime}(u_{\mathrm{RS}}^*(\nu,\zeta))}
\left(\chi\frac{\partial{\cal G}(\sigma_T\nu,\varepsilon)}{\partial\nu}+\frac{m}{\sigma_T}\right).
\label{eq_app:dudnu}
\end{eqnarray}

Applying (\ref{eq_app:dudzeta}) to the saddle point equation of $\widehat{\Theta}$, 
\begin{eqnarray}
\nonumber
\int D\zeta\frac{\partial}{\partial Q}&\widehat{f}_\xi^{(\mathrm{RS})*}(\nu,\zeta)\\
\nonumber
&=\int D\zeta\frac{\zeta\sigma_T}{2\chi\sqrt{Q\sigma_T^2-m^2}}
\left(u_{\mathrm{RS}}^*(\nu,\zeta)-\frac{m\nu+\sqrt{Q\sigma_T^2-m^2}\zeta}{\sigma_T}\right)\\
\nonumber
&=\frac{\sigma_T}{2\chi\sqrt{Q\sigma_T^2-m^2}}\int D\zeta
\left(\frac{\partial u_{\mathrm{RS}}^*(\nu,\zeta)}{\partial\zeta}-\frac{\sqrt{Q\sigma_T^2-m^2}}{\sigma_T}\right)\\
&=\frac{1}{2\chi}\int D\zeta\left(\frac{1}{1+\chi a^{\prime\prime}(u_{\mathrm{RS}}^*(\nu,\zeta))}-1\right),
\end{eqnarray}
hence we obtain
\begin{eqnarray}
\widehat{\Theta}=\alpha\int D\zeta D\nu
\frac{a^{\prime\prime}(u_{\mathrm{RS}}^*(\nu,\zeta))}{1+\chi a^{\prime\prime}(u_{\mathrm{RS}}^*(\nu,\zeta)))}.
\end{eqnarray}
Next, for the saddle point equation of $\widehat{\chi}$,
\begin{eqnarray}
\frac{\partial}{\partial\chi}\widehat{f}_\xi^{(\mathrm{RS})*}(\nu,\zeta)=\frac{1}{\chi^2}\left(u_{\mathrm{RS}}^*(\nu,\zeta)-\frac{m\nu+\sqrt{Q\sigma_T^2-m^2}\zeta}{\sigma_T}\right)^2,
\end{eqnarray}
and using (\ref{eq_app:u_GLM}), we obtain
\begin{eqnarray}
\widehat{\chi}=\alpha\int D\zeta D\nu
\left({\cal G}\left(\sigma_T\nu\right)-a^\prime(u_{\mathrm{RS}}^*(\nu,\zeta))\right)^2.
\end{eqnarray}
For the saddle point equation of $\widehat{\mu}$, we obtain
\begin{eqnarray}
\nonumber
&\int D\zeta D\nu\frac{\partial}{\partial m}\widehat{f}_\xi^{(\mathrm{RS})*}(\nu,\zeta)\\
\nonumber
&=\frac{1}{\chi\sigma_T}\int D\zeta D\nu
\left(\nu-\frac{m\zeta}{\sqrt{Q\sigma_T^2-m^2}}\right)
\left(u_{\mathrm{RS}}^*-\frac{m\nu+\sqrt{Q\sigma_T^2-m^2}\zeta}{\sigma_T}\right)\\
\nonumber
&=\frac{1}{\chi\sigma_x}\int D\zeta D\nu
\Big\{\frac{\partial u_{\mathrm{RS}}^*(\nu,\zeta)}{\partial\nu}-\frac{m}{\sigma_T}\\
&\hspace{1.5cm}-\frac{m}{\sqrt{Q\sigma_T^2-m^2}}\left(\frac{\partial u_{\mathrm{RS}}^*(\nu,\zeta)}{\partial\zeta}-\frac{\sqrt{Q\sigma_T^2-m^2}}{\sigma_T}
\right)\Big\},
\label{eq_app:mh_tmp}
\end{eqnarray}
and substituting (\ref{eq_app:dudzeta}) and (\ref{eq_app:dudnu}) into 
(\ref{eq_app:mh_tmp}), the following equation is obtained
\begin{eqnarray}
\widehat{\mu}=\alpha\int D\zeta D\nu
\frac{1}{1+\chi a^{\prime\prime}(u_{\mathrm{RS}}^*(\nu,\zeta))}\frac{\partial{\cal G}(\sigma_T\nu)}{\partial(\sigma_T\nu)}.
\end{eqnarray}

\subsection{Simplified calculation for Gaussian noise}

When the likelihood under consideration is the Gaussian distribution,
the calculation can be simplified.
We define 
\begin{eqnarray}
u^{(a)}_\mu=\frac{1}{\sqrt{N}}\sum_{i=1}^NF_{\mu i}(x_i^{(0)}-x_i^{(a)})+\varepsilon_\mu
\end{eqnarray}
for $\mu\in\bm{M}$ and $a=1,\ldots,n$. 
These variables satisfy
\begin{eqnarray}
\mathbb{E}_{\bm{y}|q^{(ab)}}\left[u_\mu^{(a)}u_\mu^{(b)}\right]=
q^{(ab)}-m^{(a)}-m^{(b)}+\sigma_T^2
\end{eqnarray}
under the subshell, hence their joint distribution is given by
\begin{eqnarray}
P_{\bm{u}}(\widetilde{\bm{u}}|{\cal V})=\frac{1}{\sqrt{(2\pi)^n|{\cal V}|}}\exp\left(-\frac{1}{2}\widetilde{\bm{u}}{\cal V}^{-1}\widetilde{\bm{u}}^{\top}\right).
\end{eqnarray}
Combining with the likelihood,
we obtain
\begin{eqnarray}
\Phi_\beta^{(E)}(n)=-\frac{1}{2}\ln|\bm{I}_n+\beta{\cal V}|.
\end{eqnarray}
Under the RS assumption,
the diagonal components and offdiagonal components of ${\cal V}$,
denoted by ${\cal V}_{\mathrm{d}}$ and ${\cal V}_{\mathrm{nd}}$, respectively,
are given by
\begin{eqnarray}
{\cal V}_{\mathrm{d}}&=Q-2m+\sigma_T^2\\
{\cal V}_{\mathrm{nd}}&=q-2m+\sigma_T^2.
\end{eqnarray}
Calculating the matrix determinant, we obtain
\begin{eqnarray}
\Psi_\beta^{(E,\mathrm{RS})}(n)\!=\!-\frac{\alpha}{2}
\left[\ln\left\{1+\frac{n\beta(q-2m+\sigma_T^2)}{1+\beta(Q-q)}\right\}
+n\ln\left\{1+\beta(Q-q)\right\}\!\right].
\end{eqnarray}
Taking $n\to 0$ limit and $\beta\to\infty$ limit,
the free energy density for the Gaussian likelihood is derived as
\begin{eqnarray}
\Psi=-\widehat{\mu}m+\frac{\widehat{\Theta}Q-\widehat{\chi}\chi}{2}
-\frac{\alpha(Q-2m+\sigma_T^2)}{2(1+\chi)}+
\frac{\widehat{\chi}+\widehat{\mu}^2\sigma_x^2}{2(\widehat{\Theta}+\lambda)}.
\end{eqnarray}

\subsection{1RSB assumption and AT instability}
\label{sec:app_1RSB}

The local stability of the RS solution against a symmetry breaking perturbation is
examined by constructing the one-step replica symmetry breaking (1RSB) solution.
We introduce the 1RSB assumption with respect to the saddle point as
\begin{eqnarray}
(q^{(ab)},\widehat{q}^{(ab)})&=\left\{ \begin{array}{ll}
\left(Q,-\displaystyle\frac{\widehat{Q}}{2}\right)  &  {\rm if}~a= b\\
(q_1,\widehat{q}_1)  &  {\rm if}~a\neq b,~B(a)=B(b)\\
(q_0,\widehat{q}_0)  &  {\rm if}~B(a)\neq B(b)\\
\end{array}, \right.
\label{eq:1RSB_q}\\
(m^{(a)},\widehat{m}^{(a)})&=(m,\widehat{m})~{\rm for~any}~a,
\label{eq:1RSB_m}
\end{eqnarray}
where $n$ replicas are separated into $n\slash\ell$ 
blocks that contain $\ell$ replicas. The index of the
block that includes the $a$-th replica is denoted by $B(a)$.
For the integral of $\theta$ under the 1RSB assumption,
we set 
\begin{eqnarray}
\theta^{(a)}&=\sqrt{Q-q_1}w^{(a)}+\sqrt{\Delta_q}z_1^{(B(a))}+\sqrt{q_0}z_0\\
\theta^{(0)}&=\sqrt{\sigma_T-\frac{m^2}{q_0}}v+\frac{m}{\sqrt{q_0}}z_0,
\end{eqnarray}
where $w^{(a)}~(a\in\bm{n})$, $z_1^{(B)}~(B\in\{1,\ldots,n\slash\ell\})$, $v$, and $z_0$
are the independent Gaussian random variables
with mean 0 and variance 1.
Utilizing these expressions,
we obtain
\begin{eqnarray}
\Psi_\beta^{(C,\mathrm{1RSB})}&=-\widehat{m}m+\frac{\widehat{Q}Q+\widehat{q}_1q_1}{2}
-\frac{\ell(\widehat{q}_1q_1-\widehat{q}_0q_0)}{2}\\
\Psi_\beta^{(E,\mathrm{1RSB})}&=\frac{1}{\ell}\int Dz_0Dv\ln\widehat{\xi}_\beta^{(\mathrm{RSB})}
(v,z_0)\\
\Psi_\beta^{(S,\mathrm{1RSB})}&=\frac{1}{\ell}\int Dz_0dx^{(0)}\phi^{(0)}(x^{(0)})\ln\xi_\beta^{(\mathrm{RSB})}(z_0,x^{(0)}),
\end{eqnarray}
where we set 
\begin{eqnarray}
&\widehat{\xi}_\beta^{(\mathrm{RSB})}(v,z_0)=\int Dz_1\left\{\!\int Dw f^\beta\!\!\left({\cal G}\left(\!\sqrt{\sigma_{T}^2\!-\!\frac{m^2}{q_0}}v\!+\!\frac{m}{\sqrt{q_0}}z\!\right)\!\Big|u\!\right)\!\!\right\}^{\ell}\\
&\xi_\beta^{(\mathrm{RSB})}(z_0,x^{(0)})=\int Dz_1\\
\nonumber
&\times\left\{\int dx\phi^\beta(x)
\exp\left(-\frac{\widehat{Q}+\widehat{q}_1}{2}x^2+
(\sqrt{\widehat{q}_1-\widehat{q}_0}z_1
+\sqrt{\hat{q}_0}z_0+\hat{m}x^{(0)})x\right)\right\}^{\ell},
\end{eqnarray} 
and $u=\sqrt{Q-q_1}w+\sqrt{q_1-q_0}z_1+\sqrt{q_0}z_0$.
Using $u$, we apply the expression
$w=(u-\sqrt{q_1-q_0}z_1-\sqrt{q_0}z_0)\slash\sqrt{Q-q_1}$,
and further, we apply following transformation from $(v,z)$ to
$(\nu,\zeta)$ that keeps $dvdz=d\nu d\zeta$ as
\begin{eqnarray}
\sqrt{\sigma_T^2-\frac{m^2}{q_0}}v+\frac{m}{\sqrt{q_0}}z&=\sigma_T\nu\\
-\frac{m}{\sqrt{q_0}}v+\sqrt{\sigma_T^2-\frac{m^2}{q_0}}z&=
\sigma_T\zeta,
\end{eqnarray}
hence $\sqrt{q_0}z=(m\nu+\sqrt{q_0\sigma_T^2-m^2}\zeta)\sigma_T^{-1}$.
At $\beta\to\infty$, the variables scale as 
$\beta(Q-q_1)=\chi$, 
$\beta\widehat{\Theta}=\widehat{Q}+\widehat{q}_1$,
$\beta^2\widehat{\chi}_1=\widehat{q}_1$,
$\beta^2\widehat{\chi}_0=\widehat{q}_0$,
$\beta\widehat{\mu}=\widehat{m}$,
and $\widetilde{\ell}=\beta\ell$.
With these scaled variables, we obtain
\begin{eqnarray}
\widehat{\xi}_{\beta}^{\mathrm{(RSB)}}(\nu,\zeta)=\int Dz_1\exp\left(\widetilde{\ell}\widehat{f}_\xi^{\mathrm{(RSB)}*}(z_1,\nu,\zeta)\right),
\end{eqnarray}
where we set $\Delta_q=q_1-q_0$, 
$\widehat{f}_\xi^{\mathrm{(RSB)}*}(z_1,\nu,\zeta)=\widehat{f}_\xi^{(\mathrm{RSB})}(u^*_{\mathrm{(RSB)}};z_1,\nu,\zeta)$
with
\begin{eqnarray}
\nonumber
\widehat{f}_\xi^{(\mathrm{RSB})}(u;z_1,\nu,\zeta)&=
\frac{-(u-\sqrt{\Delta_q}z_1-(m\nu+\sqrt{q_0\sigma_T^2-m^2}\zeta)\sigma_T^{-1})^2}{2\chi}\\
&\hspace{5cm}+\ln f\left({\cal G}\left(\sigma_T\nu\right)|u\right).\\
u_{\mathrm{RSB}}^*(z_1,\nu,\zeta)&=\mathop{\mathrm{argmax}}_u\widehat{f}^{\mathrm{(RSB)}}_\xi(u;z_1,\nu,\zeta).
\end{eqnarray}
In case of GLM, $u^*_{\mathrm{RSB}}$ satisfies 
\begin{eqnarray}
\frac{u_{\mathrm{RSB}}^*-\sqrt{\Delta_q}z_1-(m\nu+\sqrt{q_0\sigma_T^2-m^2}\zeta)\sigma_T^{-1}}{\chi}={\cal G}\left(\sigma_T\nu\right)-a^\prime(u^*_{\mathrm{RSB}}).
\end{eqnarray}

Next, by applying the saddle point method to $\xi_\beta^{\mathrm{RSB}}$, we obtain
\begin{eqnarray}
\xi_\beta^{\mathrm{(RSB)}}(z_0,x^{(0)})
=\int Dz_1\exp\left(\widetilde{\ell}f_\xi^{\mathrm{(RSB)}*}(z_1,z_0,x^{(0)})\right),
\end{eqnarray}
where we set $\widehat{\Delta}_q=\widehat{q}_1-\widehat{q}_0$, 
and
$f_\xi^*(z_1,z_0,x^{(0)})=f_\xi^{(\mathrm{RSB})}(x^*_{\mathrm{RSB}};z_1,z_0,x^{(0)})$
with
\begin{eqnarray}
f_\xi^{(\mathrm{RSB})}(x;z_1,z_0,x^{(0)})&=-\frac{\hat{\Theta}}{2}{x}^2+\!(\sqrt{\hat{\Delta}_q}z_1+\!\sqrt{\hat{\chi}_0}z+\!\hat{\mu}x^{(0)})x+\!\ln\phi(x)\\
x_{\mathrm{RSB}}^*(z_1,z_0,x^{(0)})&=\mathop{\mathrm{argmax}}_x
f_\xi^{(\mathrm{RSB})}(x;z_1,z_0,x^{(0)}).
\label{eq:1RSB_one_body_x}
\end{eqnarray}
In summary,
the free energy density under the 1RSB assumption is given by
\begin{eqnarray}
\nonumber
\Psi^{(\mathrm{1RSB})}&=\mathop{\mathrm{extr}}_{\Omega,\widehat{\Omega}}
\Big[\widehat{\mu}m-\frac{\widehat{\Theta}Q-\widehat{\chi}_1\chi}{2}
+\frac{\widetilde{\ell}(\widehat{\chi}_1Q-\widehat{\chi}_0q_0)}{2}\\
\nonumber
&-\frac{1}{\widetilde{\ell}}\int dx^{(0)}\phi^{(0)}(x^{(0)})Dz_0\ln\xi_\beta^{(\mathrm{RSB})}(z_0,x^{(0)})\\
&-\frac{\alpha}{\widetilde{\ell}}\int D\nu D\zeta\ln\widehat{\xi}_\beta^{(\mathrm{RSB})}(\nu,\zeta)\Big],
\end{eqnarray}
where $\Omega=\{Q,\chi,q_0,m\}$ and $\widehat{\Omega}=\{\widehat{\Theta},\widehat{\chi}_1,\widehat{\chi}_0,\widehat{\mu}\}$.
The saddle point equations are derived as
\begin{eqnarray}
Q&=\int dx^{(0)}\phi^{(0)}Dz_0\langle {x^*_{\mathrm{RSB}}}^2\rangle_{\mathrm{RSB}}\\
\chi&=\frac{1}{\sqrt{\widehat{\Delta}_q}}\int dx^{(0)}\phi^{(0)}Dz_0\left\langle \frac{\partial x^*_{\mathrm{RSB}}}{\partial z_1}\right\rangle_{\mathrm{RSB}}\\
q_0&=\int Dz_0dx^{(0)}\phi^{(0)}(x^{(0)})\langle {x^*_{\mathrm{RSB}}}\rangle_{\mathrm{RSB}}^2\\
m&=\int Dz_0dx^{(0)}\phi^{(0)}(x^{(0)})x^{(0)}\langle {x^*_{\mathrm{RSB}}}\rangle_{\mathrm{RSB}}\\
\widehat{\Theta}&=\alpha\int D\nu D\zeta\left\langle\frac{a^{\prime\prime}(u^*_{\mathrm{RSB}})}{1+\chi a^{\prime\prime}(u^*_{\mathrm{RSB}})}\right\rangle_{\widehat{\mathrm{RSB}}}\\
\widehat{\chi}_1&=\alpha\int D\nu D\zeta\left\langle\left({\cal G}\left(\sigma_T\nu\right)-a^\prime(u^*_{\mathrm{RSB}})\right)^2\right\rangle_{\widehat{\mathrm{RSB}}}\\
\widehat{\chi}_0&=\alpha\int D\nu D\zeta\left\langle {\cal G}\left(\sigma_T\nu\right)-a^\prime(u^*_{\mathrm{RSB}})\right\rangle^2_{\widehat{\mathrm{RSB}}}\\
\widehat{\mu}&=\alpha\int D\nu D\zeta \frac{\partial{\cal G}(\sigma_T\nu)}{\partial (\sigma_T\nu)}\left\langle\frac{1}{1+\chi a^{\prime\prime}(u^*_{\mathrm{RSB}})}\right\rangle_{\widehat{\mathrm{RSB}}},
\end{eqnarray}
where we define
\begin{eqnarray}
\langle\varphi(z_1)\rangle_{\mathrm{RSB}}&=\frac{\int Dz_1\varphi(z_1)\exp(\widetilde{\ell}f_\xi^{(\mathrm{RSB})*})}{\int Dz_1\exp(\widetilde{\ell}f_\xi^{(\mathrm{RSB})*})} \\
\langle \varphi(z_1) \rangle_{\widehat{\mathrm{RSB}}}&=\frac{\int Dz_1 \varphi(z_1)\exp(\widetilde{\ell}\widehat{f}_\xi^{(\mathrm{RSB})*})}{\int Dz_1 \exp(\widetilde{\ell}\widehat{f}_\xi^{(\mathrm{RSB})*})}.
\end{eqnarray}
When $\Delta_q=\widehat{\Delta}_q=0$,
the 1RSB solution is reduced to the RS solution. 
Hence, the stability of the RS solution is examined by the
stability analysis of the 1RSB solution of $\Delta_q=\widehat{\Delta}_q=0$. 
We apply Taylor expansion to them by assuming that they are
sufficiently small. 

For the expansion of $\Delta_q$ around $\widehat{\Delta}_q=0$
and that of $\widehat{\Delta}_q$ around $\Delta_q$,
$x^*_{\mathrm{RSB}}$ and $u^*_{\mathrm{RSB}}$
for finite $\beta$ are useful. We denote them as 
\begin{eqnarray}
\langle x\rangle=\frac{\int dx x\exp(-\beta f_\xi^{(\mathrm{RSB})}(x;z_1,z_0,x^{(0)}))}{\int dx\exp(-\beta f_\xi^{(\mathrm{RSB})}(x;z_1,z_0,x^{(0)}))}\label{eq:app_x_finite}\\
\langle u\rangle=\frac{\int du u\exp(-\beta \widehat{f}_\xi^{(\mathrm{RSB})}(u;z_1,\nu,\zeta))}{\int du \exp(-\beta \widehat{f}_\xi^{(\mathrm{RSB})}(u;z_1,\nu,\zeta))},\label{eq:app_u_finite}
\end{eqnarray}
respectively.
Using eqs. (\ref{eq:app_x_finite}) and (\ref{eq:app_u_finite}), 
we obtain
\begin{eqnarray}
\Delta_q&=\int Dzdx^{(0)}\phi^{(0)}(x^{(0)})\beta^2\left(\langle x^2 \rangle-\langle x\rangle^2\right)^2\widehat{\Delta}_q\\
\widehat{\Delta}_q&=\alpha\int D\nu D\zeta \beta^2\left(\left\langle({\cal G}\left(\sigma_T\nu\right)-a^\prime(u))^2\right \rangle-\langle {\cal G}\left(\sigma_T\nu\right)-a^\prime(u)\rangle^2\right)^2{\Delta}_q.
\end{eqnarray}
These expression corresponds to 
\begin{eqnarray}
\Delta_q&=\int Dzdx^{(0)}\phi^{(0)}(x^{(0)})
\left(\frac{\partial x^*_{\mathrm{RSB}}}{\partial\sqrt{\widehat{\chi}_0}z}\right)^2\widehat{\Delta}_q\\
\widehat{\Delta}_q&=\alpha\int D\nu D\zeta \left(
\frac{\partial}{\partial \sqrt{q_0-\frac{m^2}{\sigma_T^2}}\zeta}
\left({\cal G}\left(\sigma_T\nu\right)-a^\prime(u^*_{\mathrm{RSB}})\right)\right)^2{\Delta}_q,
\end{eqnarray}
hence 
the instability condition of $\Delta_q=\widehat{\Delta}_q=0$
is given by (\ref{eq:AT_replica}).

\section*{References}
\providecommand{\newblock}{}

\end{document}